\documentclass[10pt]{article}
\usepackage{latexsym,amsmath,url}
\usepackage{tikz}
\usepackage{textcomp}
\usepackage{mathdots}
\usepackage{amsfonts,euscript}
\usepackage{todonotes}
\usepackage{amsmath}
\usepackage{amssymb}
\usepackage{amsfonts}
\usepackage{comment}
\usepackage{amsthm}
\usepackage{mathtools}
\usepackage{mathrsfs}
\usepackage[all]{xy}
\usepackage{epstopdf}
\usepackage{subfigure}
\usepackage{rotating}
\usepackage{appendix}
\usepackage{dsfont}
\usepackage{multirow,rotating}
\usepackage{makecell, hhline}
\usepackage{url}
\usepackage[ruled,vlined]{algorithm2e}

\usepackage{algpseudocode}
\usepackage{subfigure}
\usepackage{setspace}
\usepackage{listings}
\usepackage{color}
\usepackage{tikz}
\usepackage{pifont}
\usepackage{graphics}
\usepackage[breaklinks=true,pagebackref]{hyperref}
\usepackage[english]{babel}
\usepackage{datetime}
\usepackage{booktabs}
\usepackage{rotfloat}
\usepackage{framed}
\usepackage[numbers]{natbib}
\usepackage{graphicx}
\usepackage{authblk}
\usetikzlibrary{shapes,decorations,arrows,calc,arrows.meta,fit,positioning}
\tikzset{
	-Latex,auto,node distance =1 cm and 1 cm,semithick,
	state/.style ={ellipse, draw, minimum width = 0.7 cm},
	point/.style = {circle, draw, inner sep=0.04cm,fill,node contents={}},
	bidirected/.style={Latex-Latex,dashed},
	el/.style = {inner sep=2pt, align=left, sloped}
}
\usetikzlibrary{matrix,decorations.pathreplacing}
\newcommand{\nian}[1]{{\color{blue} Nian says: #1}}

\graphicspath{pic/}
\newtheorem{theorem}{Theorem}
\newtheorem{corollary}{Corollary}[theorem]

\newtheorem{lemma}{Lemma}
\newtheorem{proposition}{Proposition}

\theoremstyle{definition}
\newtheorem{definition}{Definition}
\newtheorem{example}{Example}
\newtheorem{remark}{Remark}

\newcommand{\abs}[1]{\left|#1\right|}

\algdef{SE}[DOWHILE]{Do}{doWhile}{\algorithmicdo}[1]{\algorithmicwhile\ #1}

\allowdisplaybreaks

\newcommand{\norminf}[1]{\left\|#1\right\|_{\infty}}

\newcommand{\bd}[1]{\mathbf{#1}}
\newcommand{\mrm}[1]{\mathrm{#1}}
\newcommand{\KH}{\mathrm{K}_{\mathrm{H}}}
\newcommand{\KSAH}{\mathrm{K}_{\mathrm{H}}^{\mathrm{SA}}}
\newcommand{\KNH}{\mathrm{K}_{\mathrm{H}}^{\mathrm{N}}}
\newcommand{\KSH}{\mathrm{K}_{\mathrm{H}}^{\mathrm{S}}}
\newcommand{\KSAM}{\mathrm{K}_{\mathrm{M}}^{\mathrm{SA}}}
\newcommand{\KSM}{\mathrm{K}_{\mathrm{M}}^{\mathrm{S}}}
\newcommand{\KSAS}{\mathrm{K}_{\mathrm{S}}^{\mathrm{SA}}}
\newcommand{\KSS}{\mathrm{K}_{\mathrm{S}}^{\mathrm{S}}}
\newcommand{\KNM}{\mathrm{K}_{\mathrm{M}}^{\mathrm{N}}}

\newcommand{\KNS}{\mathrm{K}_{\mathrm{S}}^{\mathrm{N}}}

\newcommand{\PiH}{\Pi_{\mathrm{H}}}
\newcommand{\PiM}{\Pi_{\mathrm{M}}}
\newcommand{\PiS}{\Pi_{\mathrm{S}}}

\newcommand{\dsi}{\mathds{1}}
\newcommand{\R}{\mathbb{R}}

\newcommand{\Z}{\mathbb{Z}}
\newcommand{\N}{\mathbb{N}}

\newcommand{\ra}{\rightarrow}

\newcommand{\ua}{\uparrow}
\newcommand{\la}{\leftarrow}

\newcommand{\cd}{\cdot}
\newcommand{\ds}{\dots}

\newcommand{\cA}{\mathcal{A}}
\newcommand{\cB}{\mathcal{B}}

\newcommand{\cD}{\mathcal{D}}

\newcommand{\cF}{\mathcal{F}}
\newcommand{\cG}{\mathcal{G}}
\newcommand{\cH}{\mathcal{H}}

\newcommand{\cP}{\mathcal{P}}
\newcommand{\cQ}{\mathcal{Q}}

\newcommand{\cS}{\mathcal{S}}

\newcommand{\unif}{\mathrm{Unif}}

\newcommand{\spnorm}[1]{\left|#1\right|_\mathrm{span}}
\newcommand{\set}[1]{\left\{{#1}\right\}}

\newcommand{\floor}[1]{\left\lfloor{#1}\right\rfloor}

\newcommand{\sqbk}[1]{\left[ #1 \right]}
\newcommand{\sqbkcond}[2]{\left[ #1 \middle| #2 \right]}

\newcommand{\crbk}[1]{\left( #1 \right)}
\newcommand{\geom}{\mathrm{Geom}}

\newcommand{\argmax}[1]{\underset{#1}{\operatorname{arg}\,\operatorname{max}}\;}

\newcommand{\cblue}[1]{{#1}}

\definecolor{azure}{rgb}{0.0, 0.4, 0.9}

\definecolor{darkred}{rgb}{0.6, 0, 0}

\numberwithin{equation}{section}

\begin{document}

\title{On the Foundation of Distributionally Robust \\ Reinforcement Learning}
\author[1]{Shengbo Wang}
\author[2]{Nian Si}
\author[3]{Jose Blanchet}
\author[4]{Zhengyuan Zhou}
\affil[1]{Daniel J. Epstein Department of Industrial and Systems Engineering,
University of Southern California}
\affil[2]{Industrial Engineering and Decision Analytics, Hong Kong University of Science and Technology}
\affil[3]{Management Science and Engineering,
Stanford University}

\affil[4]{Stern School of Business,
New York University}
\date{August 2025 }
\maketitle

\begin{abstract}
Motivated by the need for a robust policy in the face of environment shifts between training and deployment, we contribute to the theoretical foundation of distributionally robust reinforcement learning (DRRL). This is accomplished through a comprehensive modeling framework centered around robust Markov decision processes (RMDPs). This framework obliges the decision maker to choose an optimal policy under the worst-case distributional shift orchestrated by an adversary. By unifying and extending existing formulations, we rigorously construct RMDPs that embrace various modeling attributes for both the decision maker and the adversary. These attributes include the structure of information availability—covering history-dependent, Markov, and Markov time-homogeneous dynamics—as well as constraints on the shifts induced by the adversary, with a focus on SA- and S-rectangularity. Within this RMDP framework, we investigate conditions for the existence or absence of the dynamic programming principle (DPP). From an algorithmic standpoint, the existence of DPP holds significant implications, as the vast majority of existing data and computationally efficient DRRL algorithms are reliant on the DPP. To investigate its existence, we systematically analyze various combinations of controller and adversary attributes, presenting streamlined proofs based on a unified methodology. We then construct counterexamples for settings where a fully general DPP fails to hold and establish asymptotically optimal history-dependent policies for key scenarios where the DPP is absent.
\end{abstract}

\section{Introduction}

\cblue{Our era has increasingly witnessed the transformative impact of reinforcement learning (RL), which has the potential to deliver -- and in some cases has already delivered -- significant productivity gains across a wide range of applications. Progress and applications of RL span} traditional engineering domains such as automation \citep{Landgraf2021automation}, inventory management \citep{gong2023bandits,mao2020model}, and the control of service systems \citep{hu2021prediction,zhalechian2023data}, as well as emerging areas in the digital economy, including personalized product recommendations on online platforms \citep{afsar2022reinforcement}, feature-based dynamic pricing \citep{bertsimas2006dynamic}, and real-time bidding in high-frequency auctions \citep{grislain2019recurrent,zhao2018deep}.

For deployment of RL in real-world applications, the reliability of the learning process is crucial, i.e., the policy learned from historical data or simulations needs to perform effectively when deployed in real-world scenarios. Standard reinforcement learning (RL), where the environment is typically modeled as an unknown Markov decision process (MDP), has demonstrated remarkable empirical success in simulated environments, with application domains spanning robotics \citep{kormushev2013reinforcement}, autonomous driving~\citep{shalev2016safe,pan2017virtual,kiran2021deep}, and games~\citep{mnih2015human,silver2017mastering,silver2018general}. 

\cblue{However, extending the success of RL in simulated environments to real environments remains challenging, largely due to the simplistic assumption, on which existing RL techniques critically depend, that the training and deployment environments are the same. In reality, it is common to encounter discrepancies arising from model mis-specifications (sim2real gap), environment shifts, and other confounding factors in real systems that could make the underlying dynamics non-Markovian are not captured by the simulator.}


\cblue{To address these challenges, we explore robust MDP (RMDP) formulations as a foundation for distributionally robust reinforcement learning (DRRL). We present a unified RMDP framework designed to support reliable policy learning under a wide range of model misspecifications, providing a principled approach to bridging the gap between idealized models and real-world complexities. At the core of our formulation is a two-player game: the controller seeks to maximize the expected cumulative rewards, while the adversary, modeling the possible yet unknown environmental shifts, selects a stochastic environment to hinder performance. This adversarial perspective offers a powerful lens for designing reliable controllers under uncertainty, particularly when the training occurs in potentially mis-specified environments (unknown unknowns), thereby promoting a more resilient learning paradigm.}

\cblue{
Despite a flourishing literature, see Section~\ref{subsec:literature} for a detailed discussion, on this topic, practitioners hoping to design a DRRL algorithm for their specific use case would be quickly and justifiably lost. There are two main reasons for this.
\begin{enumerate}
    \item \textbf{Modeling difficulty}: 
    Modeling the controller and the adversary, particularly the information each party has at the time of action, requires knowing -- and is hence reflecting the assumptions of -- the underlying application at hand. For example, modeling the adversary as time-homogeneous reflects the assumption or knowledge that the real environment is also an MDP, one that shares the same state and action spaces as the simulator, albeit with different and unknown transition probabilities. Here, the DRRL problem aims to learn a controller policy from the simulator that performs the best against the worst-case environment shift within a specified ambiguity set.
    
    However, in other applications, the practitioner might know beforehand that the real environment cannot be Markovian, as a result of the latent variables that are neither observed in the real system nor captured by the simulator. This calls for a richer model of the distributional shift, where incorporating a history-dependent adversary model can be valuable. Since the real environment is not fully observable, from the controller's perspective, the environment's dynamics will appear history-dependent. As such, positing that the environment is a history-dependent adversary allows the controller to robustify against this more complex environment. For details, see Section \ref{section:model_selection_guide}.

    Unfortunately, the existing literature lacks a clear exposition on this front. Often, a particular formulation is adopted throughout the paper without much discussion on the modeling aspect that led to it or the alternative formulations that one should consider under other modeling assumptions.  Another important modeling feature that tends to be left out of the discussion by the existing literature is the ambiguity set/adversary's decision set, which encodes how much power to grant to the adversary.
    
    An effective adversary must account for meaningful environment shifts without compromising the structural realism of the system or inducing overly conservative policies. Striking this balance is critical: a weak adversary may overlook relevant variations the RL agent could face in deployment, reducing robustness, whereas an overly strong adversary may force the agent to prepare for unrealistic worst-case scenarios, leading to impractical and overly cautious behavior.


    \item \textbf{Algorithm design difficulty}: 
    When the practitioner has thought through all the modeling aspects and finalized a particular RMDP model, how can he/she design a DRRL algorithm? Given that the dynamic programming principle (DPP) is the algorithmic foundation for any RL algorithm, one would naturally look into the corresponding DPP for the RMDP settings as the very first step. Unfortunately, the situation for RMDP models is more complex: not all formulations of RMDP that can reflect and represent the diverse modeling alternatives discussed above enjoy the DPP. This raises an important question for the practitioner: Which formulation satisfies the DPP, thereby making it tractable to design an effective DRRL algorithm, and which does not?
    
    Unfortunately, the existing literature does not provide definitive answers to this question, which has important consequences regardless of the answer being a yes or a no for a particular formulation the practitioner has at hand. If the answer is affirmative, a corresponding DRRL algorithm can likely be developed straightforwardly by, for example, converting the equality into an update rule. Conversely, even a negative answer is valuable. Just as knowing an optimization problem is nonconvex informs one to relax some constraints to achieve convexity, understanding the limitations of a particular formulation can guide adjustments in modeling assumptions toward one that satisfies the DPP. Such an iterative modeling approach, however, requires a comprehensive mapping of the DPP for DRRL, which is currently lacking.


\end{enumerate}
}

\cblue{Our paper aims to address these two types of difficulties and hence make the methodology of DRRL -- which is sound on a philosophical level -- practically \textit{relevant} and \textit{prescriptive} to a practitioner looking to build a robust controller. First, on the modelling side, a key contribution of our work is the development of a unified framework that systematically incorporates salient attributes of both the controller and the adversary, thereby expanding the range of behaviors captured by our RMDP formulations and enhancing their practical relevance. Our analysis explores multiple dimensions of the controller–adversary interaction, one of which is the availability of information to each party at the time of action. This includes considerations of history-dependent, Markov, and time-homogeneous structures. Additionally, we investigate the design of adversarial models in DRRL, focusing on how to capture essential features of real-world environments while preserving computational tractability. 
Our framework offers flexibility to navigate this trade-off by allowing the incorporation of specific structural properties, such as rectangularity and convexity in the adversarial decision sets, tailored to the application at hand.

Second, on the algorithm design side, we investigate the conditions under which DRRL formulations, accounting for the aforementioned variety of attributes for the controller-adversary dynamics, result in a dynamic programming principle (DPP) expressed through a Bellman-type equation \citep{bellman1954dpp}. As already mentioned, establishing the DPP provides the foundation for developing and analyzing efficient RL algorithms, even before any DRRL algorithm is actually designed.
To this end, we introduce a unifying framework for verifying the existence or absence of a DPP across various settings. A detailed overview of our results follows in the next sections. Beyond theoretical insights, our framework offers practical guidance for both researchers and practitioners in selecting and designing DRRL environments, as detailed in Section \ref{section:model_selection_guide}. This flexibility supports the design of both model-based and model-free DRRL algorithms, enabling the construction of policies that effectively hedge against the uncertainties of real-world deployment.
}

\subsection{Results and Methodology}
\par We first present an overview of the fundamental concepts in this manuscript. The study addresses the problem optimal control of RMDPs with infinite-horizon $\gamma$-discounted max-min value, formulated as follows: 
\begin{equation}
    \sup_{\pi\in \Pi} \inf _{\kappa \in \mrm{K}} E_\mu^{\pi,\kappa}\sqbk{\sum_{k=0}^\infty \gamma^{k} r(X_t,A_t)},
    \label{intro:RMDP}
\end{equation}
where $r$ is the reward function, $X_t$ and $A_t$ represent the state and action at time $t$ respectively, and $X_0\sim \mu$ as the initial distribution. The controller selects a policy $\pi$ from the class $\Pi$, maximizing the worst-case reward against an adversary who, with knowledge of the controller's policy, selects a policy from a class $\mrm{K}$. Here, the policy classes $(\Pi,\mrm{K})$ considered in this paper will be formally defined in Section \ref{section:RMDP_constuction_def}. 

Despite extensive exploration in the literature \citep{nilim2005robust,gonzalez2002minimax,iyengar2005robust,xu2010distributionally,le_tallec2007robustMDP,wiesemann2013robust,shapiro2021distributionally,Goyal2023Beyond_Rectangularity,li_shapiro2023rectangularity}, a unified framework that captures the full range of formulations for Problem \eqref{intro:RMDP} is still lacking. This gap makes it difficult to understand the modeling implications of DRRL and to design algorithms that reliably attain optimal policies. Our paper addresses this need by systematically examining the dynamic behavior resulting from different controller-adversary combinations, considering variations in information availability and rectangularity, thereby enriching our understanding of the RMDP problem and facilitating the deployment of DRRL.

\par From an information perspective, we systematically construct and analyze classes of controller and adversary policies that are history-dependent, Markov, or Markov time-homogeneous. A history-dependent agent bases actions on past observations; a Markov agent uses only the current state and time; and a Markov time-homogeneous agent relies solely on the current state, applying the same rule across all time periods. A key innovation of our framework is that these information structures can be assigned \textit{asymmetrically}—for example, a history-dependent controller facing a time-homogeneous adversary—enabling a rich family of robust decision-making models.

\par The concept of rectangularity, originally introduced in the RMDP literature to describe the adversary’s ability to vary actions over time \citep{iyengar2005robust}, has since evolved. With the adoption of different information structures and a growing emphasis on limiting adversarial power, rectangularity is now used to impose structural constraints on the adversary, as discussed in \citet{le_tallec2007robustMDP} and \citet{wiesemann2013robust}. Common assumptions in the literature include S-rectangularity and SA-rectangularity. An SA-rectangular adversary can choose distinct action distributions for each state-action pair \((s,a) \in S \times A\), while S-rectangularity imposes restrictions by requiring a consistent distribution across actions within each state. Example \ref{example:inventory_S_rec} illustrates this distinction. 

\par As mentioned earlier, the focus of this paper is to identify scenarios in which the max-min control value \eqref{intro:RMDP} satisfies a robust DPP, a.k.a. robust Bellman equation:
\cblue{\begin{equation}\label{eqn:intro:dr_minmax_bellman_eqn}
u(s) =\sup_{\phi\in \cQ}  \inf_{p_s\in\cP_s} E_{\phi,p_s}[r(s,A_0)+ \gamma u(X_1)], \quad s\in S,
\end{equation}
where the distribution of $A_0,X_1$ under $P_{\phi,p_s}$ is given by $P_{\phi,p_s}(A_0 = a_0,X_1 = s_1) =\phi(a_0)p_{s,a_0}(s_1)$. }Further, if the action of the controller is allowed to be randomized, then $\cQ$ represents the set of probability distributions, whereas $\mathcal{Q}$ will denote the action set (understood as Dirac measures) when the controller is limited to being deterministic. $\mathcal{P}_s$ is the adversary's action distribution set (also known as an ambiguity set in the RMDP literature) at state $s$.

In view of this objective, it is natural to assume an SA or S-rectangular adversary, as achieving a DPP with complete generality might not be feasible when dealing with a general-rectangular adversary \citep{le_tallec2007robustMDP,wiesemann2013robust}.

In the subsequent discussion, we provide a concise review of the results of the existing literature, drawing comparisons to the implications outlined in the theoretical framework of this paper. This comparison is presented through tables, specifically Table \ref{tab:sa-rec}, \ref{tab:s-rec_conv}, \ref{tab:s-rec_randomized}, and \ref{tab:s-rec_det}. These tables are structured based on the rectangularity and convexity characteristics of the adversary, as well as whether the controller's actions are deterministic or randomized. Here's a breakdown:
\begin{itemize}
    \item 
    Table \ref{tab:sa-rec} concerns results with an SA-rectangular adversary. The controller's actions may either be randomized or restricted to deterministic choices.\begin{table}[ht]
	\small
  \centering
  \caption{\textit{Randomized or deterministic} controller versus SA-rectangular adversary policy classes.}

  \label{tab:sa-rec}%
      \begin{tabular}{ll|ccc} 
    \Xhline{1.2pt}
          &       & \multicolumn{3}{c}{Adversary} \\
          &       & History-dependent & Markov & Time-homogeneous \\
          \Xhline{1.2pt}
    \multicolumn{1}{c}{\multirow{6}{*}{\begin{turn}{-90}Controller\end{turn}}} & History- &    \multicolumn{1}{c}{\multirow{2}{*}{{\color{green} \ding{52}} \citet{gonzalez2002minimax} }}     & \multicolumn{1}{c}{\multirow{2}{*}{{\color{green} \ding{52}} \citet{iyengar2005robust}} }     & \multicolumn{1}{c}{\multirow{2}{*}{{\color{green} \ding{52}} } }   \\ 
      &dependent &       &      & \\
     \multicolumn{1}{c}{} &  \multicolumn{1}{l|}{\multirow{2}{*}{Markov }}  &\multicolumn{1}{c}{\multirow{2}{*}{{\color{green} \ding{52}}}}  &\multicolumn{1}{c}{\multirow{2}{*}{{\color{green} \ding{52}} \citet{nilim2005robust}} }     &\multicolumn{1}{c}{\multirow{2}{*}{{\color{green} \ding{52}} \citet{nilim2005robust}}} \\ 
       \multicolumn{1}{c}{} & &     &      & \\
    \multicolumn{1}{c}{} &  Time-  &\multicolumn{1}{c}{\multirow{2}{*}{\color{green} \ding{52}}}  & {\color{green} \ding{52}}  \citet{iyengar2005robust}    &\multicolumn{1}{c}{\multirow{2}{*}{{\color{green} \ding{52}} \citet{nilim2005robust}}} \\ 
      \multicolumn{1}{c}{} & homogeneous &      &  \citet{nilim2005robust}    &\\
    \Xhline{1.2pt}
    \end{tabular}%

\end{table}%
    \item 
    Table \ref{tab:s-rec_conv} considers scenarios with an S-rectangular adversary operating in a convex action set, while the controller opts for fully randomized actions.\begin{table}[!ht]
	\small
	\centering
	 \caption{\textit{Randomized} controller versus \textit{convex} S-rectangular adversary policy classes.}
	\label{tab:s-rec_conv}%
	\begin{tabular}{ll|ccc} 
		\Xhline{1.2pt}
		&       & \multicolumn{3}{c}{Convex Adversary} \\
		&       & History-dependent & Markov & Time-homogeneous \\
		\Xhline{1.2pt}
		\multicolumn{1}{c}{\multirow{6}{*}{\begin{turn}{-90}Randomized \end{turn} \begin{turn}{-90}Controller \end{turn} }} & History- &    \multicolumn{1}{c}{\multirow{2}{*}{{\color{green} \ding{52}}}}     &   \multicolumn{1}{c}{\multirow{2}{*}{{\color{green} \ding{52}} \citet{xu2010distributionally}} }    &\multicolumn{1}{c}{{{\color{green} \ding{52}} \citet{wiesemann2013robust} }} \\ 
		&dependent &       &     & \phantom{{\color{green} \ding{52}}}\citet{xu2010distributionally}\\
		\multicolumn{1}{c}{} &  \multicolumn{1}{l|}{\multirow{2}{*}{Markov }}  &\multicolumn{1}{c}{\multirow{2}{*}{{\color{green} \ding{52}}}  }  &\multicolumn{1}{c}{\multirow{2}{*}{{\color{green} \ding{52}}    \citet{li_shapiro2023rectangularity} }} &\multicolumn{1}{c}{\multirow{2}{*}{{\color{green} \ding{52}}  }} \\ 
		\multicolumn{1}{c}{} & &     &      & \\
		\multicolumn{1}{c}{} &  Time-  &\multicolumn{1}{c}{\multirow{2}{*}{\color{green} \ding{52}}}  &  \multicolumn{1}{c}{\multirow{2}{*}{{\color{green} \ding{52}}}  }   &\multicolumn{1}{c}{\multirow{2}{*}{{\color{green} \ding{52}} \citet{le_tallec2007robustMDP} } } \\ 
		\multicolumn{1}{c}{} & homogeneous &      &      &\\
		\Xhline{1.2pt}
	\end{tabular}%

\end{table}
    \item 
    Table \ref{tab:s-rec_randomized} assumes an S-rectangular adversary that takes action in a possibly non-convex set, while the controller can choose randomized actions. \begin{table}[!ht]
	\small
	\centering
	 \caption{\textit{Randomized} controller versus \textit{non-convex} S-rectangular adversary policy classes.}
	\label{tab:s-rec_randomized}%
	\begin{tabular}{ll|ccc} 
		\Xhline{1.2pt}
		&       & \multicolumn{3}{c}{Non-Convex Adversary} \\
		&       & History-dependent & Markov & Time-homogeneous \\
		\Xhline{1.2pt}
		\multicolumn{1}{c}{\multirow{6}{*}{\begin{turn}{-90}Randomized \end{turn} \begin{turn}{-90}Controller \end{turn} }} & History- &    \multicolumn{1}{c}{\multirow{2}{*}{{\color{green} \ding{52}}}}     & \multicolumn{1}{c}{\multirow{2}{*}{{\color{green} \ding{52}}}}     &\multicolumn{1}{c}{\multirow{2}{*}{{\color{red} \ding{55}} \citet{wiesemann2013robust}} } \\ 
		&dependent &       &     & \\
		\multicolumn{1}{c}{} &  \multicolumn{1}{l|}{\multirow{2}{*}{Markov }}  &    \multicolumn{1}{c}{\multirow{2}{*}{{\color{green} \ding{52}}}}  & \multicolumn{1}{c}{\multirow{2}{*}{{\color{green} \ding{52}} \citet{li_shapiro2023rectangularity}  }}      &\multicolumn{1}{c}{\multirow{2}{*}{{\color{red} \ding{55}}  }} \\ 
		\multicolumn{1}{c}{} & &     &      & \\
		\multicolumn{1}{c}{} &  Time-  &    \multicolumn{1}{c}{\multirow{2}{*}{\color{green} \ding{52}}}  & \multicolumn{1}{c}{\multirow{2}{*}{{\color{green} \ding{52}}}}   &\multicolumn{1}{c}{\multirow{2}{*}{{\color{green} \ding{52}}\citet{le_tallec2007robustMDP}}} \\ 
		\multicolumn{1}{c}{} & homogeneous &      &      &\\
		\Xhline{1.2pt}
	\end{tabular}%

\end{table}%
    \item 
    Table \ref{tab:s-rec_det} also assumes an S-rectangular adversary taking actions in a convex or possibly non-convex set. However, the controller is constrained to choose deterministic actions.
    \begin{table}[!ht]
	\small
	\centering
	 \caption{\textit{Deterministic} controller versus \textit{convex or non-convex} S-rectangular adversary policy classes.}
	\label{tab:s-rec_det}%
	\begin{tabular}{lll|ccc} 
		\Xhline{1.2pt}
		& &       & \multicolumn{3}{c}{Convex/Non-Convex Adversary} \\
		&       & & History-dependent & Markov & Time-homogeneous \\
		\Xhline{1.2pt}
		\multicolumn{2}{c}{\multirow{6}{*}{\begin{turn}{-90}Deterministic \end{turn} \begin{turn}{-90}Controller \end{turn} }} 
   & History- &    \multicolumn{1}{c}{\multirow{2}{*}{{\color{green} \ding{52}}}}     &\multicolumn{1}{c}{\multirow{2}{*}{{\color{red} \ding{55}}}}    &\multicolumn{1}{c}{\multirow{2}{*}{{\color{red} \ding{55}}}} \\ 
		& &dependent &       &     & \\
		& \multicolumn{1}{c}{} &  \multicolumn{1}{l|}{\multirow{2}{*}{Markov }}  &    \multicolumn{1}{c}{\multirow{2}{*}{{\color{green} \ding{52}}}}  & \multicolumn{1}{c}{\multirow{2}{*}{{\color{green} \ding{52}}}}      &\multicolumn{1}{c}{\multirow{2}{*}{{\color{red} \ding{55}}}} \\ 
		&\multicolumn{1}{c}{} & &     &      & \\
		& \multicolumn{1}{c}{} &  Time-  &    \multicolumn{1}{c}{\multirow{2}{*}{\color{green} \ding{52}}}  &\multicolumn{1}{c}{\multirow{2}{*}{{\color{green} \ding{52}}}}    &\multicolumn{1}{c}{\multirow{2}{*}{{\color{green} \ding{52}}}} \\ 
		& \multicolumn{1}{c}{} & homogeneous &      &      &\\
		\Xhline{1.2pt}
	\end{tabular}%

\end{table}%
\end{itemize}
Within these tables, a green check mark signifies the validation of the DPP, as defined in Definition \ref{def:DPP}. Conversely, a red cross indicates a counterexample discussed in Section \ref{section:counterexample}. We address all the 36 cases in the tables and beyond, either confirming or disproving them. While some findings rediscover results in the previous literature, as evidenced by citations in table entries, we present streamlined proofs of the existence of DPPs grounded in unified principles.

\cblue{
\par As summarized in Section \ref{section:DPP_summary}, the unifying principles stem from two perspectives: the convexity of the action sets of the controller and the adversary, and the interchangeability of the sup-inf within the robust Bellman equation \eqref{eqn:dr_bellman_eqn}.

\par \textbf{Convexity perspectives:} The six cases located in the lower triangles of all tables consistently hold without requiring convexity assumptions, as proved in Theorem \ref{thm:s-rec_constrained_checks}. If the controller opts for randomized actions from a convex set of distributions, then, additionally, the scenario where a history-dependent controller interacts with an S-rectangular Markov adversary with a possibly non-convex action set results in a DPP (c.f. Theorem \ref{thm:s-rec_conv_constrained_checks}). It's important to note that even if the adversary's action set is convex, a non-convex controller could lead to the same non-existence of DPPs as in the non-convex adversary case.

\par \textbf{Interchangeability perspective:} Under our unifying framework, we establish an important novel result in Theorem~\ref{thm:s-rec_constrained_max_min_eq_min_max}: if the interchanged Bellman equation \eqref{eqn:dr_minmax_bellman_eqn} admits the same solution as \eqref{eqn:dr_bellman_eqn}, then the DPP holds universally, regardless of the information asymmetry between the controller and the adversary. This insight, in particular, confirms the validity of the DPP in settings where the adversary is SA-rectangular (Theorem~\ref{thm:sa-rec_dpp}), leading to the results in Table~\ref{tab:sa-rec}.

\par The convergence of these two perspectives occurs when the conditions of Sion's min-max principle \citep{Sion1958minmax} are met--specifically, when both parties have convex action distribution sets, and one of them is compact. Consequently, the interchanged equation \eqref{eqn:dr_minmax_bellman_eqn} has the same solution as \eqref{eqn:dr_bellman_eqn}, ensuring that all 9 cases adhere to the DPP (Corollary \ref{cor:conv_comp_ctrl_adv}). This results in the validity of Table \ref{tab:s-rec_conv}.  
}

\par We further show that the DPP implies the optimality of a Markov time-homogeneous controller and that such a policy can be directly obtained from the solution to the robust Bellman equation. This confirms the validity of DRRL approaches that learn a robust optimal policy by approximately solving the robust Bellman equation—either the value function in the S-rectangular case \eqref{eqn:dr_bellman_eqn} or the \( q \)-function in the SA-rectangular case \eqref{eqn:dr_bellman_eqn_q}. 
\cblue{
\par Moreover, we establish counterexamples for all cases that cannot be covered by the previously mentioned principles, thus completing our unified framework. A notable insight behind the construction of these counterexamples is that the non-existence of a DPP is closely related to the controller's ability to learn certain characteristics of the worst-case adversary over a sequence of strategically deployed actions. Therefore, an adept controller actively engages in learning, and the existence of a DPP is tied to the controller's inability to learn within the given information structure. Consequently, in constructing counterexamples, it proves fruitful to conceptualize the controller as an agent in a multi-armed bandit or RL problem, as opposed to an MDP problem.

\par These counterexamples highlight the fundamental limitations of non-history-adaptive policies in RMDPs where Markovian or time-homogeneous adversaries select actions from non-convex S-rectangular sets. A particularly relevant setting of practical interests involves history-dependent controllers—either randomized or deterministic—interacting with time-homogeneous adversaries. To address this challenge, we propose an RL-based history-dependent policy that is both provably efficient and practically implementable (c.f. Theorem \ref{thm:asymp_opt_policy}). Our policy achieves asymptotic optimality as the effective horizon \( 1/(1-\gamma) \) grows and builds on model-based RL algorithms that are widely used in practice. Indeed, it exploits the adversary’s time-homogeneity to enable strategic adaptation, ensuring sufficient exploration and the near-optimal use of historical data.}

\subsection{Literature Review}\label{subsec:literature}

\par This paper contributes to the existing literature by offering a unified perspective on the existence of a DPP in the context of RMDP, encompassing various information structures and adversarial rectangularity. This unification leads to the identification of new results regarding the existence and non-existence of DPP under various policy structures of the controller and the adversary. To appreciate this contribution, we begin with a concise overview of the relevant literature in this domain.

\par RMDP models are motivated by optimal decision-making in stochastic environments where the underlying probability structure is shaped by a decision maker (referred to as the controller) and an opposing counterpart (referred to as the adversary), engaged in a zero-sum interaction. This naturally aligns with the literature on two-player zero-sum stochastic games \citep{shapley1953stochastic,solan2015stochastic}, a connection we will explore subsequently in Section \ref{section:stochastic_game}. However, unlike traditional stochastic games, in RMDP,  the adversary's choices are confined to a predetermined subset of transition distributions, while the controller typically still has the liberty to choose an arbitrary randomized action.

\par In RMDPs, DPP is highly desirable because it certifies the optimality of stationary Markov policies and enables efficient control and RL algorithms. Accordingly, research has focused on conditions that retain realistic structural features yet still guarantee a DPP—most notably by imposing the adversary’s \emph{rectangularity} property. This concept was first introduced by \citet{iyengar2005robust} and \citet{nilim2005robust}. In these foundational works, what is now recognized as an SA-rectangular adversary was considered, and the existence of a DPP was demonstrated, as summarized in Table \ref{tab:sa-rec}. This assumption also underlies \citet{gonzalez2002minimax} on min-max control, where the same DPP is established within the context of a history-dependent adversary.

\par  However, as illustrated in Example \ref{example:inventory_S_rec} below, the SA-rectangular adversary may wield excessive power in many practical modeling environments. This leads to the introduction of S-rectangularity  \citep{le_tallec2007robustMDP,xu2010distributionally,wiesemann2013robust}. Notably, these works consider specific convex S-rectangular adversaries and establish the existence of a DPP, as outlined in Table \ref{tab:s-rec_conv}. Unlike the SA-rectangular setting, \citet{wiesemann2013robust} shows that, in general, the optimal stationary Markov robust policy needs to be randomized. 

\par Building on this, \citet{le_tallec2007robustMDP} and \citet{wiesemann2013robust} also introduce the concept of a general rectangular adversary and establish the non-existence of a canonical DPP. To further constrain the adversary's power and preserve the structural integrity of the modeled environment, the recent work by \citet{mannor2016robust, Goyal2023Beyond_Rectangularity} introduces a new notion known as k- and r-rectangularity.

\par In addition to rectangularity, another crucial factor influencing the existence or non-existence of the DPP is the asymmetry in information availability between the controller and the adversary. \citet{iyengar2005robust} noted, without offering a proof, that the asymmetry in historical information structure between the controller and the adversary dynamics could lead to the non-existence of a stationary optimal policy for maximizing infinite-horizon discounted rewards. This insight is also noted by \citet[Table 1, S-rectangular,  nonconvex]{wiesemann2013robust} without giving a counterexample. The work by \citet{gonzalez2002minimax} approaches the min-max control problem with symmetric history-dependent controllers and adversaries. On the other hand, \citet{iyengar2005robust, nilim2005robust, xu2010distributionally, wiesemann2013robust} take into account non-symmetric information structures, as indicated in the tables.
\par There exists a substantial body of literature dedicated to investigating the finite-time sample complexity of achieving optimality  in tabular DRRL environments. This line of research crucially relies on DPPs and specifically focuses on divergence-based ambiguity sets of SA-rectangular adversaries. In line with the classical tabular RL setting, the approaches in these algorithms can be categorized based on whether the algorithm explicitly estimates the entire transition kernel. If it does, it falls under the scope of a model-based approach. There is a rapidly growing literature, adopting either a model-based \citep{zhou21, Panaganti2021, yang2021, ShiChi2022, xu2023improved, shi2023curious_price,blanchet2023double_pessimism_drrl} or a model-free \citep{liu22DRQ, Wang2023MLMCDRQL, wang2023VRDRQL, yang2023avoiding} approach. In the specific domain of health care, \citet{saghafian2023ambiguous} proposes RL methods to learn an optimal dynamic treatment regime in an ambiguous environment. 
\par Much like the formulations of RMDPs using the language of classical MDPs, there exists a closely related line of research to RMDP that studies the robust variant of multistage stochastic programs (MSPs), see for example \citep{huang2017study, shapiro2021distributionally,pichler2021mathematical,shapiro2021lectures} and references therein. While this literature primarily focuses on finite-horizon environments, it shares strong connections with our RMDP framework and can be seen as alternative interpretations for DRRL models.

\cblue{ 
Moreover, robust MSPs are closely related to risk measures. In the static setting, certain distributionally robust objectives are known to correspond to coherent risk measures. A well-known example is that the conditional value-at-risk (CVaR) measure is equivalent to a distributionally robust objective with the ambiguity set
$\mathcal{P} = \left\{Q: dQ/dP\in[0,1/(1-\alpha)],Q(\Omega)=1 \right\}$
 In the dynamic setting, \citet{le_tallec2007robustMDP,osogami2012robustness,pichler2021mathematical,shapiro2023conditional} show that robust multistage programming is equivalent to optimization with respect to nested (also called iterated or dynamic) conditional risk functionals \citep{ruszczynski2010risk}.
}


\cblue{
Another related line of research is risk-sensitive Markov Decision Processes (MDPs). In risk-sensitive MDPs, the objective is to optimize
$E \left[U\left(\sum_{k=0}^\infty \gamma^{k} r(X_t,A_t) \right)\right],$
where $U$ is a nonlinear utility function \citep{hernandez1996risk,fleming1997risk,di1999risk,di2007infinite,bauerle2014more}. A common choice would be $U(\cdot)=\exp(\cdot)$.  The risk-aware MDP literature also study objectives of the form
$\rho\left(\sum_{k=0}^\infty \gamma^{k} r(X_t,A_t) \right),$
where $\rho$ is a risk measure \citep{haskell2015convex,tamar2016sequential,shapiro2016time}. Notably, \citet{chow2015risk} showed that risk-sensitive MDPs  with CVaR risk measure is equivalent to robust MDPs with multiplicative probability perturbation. In addition, \citet{delage2010percentile,haskell2013stochastic,chow2018risk} studied the risk-constrained MDPs, where they maximize expected cumulative rewards under risk measure constraints. 

Establishing a dynamic programming principle in risk-sensitive MDPs often requires augmenting the state space to maintain the necessary Markov property, which differs from the setting in this paper.  
Moreover,  unlike robust MDPs, risk-sensitive MDPs do not involve ambiguity sets or adversarial agents. As a result, there is no issue of information asymmetry in this framework.}

\subsection{Remark on Paper Organization}
The structure of the paper unfolds as follows: Section \ref{section:RMDP_constuction_def}  formulates RMDP problems, encompassing various types of controllers and adversaries. Our key findings for scenarios where the DPP holds with full generality are expounded in Section \ref{section:DPP_constrained}. Section \ref{section:markov_time_homo_opt} delves into the implications of the DPP, specifically addressing the optimality of Markov time-homogeneous policies. In Section \ref{section:counterexample}, we provide counterexamples for cases where the DPP does not hold. 
Section \ref{section:asymp_opt_hd_policy} proposes and analyzes an efficient RL-based history-dependent policy that achieves asymptotically optimal performance against a time-homogeneous non-convex adversary. Sections \ref{section:multistage_sp} and \ref{section:stochastic_game} explore the connections between our RMDP formulations and robust multistage stochastic programs and stochastic games, respectively. Finally, in Section~\ref{section:model_selection_guide}, we offer insights to guide RMDP model selection for practitioners and outline several promising directions for future research.

\section{Robust MDPs: Construction and Definitions}
\label{section:RMDP_constuction_def}
While the theory we present in this paper naturally extends to infinite (countable, continuum, or general measurable) state and action spaces, to facilitate a better understanding of the readers from multiple disciplines, we focus our discussion exclusively on the setting of finite state and action spaces. In particular, let $S$ and $A$ denote the finite state and action spaces, respectively, and equip them with the $\sigma$-fields of all subsets $\cS = 2^S$ and $\cA = 2^A$. It should be noted that the formulation in this section is valid for abstract state and action measurable spaces $(S,\cS)$ and $(A,\cA)$.
\par In this context of finite states and actions, we consider the canonical space of an MDP. Specifically, let $(\Omega = (|S|\times |A|)^{\Z_{\geq 0}}, \cF)$ be the underlying measurable space where $\cF$ is the cylinder $\sigma$-field. Define the stochastic process $\set{(X_t,A_t),{t\geq 0}}$ by the point evaluation $X_t(\omega) = s_t,A_t(\omega) = a_t$ for all $t\geq 0$ and any $\omega = (s_0,a_0,s_1,a_1,\ds)\in \Omega$. In classical MDP formulations, the policy of the controller induces a probability measure on $(\Omega,\cF)$. In comparison, within the realm of RMDPs, the interplay between the controller and the adversary dictates a measure on the sample path space. Consequently, to rigorously formulate an RMDP, we commence by introducing the policy classes of both the controller and the adversary.
\par Similar to regular MDP formulations, a controller typically has the flexibility to choose its action at the current time and space based on all historical information (i.e. the state-action sequence realized until the current time instance) available. In an RMDP, both the controller and the adversary, in general, are empowered to employ history-dependent policies as well. To establish a formulation of RMDP with this characteristic, we introduce the definitions of the controller's and adversary's history and their respective notions of history-dependence in the context of an RMDP setting.
\subsection{Controller's Policy}
\par Similar to classical MDPs, a history-dependent controller in the context of an RMDP takes (randomized) actions based on the state-action sequence until the current state. Formally, the \textit{controller's history} $\set{\bd{H}_t:t\geq 0 }$ is the $t$ indexed family of truncated sample paths
\[
\bd{H}_t := \set{h_t = (s_0,a_0,\ds,a_{t-1} ,s_t): \omega = (s_0,a_0,\ds,a_{t-1} ,s_{t},\ds )\in\Omega}.
\]
We also define the random element $H_t:\Omega\ra \bd{H}_t$ by point evaluation $H_t(\omega) = h_t$. 
\cblue{
\par Given a prescribed subset of probability measures $\cQ\subset\cP(\cA)$, a controller policy $\pi$ is a sequence of \textit{decision rules} $\pi = (\pi_0,\pi_1,\pi_2,\ds)$. Each decision rule $\pi_t$, indexed by $t$, is a measure valued function $\pi_t:\bd{H}_t\ra \cQ$ and is represented in the conditional distribution format as $\pi_t(a|h_t)\in [0,1]$ and $\sum_{a\in A}\pi_t(a|h_t) = 1$ . We call the set $\cQ$ the \textit{action distributions} of the controller. 
\par The set of \textit{$\cQ$-constrained} history-dependent controller policies is
\begin{equation}\label{eqn:constrained_controller}
\PiH(\cQ):=\set{\pi = (\pi_0,\pi_1,\ds):\pi_t \in \set{\bd{H}_t\ra \cQ}, \forall t\geq 0 }. 
\end{equation}
}

\par A decision rule $\pi_t$ is \textit{Markov} if for all $h_t,h'_t\in\bd{H}_t $ s.t. $s_t = s'_t$, then $\pi_t(\cd |h_t) =\pi_t(\cd |h'_t)$. Therefore, a Markov decision rule $\pi_t$ is indifferent to the state-action sequence $(s_0,a_0,\ds,s_{t-1},a_{t-1})$. Consequently, we can express $\pi_t(\cd |s_t) = \pi_t(\cd |h_t)$ when $\pi_t$ is Markov, employing an abuse of notation for the sake of simplicity. Then, define the set of $\cQ$-constrained Markov controller policies is defined as
\[
\PiM(\cQ) : = \set{\pi = (\pi_0,\pi_1,\ds):\pi_t \in \set{\bd{H}_t\ra \cQ} \text{ is Markov}, \forall t\geq 0 }. 
\]

\par A contoller's policy $\pi =  (\pi_0,\pi_1,\ds)$ is called \textit{Markov time-homogeneous} (a.k.a. \textit{static} as in \citet{iyengar2005robust}) if for any $t_1,t_2\geq 0$ and $h_{t_1}\in\bd{H}_{t_1},h'_{t_2}\in\bd{H}_{t_2}$ s.t. $s_{t_1} = s'_{t_2}$, then $\pi_{t_1}(\cd |h_{t_1}) = \pi_{t_2}(\cd | h'_{t_2})$. In the case of a Markov time-homogeneous policy, the decision rule remains consistent regardless of the time and history once a state $s$ is specified---hence the term ``static". Consequently, such a policy $\pi$ can be identified with a measure-valued function $\pi: S \rightarrow \mathcal{P}(\mathcal{A})$ by considering the policy as $(\pi,\pi,\pi,\ldots)$. Then, the set of $\cQ$-constrained Markov time-homogeneous (or static) controller policies is defined as
\[
\PiS(\cQ) : = \set{(\pi,\pi,\ds):\pi\in \set{S\ra \cQ}}. 
\]
The underlying constraint set \(\cQ\) is usually clear from the context. Thus, we simplify the notation by writing \(\PiH \equiv \PiH(\cQ)\), and the same convention is applied to \(\PiH\) and \(\PiS\).
\par  Of particular interests to RL applications are the following two special type of constraints for controller's decision rule. First, when $\cQ = \cP(\cA)$, $\PiH(\cP(\cA))/\PiM(\cP(\cA))/ \PiS(\cP(\cA))$ represent the fully \textit{randomized} (or unconstrained) history-dependent/Markov/Markov time-homogeneous policy classes. Second, if $\cQ =\cQ^\mrm{D}:=  \set{\delta_{\set{a}}:a\in A}$, then $\PiH(\cQ^\mrm{D})/\PiM(\cQ^\mrm{D})/ \PiS(\cQ^\mrm{D})$ are the \textit{deterministic} history-dependent/Markov/Markov time-homogeneous policy classes, respectively.

\subsection{Adversary's Policy}
\par We define the \textit{adversary's history} $\set{\bd{G}_t:t\geq 0 }$ as the $t$ indexed family of truncated sample paths
\[
\bd{G}_t := \set{g_t = (s_0,a_0,\ds ,s_t,a_t): \omega = (s_0,a_0,\ds ,s_{t},a_t\ds )\in\Omega}.
\]
Note that $g_t$ represents the concatenation of the history $h_t$ with the controller's action at time $t$, i.e., $g_t = (h_t,a_t)$, where $h_t\in\bd{H}_t$. Also, define the random element $G_t:\Omega\ra \bd{G}_t$ by point evaluation $G_t(\omega) = g_t$. By incorporating this additional $a_t$ into the definition of the adversarial history, we provide the modeling flexibility for the adversary to be informed about the realization of the action based on the controller's decision at the current time point. 
\cblue{
\par Similar to the controller, an adversarial policy $\kappa := (\kappa_0,\kappa_1,\kappa_2,\ds)$ is also a sequence of measure-valued functions of the history $\kappa_t: \bd{G}_t\ra \cP(\cS)$ and represented as $\kappa_t(s|g_t)$. }For any $t\geq 0$, we call $\kappa_t$ the \textit{adversarial decision rule} at $t$. Under this formulation, the largest set of adversarial policies $\KH$, a.k.a. the set of \textit{unconstrained} history-dependent adversarial policies is
$\set{\kappa = (\kappa_0,\kappa_1,\ds):\kappa_t \in \set{\bd{G}_t\ra \cP(\cS)}, \forall t\geq 0 }. $
\par Analogue to the controller-side definitions, we call an adversarial decision rule $\kappa_t$  Markov if for all $g_t,g'_t\in\bd{G}_t $ s.t. $(s_t,a_t) = (s'_t,a_t')$, then $\kappa_t(\cd |g_t) =\kappa_t(\cd |g'_t)$. So, a Markov adversarial decision rule $\kappa_t$ is indifferent to $g_{t-1} = (s_0,\ds,a_{t-1})$ given $(s_t,a_t)$. Thus, by an abuse of notation, we can write 
\begin{equation}\label{eqn:markov_adv_decision_rule_notation}
\kappa_t(\cd|s_t,a_t) = \kappa_t(\cd|g_t).
\end{equation}
\par Moreover, an adversary's policy $\kappa =  (\kappa_0,\kappa_1,\ds)$ is called Markov time-homogeneous  (or static) if for all $t_1,t_2\geq 0$ and $g_{t_1}\in\bd{G}_{t_1},g'_{t_2}\in\bd{G}_{t_2}$ s.t. $(s_{t_1},a_{t_1}) = (s'_{t_2},a'_{t_2})$, then $\kappa_{t_1}(\cd|g_{t_1}) = \kappa_{t_2}(\cd|g'_{t_2})$. So, for a Markov time-homogeneous policy, once a state action pair $(s,a)$ is given, the distribution of the next state is the same regardless of the time and history---hence the term ``static". Again, such policies can be represented as $(\kappa,\kappa,\kappa,\ds)$ where $\kappa:S\times A\ra \cP(\cS)$.  

\par As motivated in the introduction, in environments where an RMDP serves as a viable model, it becomes necessary to restrict the adversary's power to a constrained set of policies. On the other hand, a DPP, expressed through the Bellman equation \eqref{eqn:intro:dr_minmax_bellman_eqn}, is crucial for certifying stationary Markov optimal policies and for enabling reinforcement learning and optimal control. These two considerations form the primary motivation behind the concept of rectangularity, as introduced in \cite{iyengar2005robust}. This notion has subsequently been extended and generalized in works such as \cite{wiesemann2013robust} and \cite{Goyal2023Beyond_Rectangularity}.
\par We first introduce the notion of rectangularity considered by \cite{gonzalez2002minimax}, \cite{iyengar2005robust}, and \cite{nilim2005robust}. In contemporary terminology, this concept is referred to as SA-rectangularity.

\subsubsection{SA-Rectangular Set of Adversary's Policies}

\cblue{
\par For each $(s,a)\in  S\times A$, let $\cP_{s,a}\subset{\cP(\cS)}$ be a prescribed non-empty subset of probability measure, which can be equivalently seen as a subset of row vectors in $\R^{| S|}$. Let $\cP^{\mrm{SA}}:= \bigtimes_{(s,a)\in  S\times A} \cP_{s,a}:$ be the collection of SA-rectangular adversarial ambiguity sets, where $\bigtimes$ denotes the Cartesian product of sets. We first consider the \textit{SA-rectangular} marginal set of history-dependent adversarial decision rules induced by $\cP^{\mrm{SA}}$
\begin{equation}\label{eqn:KSAHt}
\KSAH(\cP^{\mrm{SA}},t) := \set{\kappa_t:   \kappa_t(\cd |g_{t-1},s,a) \in \cP_{s,a}, \forall g_{t-1}\in\bd{G}_{t-1}, (s,a)\in  S\times A}. 
\end{equation}
\( \mathrm{K}^{\mathrm{SA}}_{\mathrm{M}}(\mathcal{P}^{\mathrm{SA}}, t) \) and \( \mathrm{K}^{\mathrm{SA}}_{\mathrm{S}}(\mathcal{P}^{\mathrm{SA}}, t) \) can be analogously defined when the adversary is restricted to be Markov or Markov time-homogeneous, respectively.
\par  Then, the SA-rectangular set of history-dependent adversarial policies induced by $\cP^{\mrm{SA}}$ is
\[
\KSAH(\cP^{\mrm{SA}}):= \set{\kappa = (\kappa_0,\kappa_1,\ds):\kappa_t \in \KSAH(\cP^{\mrm{SA}},t), \forall t\geq 0} = \bigtimes_{t\geq 0} \KSAH(\cP^{\mrm{SA}},t).
\]
} Similar to the controller policy classes, we define the SA-rectangular set of Markov/Markov time-homogeneous adversarial policies as
\begin{equation}\label{eqn:KSA_MS_def}
\begin{aligned}
\KSAM(\cP^{\mrm{SA}}) &:= \set{\kappa = (\kappa_0,\kappa_1,\ds):\kappa_t \in \KSAH(\cP^{\mrm{SA}},t) \text{ is Markov}, \forall t\geq 0} \\
\KSAS(\cP^{\mrm{SA}}) &:= \set{\kappa = (\kappa_0,\kappa_1,\ds):\kappa_t \in \KSAH(\cP^{\mrm{SA}},t) \text{ is Markov time-homogeneous}, \forall t\geq 0}. 
\end{aligned}
\end{equation}
\par Typically, the set \(\cP^{\mrm{SA}}\) in the definition of adversarial policy classes is clear from the context. Therefore, we write \(\KSAH \equiv \KSAH(\cP^{\mrm{SA}})\), with the same convention applied to \(\KSAM\) and \(\KSAS\) as well.

\par Recall the alternative notation for Markov and Markov time-homogeneous adversarial decision rule in \eqref{eqn:markov_adv_decision_rule_notation}, we recognize that the above definition can be equivalently characterized by
\begin{equation}\label{eqn:KSA_MS_alternative_def}
\begin{aligned}
\KSAM &= \set{(\kappa_0,\kappa_1,\ds):\kappa_t = \set{\kappa_{t}(\cd |s,a):(s,a) \in  S\times A}, \kappa_{t}(\cd |s,a)\in \cP_{s,a},\forall (s,a)\in  S\times A, t\geq 0}, \\
\KSAS &= \set{(\kappa,\kappa,\ds):\kappa = \set{\kappa(\cd |s,a):(s,a) \in  S\times A} , \kappa(\cd |s,a)\in \cP_{s,a},\forall (s,a)\in  S\times A}. 
\end{aligned}
\end{equation}
 We note that to preserve the consistency for later constructions (e.g. definition \eqref{eqn:prob_meas_from_pi_kappa}), \eqref{eqn:KSA_MS_alternative_def} should be understood in terms of \eqref{eqn:KSA_MS_def}. A benefit of the representation in \eqref{eqn:KSA_MS_alternative_def} is that it has a more natural interpretation and hence widely used in the literature; e.g. \citep{gonzalez2002minimax,iyengar2005robust, nilim2005robust}, and formulation (2.21) in \citet{li_shapiro2023rectangularity}. Take the Markov setting as an example, \eqref{eqn:KSA_MS_alternative_def} can be interpreted as the adversary choosing the transition kernel for each time $t$ as $p_t =\set{\kappa_t(\cd |s,a)\in\cP_{s,a}:(s,a) \in  S\times A}\in \cP^{\mrm{SA}}.$ In this paper, we will take on the notation in \eqref{eqn:KSA_MS_def} as it expresses history-dependence in an elegant way. We emphasize that this digression is only meant for a comparison with the terminologies and formulations in the literature. 
\subsubsection{S-Rectangular Set of Adversary's Policies}
\par We also consider the notion of S-rectangularity. To our knowledge, this concept is introduced to the RMDP literature independently in \citet{le_tallec2007robustMDP} and \citet{wiesemann2013robust} as a way to further limit the power of the adversary while preserving the dynamic programming principle. As exemplified in Example \ref{example:inventory_S_rec}, an S-rectangular adversary could be a more natural assumption in many relevant RMDP models in the space of operations research and management sciences in comparison to SA-rectangularity. This is often due to that an SA-rectangular adversary has too much freedom, and thus has the potential to violate some natural constraint of the modeled environment (e.g. see Example \ref{example:inventory_S_rec}). 

\par Instead of freely choosing $\kappa_t(\cd|g_{t-1},s,a)\in \cP_{s,a}$, in the S-rectangular setting, for a fixed $s\in S$, choosing adversarial decision rule  $\kappa_t(\cd|g_{t-1},s,a) = p_{s,a}$ for one $a\in A$ could affect the possible choice of $\kappa_t(\cd|g_{t-1},s,a')$ for any other $a'\in A$. Concretely, fix a prescribed non-empty set of measure-valued functions $\cP_s \subset \set{A\ra \cP(\cS)}$ for every $s\in S$. An element $p_s\in \cP_s$ can be equivalently seen as a $\R^{| A|\times | S|}$ matrix, where each row is a probability vector on $S$. Also, we denote the collection of S-rectangular ambiguity sets as $\cP^{\mrm{S}}:= \bigtimes_{s\in S} \cP_s$. 

\par Define the \textit{S-rectangular} marginal set of history-dependent adversarial decision rules induced by $\cP_s$ as
\begin{equation}\label{eqn:KSHt}
\KSH(\cP^{\mrm{S}}, t):= \set{\kappa_t: \kappa_t(\cd |g_{t-1},s,\cd) \in \cP_{s}, \forall  g_{t-1}\in\bd{G}_{t-1}, s\in S }
\end{equation}
and the S-rectangular set of history-dependent/Markov/Markov time-homogeneous adversarial policies induced by $\cP^{\mrm{S}}$ as 
\begin{equation}\label{eqn:KS_HMS_def}
\begin{aligned}
\KSH(\cP^{\mrm{S}}) &:= \set{\kappa = (\kappa_0,\kappa_1,\ds):\kappa_t \in \KSH(\cP^{\mrm{S}},t) , \forall t\geq 0} \\
\KSM(\cP^{\mrm{S}}) &:= \set{\kappa = (\kappa_0,\kappa_1,\ds):\kappa_t \in \KSH(\cP^{\mrm{S}},t) \text{ is Markov}, \forall t\geq 0} \\
\KSS(\cP^{\mrm{S}}) &:= \set{\kappa = (\kappa_0,\kappa_1,\ds):\kappa_t \in \KSH(\cP^{\mrm{S}},t) \text{ is Markov time-homogeneous}, \forall t\geq 0}. 
\end{aligned}
\end{equation}
As before, we will omit the dependence on \(\cP^{\mrm{S}}\) when it is clear from the context.

\par Similar to the SA-rectangular setting, we can use the notation identification in \eqref{eqn:markov_adv_decision_rule_notation} to represent
\begin{align*}
\KSM &= \set{(\kappa_0,\kappa_1,\ds):\kappa_t = \set{\kappa_{t}(\cd|s,a):(s,a) \in  S\times A}, \kappa_{t}(\cd|s,\cd)\in \cP_{s},\forall s\in S, t\geq 0}, \\
\KSS &= \set{(\kappa,\kappa,\ds):\kappa = \set{\kappa(\cd|s,a):(s,a) \in  S\times A} , \kappa(\cd|s,\cd)\in \cP_{s},\forall s\in S}. 
\end{align*}
Again, this should be understood as a re-interpretation of \eqref{eqn:KS_HMS_def}. In this representation, the adversary chooses transition kernels $p_t \in \cP^{\mrm{S}}$.

\begin{example}[Simple Inventory Model] \label{example:inventory_S_rec} 
Consider an inventory control problem where the next day's inventory level is $I_{t+1} = (I_{t} + a_t+D_t)_+$
where $\set{D_t:t\geq 0}$ i.i.d. is the demand process and $a_t$ is the ordered inventory. In this modeling environment, it is natural to assume that at each time $t$, the adversary can only change the law of $D_t$ dependent on $I_t$ but not dependent on the controller's action $a_t$. This leads to an S-rectangular set of adversarial decision rules. However, if we instead assume SA-rectangularity, then the adversary must be able to freely choose different laws for $D_{t}$ for different controller's action $a$; say if $a_t$ is large, then the adversary chooses the demand to be small, and if $a_t$ is small, chooses the demand to be large. Through this example, one can see that the S-rectangularity allows the modeler to limit the power of the adversary so that it cannot use the realization of the controller's action. 


\end{example}
\par Comparing the definition of $\cP_s$ and $\cP_{s,a}$, it is evident that an SA-rectangular set of adversarial policies is S-rectangular with 
\begin{equation}\label{eqn:sa-rec_cP_s}
\cP_s := \set{p_s = \set{ p_{s,a }:a\in A }: p_{s,a}\in\cP_{s,a}}  = \bigtimes_{a\in A}\cP_{s,a}
\end{equation}
the Cartesian product set. We say that $\cP_s$ is SA-rectangular if it admits a product representation as in \eqref{eqn:sa-rec_cP_s} for some $\set{\cP_{s,a}:a\in A}$. 

\subsubsection{General Rectangular Set of Adversary's Policies}
\par For completeness, we also introduce the general (a.k.a. non-rectangular in the literature, see \citet{wiesemann2013robust}) sets of adversarial decision rules. In this case, choosing adversarial decision rule  $\kappa_t(\cd|g_{t-1},s,a) = p_{s,a}$ for one pair of $(s,a)\in S\times A$ could affect the possible choice of $\kappa_t(\cd|g_{t-1},s',a')$ for any other $(s',a')\in S\times A$. 
\par Concretely, fix prescribed subset $\cP \subset \set{S\times A\ra \cP(\cS)}$ where each $p\in\cP$ is a transition kernel of the form $\set{p_{s,a}(s'):s,s'\in S,a\in A}$, which can be equivalently seen as a $\R^{|S||A|\times |S|}$ matrix. For any $t\geq 0$, Define the general rectangular set of adversarial decision rules
\begin{equation}\label{eqn:KNHt}
\KNH(\cP,t):= \set{\kappa_t: \kappa_t(\cd|g_{t-1},\cd,\cd)\in \cP,\forall g_{t-1}\in\bd{G}_{t-1}}. 
\end{equation}
The general history-dependent/Markov/Markov time-homogeneous adversarial policy classes $\KNH/\KNM/\KNS$ can be defined analogously. We omit the presentation. \cblue{We remark that the term “general rectangular” is used to highlight that the adversary remains rectangular in the sense that the history \( g_{t-1} \) does not influence the set \( \cP \) from which the adversary selects its action.}
\subsubsection{Adversary's Policy: Summary of Notations}
\par For future references, we summarize the definitions and notations concerning adversary's policy classes and discuss some elementary properties. At the heart of these definitions lie the indexed sets $\cP_{s,a}$, $\cP_{s}$, and $\cP$, all of which we will refer to as \textit{adversarial action sets}. This name comes from the fact that a SA/S/general rectangular adversary can be viewed as choosing an action within $\cP_{s,a}/\cP_{s}/\cP$, see definitions \eqref{eqn:KSAHt}/\eqref{eqn:KSHt}/\eqref{eqn:KNHt}. However, it is crucial to emphasize that in the context of finite state and action spaces, the elements within each of these sets possess varying dimensions when viewed as vectors and matrices. To clarify these concepts and their respective representations, we summarize them in Table \ref{tab:adv_action_sets}.
\begin{table}[ht]
    \centering
     \caption{Adversarial actions at time $t$ and history $g_{t-1}$ for different notions of rectangularity.} \label{tab:adv_action_sets}
    {\renewcommand{\arraystretch}{1.2}
    \begin{tabular}{llll}
    \Xhline{1.2pt}
    Rectangularity & SA & S & General  \\
    \Xhline{1.2pt}
    Adversarial Action Sets & $\cP_{s,a}\subset  \cP(\cS)$ & $\cP_s\subset\set{A\ra \cP(\cS)}$ & $\cP\subset \set{S\times A\ra \cP(\cS)}$\\
    Adversarial Action & $\kappa_t(\cd|g_{t-1},s,a) \in\cP_{s,a}$ & $\kappa_t(\cd|g_{t-1},s,\cd) \in \cP_s$ & $\kappa_t(\cd|g_{t-1},\cd,\cd)\in \cP$\\
    Dimensions (as Matrix) & $p_{s,a}: 1\times |S|$ & $p_s:|A|\times |S|$ & $p:|S||A|\times |S|$\\
    \Xhline{1.2pt}
    \end{tabular}
    }
   
\end{table}

\par It's crucial to realize from Table \ref{tab:adv_action_sets} that the concept of rectangularity is connected to the extent of historical information available to the adversary before making a decision. To illustrate, consider the following scenarios: An SA-rectangular adversary at time $t$ possesses visibility of $g_t = (g_{t-1}, s_t, a_t)$ and hence can be seen as making a choice from the action set $p_{s_t, a_t} \in \cP_{s_t, a_t}$. In contrast, an S-rectangular adversary sees only $(g_{t-1}, s_t)$ and must determine the distribution of the next state for each possible action $a_t$. Consequently, it selects an action $p_{s_t} \in \cP_{s_t}$. It is possible to further reduce the amount of historical information accessible to the adversary at the time of decision making, leading to the definition of even more general adversarial policy classes. However, such settings typically lack a valid DPP and are therefore beyond the scope of this paper. We remark that this perspective of information availability is more explicitly represented by the multistage stochastic program formulation in Section \ref{section:multistage_sp}. 

\par These distinctions emphasize how rectangularity relates to the depth of information available to the adversary at the time of decision-making. It is a reasonable inference that the extent of information available to the adversary directly correlates with its power. In this context, the SA-rectangular adversary, which has access to the most information, would be the most powerful one. This holds true when the action sets are compatible in the sense of being derived from the marginalization of a common general rectangular action set $\cP$, which we will explain later, c.f. \eqref{eqn:adv_policy_sets_order_rec}. 

\par Next, we present Table \ref{tab:adv_policy_classes} which provides a concise summary of the various adversarial policy classes and the adversarial action sets. 
\begin{table}[ht]
    \centering
    \caption{Adversary's policy classes.}\label{tab:adv_policy_classes}
    {\renewcommand{\arraystretch}{1.2}
    \begin{tabular}{c|ccc}
    \Xhline{1.2pt}
      & SA: $\set{\cP_{s,a}:(s,a)\in S\times A}$ & S: $\set{\cP_{s}:s\in S}$& General: $\cP$  \\
    \Xhline{1.2pt}
    History-dependent & $\KSAH$ & $\KSH$ & $\KNH$\\
    Markov & $\KSAM$ & $\KSM$ & $\KNM$\\
    Markov Time-homogeneous & $\KSAS$ & $\KSS$ & $\KNS$\\
    \Xhline{1.2pt}
    \end{tabular}
    }
    
\end{table}
\par An element within any of the policy classes listed in Table \ref{tab:adv_policy_classes} can be represented as $\kappa = (\kappa_0, \kappa_1, \ldots)$, where $\kappa_t, t\geq 0$ are adversarial decision rules. Consequently, it is natural to explore the inclusion relationships among these adversary policy classes. 
\par By the definitions in the previous sections, it is clear that every adversarial policy sets defined in Table \ref{tab:adv_policy_classes} is a subset of $\KH$. Moreover, ordering in terms of the memory of the adversary, we have that 
\begin{equation}\label{eqn:adv_policy_sets_order_hist}
\KSAH\supset\KSAM\supset\KSAS
\end{equation}
where $\mrm{SA}$ can be replaced with $\mrm{S}$ or $\mrm{N}$. Note that this can be memorized as ``top contains bottom" in reference to Table \ref{tab:adv_policy_classes}.  One recognizes that a large adversarial policy class implies that the adversary enjoys greater flexibility in selecting its policies, thereby becoming more powerful. Hence, by \eqref{eqn:adv_policy_sets_order_hist}, a history-dependent adversary possesses greater power compared to a Markov adversary, and, in turn, a Markov adversary is more powerful than a Markov time-homogeneous adversary.
\par Given a general rectangular adversarial action set $\cP$, we can construct S or SA-rectangular action sets $\set{\cP_s:s\in S}$ or $\set{\cP_{s,a}:s\in S,a\in A}$ by \textit{marginalization}: 
\begin{equation}\label{eqn:marginalization_of_general_P}
\cP_s := \set{p_{s} :p\in\cP  } \text{ and }\cP_{s,a} := \set{p_{s,a} :p\in\cP  }.   
\end{equation}
Let the SA and S-rectangular adversarial action sets $\set{\cP_{s,a}:s\in S,a\in A}$ and $\set{\cP_{s}:s\in S}$ be formed by via marginalizing $\cP$ as in \eqref{eqn:marginalization_of_general_P}, and we construct the adversary's policy classes $\KSAH/\KSH/\KNH$ using the action sets $\set{\cP_{s,a}:s\in S,a\in A}/\set{\cP_{s}:s\in S}/\cP$, respectively. Then, as noted before, we have the inclusion relation
\begin{equation}\label{eqn:adv_policy_sets_order_rec}
\KSAH \supset \KSH\supset \KNH
\end{equation}
where $\mrm{H}$ can be replaced with $\mrm{M}$ or $\mrm{S}$. This can be remembered as ``left contains right" by referencing Table~\ref{tab:adv_policy_classes}.

\subsection{The Max-Min Control Problem}
\par With the careful establishment of the controller's and adversary's policy classes in this section, we lay the foundation for the subsequent definition of the \textit{max-min control value} of an RMDP. The formulation of the value function allows us to formalize robust policy learning and decision-making using the RMDP framework.

\par Having rigorously defined the controller's policy $\pi$ and adversary's policy $\kappa$, we characterize the probability model for the state-action sequence $\set{(X_t, A_t): t \geq 0}$ under the policy pair $(\pi, \kappa)$. 
\begin{definition}[Induced Probability Measure]
For any pair of policies $\pi = (\pi_0,\pi_1,\ds)$ and $\kappa = (\kappa_0,\kappa_1,\ds)$, and initial distribution $\mu\in\cP(\cS)$, define $P^{\pi,\kappa}_\mu$ to be the unique probability measure on $(\Omega,\cF)$ such that
\begin{equation}\label{eqn:prob_meas_from_pi_kappa}
P^{\pi,\kappa}_\mu(G_t(\omega) = g_t) = \mu(s_0)\pi_0(a_0|s_0)\kappa_0(s_1|s_0,a_0)\pi_1(a_1|s_0,a_0,s_1)\kappa_1(s_2|s_0,a_0,s_1,a_1)\cdots \pi_t(a_t|h_t)
\end{equation}
for all $t\geq 0$ and $g_t = (s_0,a_0,\ds,s_t,a_t)\in\bd{G}_t$. 
\end{definition}
We remark that \eqref{eqn:prob_meas_from_pi_kappa} consistently defines the measure value on the cylinder sets on $\Omega$, hence uniquely extends to a measure on $(\Omega,\cF)$ by Carathéodory's extension theorem. Also, note that \eqref{eqn:prob_meas_from_pi_kappa} determines the finite dimensional distribution of the stochastic process $\set{(X_t,A_t):t\geq 0}$. 

\par Now, we are ready to define the max-min control value of an RMDP problem. Without loss of generality, we will assume that the reward function is non-negative and bounded by 1 throughout the paper; i.e. $r:S\times A\ra[0,1]$. Let $\gamma\in (0,1)$ be the discount factor. 
\begin{definition}[Max-Min Control Value]
Let $E^{\pi,\kappa}_\mu$ denote the expectation w.r.t. measure $P^{\pi,\kappa}_\mu$. We define the value of the triple $(\mu,\pi,\kappa)$ as
\[
v(\mu,\pi,\kappa) := E_\mu^{\pi,\kappa}\sqbk{\sum_{k=0}^\infty\gamma^k r(X_k,A_k)}. 
\]
For any pair controller and adversarial policy classes $(\Pi,\mrm{K})$ and initial distribution $\mu$, we define the max-min control optimal value by
\begin{equation}\label{eqn:maxmin_opt_val}
v(\mu,\Pi,\mrm{K}) := \sup_{\pi\in \Pi}\inf_{\kappa\in \mrm{K}} v(\mu,\pi,\kappa). 
\end{equation}
\end{definition}

We introduce the following set of elementary properties of the max-min control optimal value. 
\begin{lemma}\label{lemma:top->bot_right->left} Recall the definitions of the controller's and adversary's policy classes. Then, for any controller action set $\cQ$ and S-rectangular adversary action set $\cP^{\mrm{S}}$, $v(\mu,\PiH,\mrm{K})\geq v(\mu,\PiM,\mrm{K})\geq v(\mu,\PiS,\mrm{K})$
for $\mrm{K} = \KSH, \KSM,\KSS$ and $v(\mu,\Pi,\KSH)\leq v(\mu,\Pi,\KSM)\leq v(\mu,\Pi, \KSS)$
for $\Pi= \PiH,\PiM,\PiS$ as defined in \eqref{eqn:constrained_controller}.
\end{lemma}
This lemma follows from the inclusion relationships of the policy classes, see Appendix \ref{A_sec:proof:lemma:top->bot_right->left} for a proof. Note that since a SA-rectangular set of adversarial policies is S-rectangular as explained in \eqref{eqn:sa-rec_cP_s}, Lemma \ref{lemma:top->bot_right->left} is also true if we replace $ \KSH, \KSM,\KSS$ with $ \KSAH, \KSAM,\KSAS$. 
\par In this section, we have formally constructed the controller's and adversary's policy classes and the probability representation of the max-min optimal value. Our journey now leads us to address two fundamental questions: When does the optimal value satisfy a  dynamic programming principle (DPP)? Is there an optimal policy that is time- and history-independent? The former will be addressed in Section \ref{section:DPP_constrained}, while the latter will be explored in Section \ref{section:markov_time_homo_opt}.

\section{Dynamic Programming Principles}\label{section:DPP_constrained}

\par The motivation for the use of rectangular adversarial policy classes is to ensure that the max-min control problem satisfies a DPP. Specifically, rectangularity ensures that the resulting value function satisfies the robust Bellman equation, which we define next.
\cblue{
\begin{definition}[Robust Bellman Equation]
Given the constraint set of action distributions $\cQ\subset\cP(\cA)$ and S-rectangular adversarial action sets $\cP_s,s\in S$, define the robust Bellman equation as
\begin{equation}\label{eqn:dr_bellman_eqn}
u(s) = \sup_{\phi\in\cQ}\inf_{p_s\in\cP_s} E_{\phi,p_s}\sqbk{r(s,A_0 ) + \gamma u(X_1)}, \quad s\in S.
\end{equation}
Here, the law of $A_0,X_1$ under $E_{\phi,p_s}$ is defined by $P_{\phi,p_s}(A_0 = a_0,X_1 = s_1) = \phi(a_0)p_{s,a_0}(s_1)$. 
\par Another equation of significant interest is the following inf-sup equation, where the supremum and infimum operations are interchanged. This equation plays a crucial role in deducing sufficient conditions for the optimal value $v^*$ to satisfy a DPP:
\begin{equation}\label{eqn:dr_minmax_bellman_eqn}
u'(s) =\inf_{p_s\in\cP_s} \sup_{\phi\in\cQ} E_{\phi,p_s}\sqbk{r(s,A_0 ) + \gamma u'(X_1)}, \quad s\in S.
\end{equation}
\end{definition}
}
In the context of finite state and action spaces RMDPs, a bounded solution of the robust Bellman equation \eqref{eqn:dr_bellman_eqn} always exists and is unique: 

\begin{proposition}\label{prop:bellman_exst_unq_sol}
There exists a unique solution $u^*$ to \eqref{eqn:dr_bellman_eqn}. Moreover, $\|u^*\|_\infty\leq 1/(1-\gamma)$. 
\end{proposition}
\par This proposition follows from the well-known contraction property of the robust Bellman operator (see, for example, Theorem 4 in \citet{wiesemann2013robust}). We include a proof in Appendix \ref{A_sec:proof:prop:bellman_exst_unq_sol} to make the work self-contained. Having confirmed the existence of $u^*$, we define the satisfaction of the DPP as the max-min optimal value being identifiable with $u^*$ in the following sense. 

\begin{definition}[Dynamic Programming Principle]\label{def:DPP}
Given the action distribution sets $\cQ$ and $\cP^{\mrm{S}} = \set{\cP_s:s\in S}$, we say that the pair of controller and adversary's policy class $(\Pi,\mrm{K})$ satisfies the \textit{dynamic programming principle} (DPP) if the solution $u^*$ to the robust Bellman equation \eqref{eqn:dr_bellman_eqn} satisfies $E_\mu[u^*(X_0)] = v(\mu,\Pi,\mrm{K})$ for all $\mu\in\cP(\cS)$. 
\end{definition}
As highlighted in the introduction, the significance of the satisfaction of a DPP in the context of DRRL lies in the fundamental role played by the Bellman equation, which serves as the primary computational tool underlying nearly all RL algorithms.  \cblue{In particular, all existing DRRL algorithms achieve policy learning by representing the value function through fixed-point equation \eqref{eqn:dr_bellman_eqn} or \eqref{eqn:dr_minmax_bellman_eqn}.}

\par With the DPP defined, we proceed to systematically explore the satisfaction of a DPP under various assumptions regarding the policy classes of the controller and the adversary. A concise summary of the results can be found in the tables provided earlier and Section \ref{section:DPP_summary}.

\subsection{Max-Min Optimal Values and Bellman Equation: The General Case}
We first consider the most general setting formulated in this paper where the action sets of the controller $\cQ\subset \cP(\cA)$ and the S-rectangular adversary $\cP_s\subset \set{A\ra\cP(\cS)}$ are arbitrary prescribed subsets. At this level of generality, the following theorem holds: 
\cblue{\begin{theorem}\label{thm:s-rec_constrained_checks}
Let $u^*:S\ra \R_+$ be the unique solution of \eqref{eqn:dr_bellman_eqn}. Then, 
\begin{align*}
&E_\mu[u^*(X_0)] &&= v(\mu,\PiH,\KSH )\\
&&&= v(\mu,\PiM,\KSH) &&= v(\mu,\PiM,\KSM)\\
&&&= v(\mu,\PiS,\KSH) &&= v(\mu,\PiS,\KSM)&&= v(\mu,\PiS,\KSS)
\end{align*}
for all $\mu\in\cP(\cS)$; i.e. the DPP holds for these cases. 
\end{theorem}}

\cblue{The proof of Theorem \ref{thm:s-rec_constrained_checks} relies on two important facts. First, the DPP always holds for controllers and adversaries with \textit{symmetric} information structures. Second, $E_\mu[u^*(X_0)]$ is always a lower bound for the optimal maxmin control values regardless of information asymmetries and constraints. With these two facts and Lemma \ref{lemma:top->bot_right->left}, one can directly conclude Theorem \ref{thm:s-rec_constrained_checks}.} The detailed proof is deferred to Appendix \ref{A_sec:proof:thm:s-rec_constrained_checks}. 
\par Since a SA-rectangular set of adversarial policies is S-rectangular as constructed in (\ref{eqn:sa-rec_cP_s}), the same result as in Theorem \ref{thm:s-rec_constrained_checks} is also true for the SA-rectangular adversaries. Therefore, Theorem \ref{thm:s-rec_constrained_checks} resolves the existence of the DPP for all $6\times4$ cases corresponding to the lower triangular portion of Tables \ref{tab:sa-rec} - \ref{tab:s-rec_det}. 

\subsubsection{Deterministic Controller Policies}\label{section:DPP_det_controller}
\par An important special case is when $\cQ = \cQ^{\mrm{D}} = \set{\delta_{\set{a}}:a\in A}$. A deterministic rule is naturally uniquely identified with a function $\pi_t:\bd{H}_t\ra A$. We will always assume such identification and use $\pi_t(h_t)$ as a delta measure or an action interchangeably. 
\par As clarified in \cite{wiesemann2013robust}, in the S-rectangular setting with $\cQ = \cQ^\mrm{D}$ for $\Pi$ and $\cQ = \cP(\cA)$ for $\Pi$, $v(\mu,\Pi,\mrm{K})< v(\mu,\Pi,\mrm{K})$ could happen for $\mrm{K} = \KSH,\KSM,\KSS$. Thus, there could be no deterministic policy that is optimal for the controller that is allowed to take randomized actions. However, if one is to constrain the controller to take deterministic action, it is still instrumental to define the $q$-function and its Bellman equation as in the classical MDP settings. 
\cblue{
\begin{definition}
Define the robust Bellman equation for the $q$-function as 
\begin{equation}\label{eqn:dr_bellman_eqn_q_original}
q(s,a) = \inf_{p_{s}\in\cP_s} E_{\delta_a,p_s} \sqbk{r(s,A_0) +  \gamma\max_{b\in A}q(X_1,b)}.
\end{equation}
\end{definition}

It is important to note that for each $s$, the infimum chooses a $p_{s} = \set{p_{s,a}:a\in A}$ for each $a\in A$; i.e., such a choice could be different across $a$. Effectively, for different controller's action $a\in A$, 
\begin{equation}\label{eqn:dr_bellman_eqn_q}
q(s,a) = r(s,a ) +\inf_{\psi\in\set{p_{s,a}:p_s\in \cP_s }} E_\psi\sqbk{ \gamma \max_{b\in A}q(X_1,b)},
\end{equation} 
where $X_1\sim\psi$ under $E_\psi$. In particular, $\set{p_{s,a}:p_s\in \cP_s }$ is the marginalization (see \eqref{eqn:marginalization_of_general_P}) of $\cP_s$. 
\par With the robust Bellman equation for the $q$-function defined, we explore the connection of its solution with that of the Bellman equation of the value function \eqref{eqn:dr_bellman_eqn}. Let $u^*$ denote the unique solution of \eqref{eqn:dr_bellman_eqn}. Then, we define the optimal $q$-function as follows: for each $(s,a)\in S\times A$
\[
q^*(s,a) =   \inf_{p_{s}\in\cP_s} E_{\delta_a,p_s} \sqbk{r(s,A_0) +  \gamma u^*(X_1)} = r(s,a ) +  \inf_{\psi\in\set{p_{s,a}:p_s\in \cP_s }} E_\psi\sqbk{\gamma u^*(X_1)}. 
\]
}

\par The following corollary links $q^*$ and $u^*$ to the solution of \eqref{eqn:dr_bellman_eqn_q}, and hence establishes an alternative equivalent formulation of the DPP in the case of a deterministic controller.
\begin{corollary}\label{cor:det_ctrl_q_func}
$q^*$ is the unique solution to \eqref{eqn:dr_bellman_eqn_q}. Moreover, if $\cQ = \cQ^{\mrm{D}}$, then $u^*(\cd) = \max_{a\in A} q^*(\cd,a)$ and Theorem \ref{thm:s-rec_constrained_checks} holds. 
\end{corollary}

\par  Corollary \ref{cor:det_ctrl_q_func} verifies the six check marks in Table \ref{tab:s-rec_det}.  Subsequently, we will examine and establish the optimality attributed to the greedy policies derived from the optimal $q^*$-function in Remark \ref{rmk:greedy_policy_from_q}. Moreover, we will provide counterexamples in Section \ref{section:counterexample_det_ctrl}, substantiating the red crosses within Table \ref{tab:s-rec_det}. The proof of Corollary \ref{cor:det_ctrl_q_func} is provided in Appendix \ref{A_sec:proof:cor:det_ctrl_q_func}.

\subsection{The Inf-Sup Equation and the Minimax Theorem}
In this section, we revisit the setting where \(\cQ \subset \cP(\cA)\) is an arbitrary subset of action distributions. Our focus now shifts to establishing the DPP under an additional assumption that the solution \( u^* \) of \eqref{eqn:dr_bellman_eqn} also satisfies \eqref{eqn:dr_minmax_bellman_eqn}. 

\par \cblue{Besides establishing the DPP within Table \ref{tab:s-rec_conv}, we explore a related fundamental question: At this level of generality, does the optimal maxmin control value defined in \eqref{eqn:maxmin_opt_val} always satisfy a minimax theorem? In other words, can the sup and inf operations over the respective entire policy classes in \eqref{eqn:maxmin_opt_val} always be interchanged without altering the value? We provide an affirmative answer in Theorem \ref{thm:s-rec_constrained_max_min_eq_min_max}.}
\begin{theorem}\label{thm:s-rec_constrained_max_min_eq_min_max} 
Let $u^*$ be the solution of the robust Bellman equation \eqref{eqn:dr_bellman_eqn}. Assume $u^*$ further satisfies the inf-sup equation \eqref{eqn:dr_minmax_bellman_eqn}. 
Then, the following properties holds for all $\mu\in\cP(\cS)$ and any combination of $\mrm{K} = \KSH, \KSM,\KSS$ and $\Pi = \PiH, \PiM,\PiS$:
\begin{enumerate}
    \item 
    The optimal values satisfy $v(\mu,\Pi,\mrm{K}) = E_\mu[u^*(X_0)]$.
    \item The interchange of sup-inf preserves the optimal values: 
    \[
    v(\mu,\Pi,\mrm{K}) = \sup_{\pi\in\Pi} \inf_{\kappa\in\mrm{K}} v(\mu,\pi,\kappa)= \inf_{\kappa\in\mrm{K}}\sup_{\pi\in\Pi} v(\mu,\pi,\kappa). 
    \] 
\end{enumerate}
\end{theorem}
We emphasize that, in particular, Theorem \ref{thm:s-rec_constrained_max_min_eq_min_max} implies that if the supremum and the infimum in the Bellman equation \eqref{eqn:dr_bellman_eqn} interchange, then so do the sup and inf in the optimal value function. Theorem \ref{thm:s-rec_constrained_max_min_eq_min_max} is proved in Appendix \ref{A_sec:proof:thm:s-rec_constrained_max_min_eq_min_max(1)}. \cblue{The proof is constructive. Intuitively, the sup-inf Bellman equation \eqref{eqn:dr_bellman_eqn} is used to identify a candidate optimal time-homogeneous control policy $\pi^*$, while the inf-sup equation \eqref{eqn:dr_minmax_bellman_eqn} yields a candidate worst-case time-homogeneous adversary strategy $\kappa^*$. Theorem \ref{thm:s-rec_constrained_max_min_eq_min_max} then follows from the observation that if both equations hold, $\pi^*$ attains the maximum of $\pi \mapsto \inf_{\kappa \in \mrm{K}} v(\mu, \pi, \kappa)$, and $\kappa^*$ attains the minimum of $\kappa \mapsto \sup_{\pi \in \Pi} v(\mu, \pi, \kappa)$, for any combination of policy classes $\mrm{K} \in \{\KSH, \KSM, \KSS\}$ and $\Pi \in \{\PiH, \PiM, \PiS\}$.}

\subsubsection{Sion's Minimax Principle}
\par In light of the exchange of the order of inf-sup in \eqref{eqn:dr_bellman_eqn} and \eqref{eqn:dr_minmax_bellman_eqn}, a natural sufficient condition that can guarantee the assumption of Theorem \ref{thm:s-rec_constrained_max_min_eq_min_max} can be deduced from Sion's minimax principles \citep{Sion1958minmax}. 
\cblue{\par We consider mappings $w_s^*:\cQ\times\cP_s\ra\R$ for every $s\in S$ defined by
\begin{equation}\label{eqn:def_ws}
w_s^*(\phi,p_s) =  E_{\phi,p_s} [r(s,A_0 ) + \gamma u^*(X_1)].   
\end{equation}
Notice that, seeing $\phi$ as a vector and $p_s$ as a matrix, we have
\[w_s^*(\phi,p_s) = \sum_{a\in A} \phi(a)\sqbk{r(s,a)
+\gamma \sum_{ s'\in S}p_{s,a}(s')u^*(s')}.\]
This suggest that $w_s^*(\cd,p_s)$ is linear and $w_s^*(\phi,\cd)$ is affine.  
Therefore, we can apply Sion's minimax theorem if $\cQ$ and $\cP_s$ satisfy appropriate convexity and compactness assumptions. }
\begin{definition}\label{def:convex_action_dist_trans_meas}
We say that a set of action distributions $\cQ\subset\cP(\cA)$ is convex if for all $\phi,\phi'\in \cQ$,  $\set{t\phi+(1-t)\phi':t\in[0,1]}\subset \cQ.$ Moreover, the set of adversarial decision rules $\cP_s$ is convex if for all $p_s,p_s'\in\cP_s$, 
$\set{tp_{s}+(1-t)p_s':t\in[0,1]}\subset \cP_s.$
\end{definition}
\begin{corollary}[Convexity and Compactness]\label{cor:conv_comp_ctrl_adv}
Assume that $\cQ$ and $\cP_s : s\in S$ are convex in the sense in Definition \ref{def:convex_action_dist_trans_meas}. Moreover, assume that either $\cQ$ or all $\cP_s:s\in S$ are compact. Then, the solution $u^*$ of \eqref{eqn:dr_bellman_eqn} satisfies \eqref{eqn:dr_minmax_bellman_eqn}. Hence, the conclusions of Theorem \ref{thm:s-rec_constrained_max_min_eq_min_max} hold. 
\end{corollary}

\begin{proof}{Proof of Corollary \ref{cor:conv_comp_ctrl_adv}}
A direct application of Sion's minimax theorem (Corollary 3.3 in \cite{Sion1958minmax}) to each of $\set{w_s^*:s\in S}$ yields
$u^*(s) =\sup_{\phi\in\cQ}\inf_{p_s\in\cP_s} w_s^*(d,p_s) =\inf_{p_s\in\cP_s} \sup_{\phi\in\cQ} w_s^*(d,p_s)$
for all $s\in S$. Therefore, $u^*$ satisfies \eqref{eqn:dr_minmax_bellman_eqn}. 
\end{proof}
We note that Corollary \ref{cor:conv_comp_ctrl_adv} validates all nine check marks in Table \ref{tab:s-rec_conv}.
\begin{remark}
The convexity and compactness assumption in Corollary \ref{cor:conv_comp_ctrl_adv} is a sufficient condition for $u^*$ satisfying both \eqref{eqn:dr_bellman_eqn} and \eqref{eqn:dr_minmax_bellman_eqn}. Many S-rectangular RMDP models of engineering systems naturally have a convex adversary under mild assumptions, c.f. the inventory model in Section 4 of \citet{shapiro2022distributionally}.
\par Yet, the convexity is not a necessary condition for the interchange in \eqref{eqn:dr_bellman_eqn} and \eqref{eqn:dr_minmax_bellman_eqn} to hold. As we will see in the next section, in the important special setting of SA-rectangular adversarial policy classes, the DPP always holds, even if both the controller's and the adversary's action set are non-convex.
\end{remark}

\subsubsection{The SA-Rectangular Setting}
Next, we focus on the important special case where the adversary is SA-rectangular with sets of adversarial decision rules $\set{\cP_{s,a}\in \cP(\cS):s\in S,a\in A}$. Note that the SA-rectangular setting, in general, doesn't satisfy Corollary \ref{cor:conv_comp_ctrl_adv} as the set $\cP_s = \bigtimes_{a\in A}\cP_{s,a}$ can be non-convex if $\cP_{s,a}$ is not convex. Moreover, we allow the controller to be non-convex as well; e.g. $\cQ = \cQ^{\mrm{D}}$. However, the sup-inf in the Bellman equation does interchange. A generalization of this is stated below. 
\begin{theorem}\label{thm:sa-rec_dpp}
Suppose $\set{\cP_s:s\in S}$ are SA-rectangular (c.f. \eqref{eqn:sa-rec_cP_s}) and $\cQ^\mrm{D}\subset\cQ$. Let $u^*$ be the solution of \eqref{eqn:dr_bellman_eqn}. Then 
\begin{equation}\label{eqn:SA_q_func}
q^*(s,a) = r(s,a) +\gamma \inf_{\psi\in\cP_{s,a}} E_{\psi}[u^*(X_1)]
\end{equation}
solves the Bellman equation of the $q$-function \eqref{eqn:dr_bellman_eqn_q} and $u^*(\cd) = \max_{a\in A}q^*(\cd,a)$. Moreover, $u^*$ solves \eqref{eqn:dr_minmax_bellman_eqn} and hence conclusions of Theorem \ref{thm:s-rec_constrained_max_min_eq_min_max} holds. 
\end{theorem}
Theorem \ref{thm:sa-rec_dpp} implies that if $\cP_s= \bigtimes_{a\in A}\cP_{s,a}$ for all $s\in S$ are SA-rectangular and the controller is allowed to take on deterministic policies, then a DPP always holds with deterministic decision rules. We defer the proof of this theorem to Appendix \ref{A_sec:proof:thm:sa-rec_dpp}. 
\begin{remark}
In either cases when $\cQ = \cP(\cA)$ or $\cQ = \cQ^\mrm{D}$, the assumptions of Theorem \ref{thm:sa-rec_dpp} holds. This proves the validity of all nine cases in Table \ref{tab:sa-rec}. We note that in the SA-rectangular setting, the enlargement of $\cQ$ to non-deterministic controller actions doesn't change the optimal value of the max-min control problem. This is not true in general for S-rectangular adversaries, as pointed out in Section \ref{section:DPP_det_controller}.
\end{remark}

\subsection{Convex Set of Controller's Action}
We have proved the validity of the check marks for Table \ref{tab:s-rec_det}. One notices the difference between Table \ref{tab:s-rec_randomized} and \ref{tab:s-rec_det}, where the (1,2)th entry, i.e. history-dependent controllers with Markov adversaries, becomes a check. In this section, we prove its validity by establishing a general principle for controllers with a convex set of action distributions $\cQ$ defined in Definition \ref{def:convex_action_dist_trans_meas}. The proof is deferred to Appendix \ref{A_sec:proof:thm:s-rec_conv_constrained_checks}. 
\begin{theorem}\label{thm:s-rec_conv_constrained_checks} Let the set of action distributions $\cQ$ be convex, and $u^*$ be the solution to \eqref{eqn:dr_bellman_eqn}. Then, we have that
$E_\mu[u^*(X_0)] = v(\mu,\PiH,\KSM)$ in addition to the equalities in Theorem \ref{thm:s-rec_constrained_checks}.   
\end{theorem}

\begin{remark}
As the probability simplex is convex, the case $\cQ = \cP(\cA)$ satisfies Theorem \ref{thm:s-rec_conv_constrained_checks}. Therefore, in conjunction with Theorem \ref{thm:s-rec_constrained_checks}, we have established all seven check marks in Table \ref{tab:s-rec_randomized}.  The counterexamples for the remaining two cases are then presented in Section \ref{section:counterexample}. 
\end{remark}

\subsection{Satisfaction of the DPP: A Summary}\label{section:DPP_summary}
We can summarize the results by providing structural insights into the conditions under which a DPP (see Definition \ref{def:DPP}) holds. The discussion in this section can be approached from two perspectives: one in terms of the convexity of the action sets of the controller and the adversary (in the sense of Definition \ref{def:convex_action_dist_trans_meas}), and the other regarding the interchangeability of the sup-inf operations in the Bellman equation \eqref{eqn:dr_bellman_eqn}.
\begin{enumerate}
\item Convexity perspective:

\begin{itemize}
\item A DPP always holds for the six cases in Theorem \ref{thm:s-rec_constrained_checks}, where neither the controller's nor the adversary's action set is assumed to be convex. An important setting with a non-convex controller action distribution set is $\cQ = \cQ^{\mrm{D}}$. In this context, Corollary \ref{cor:det_ctrl_q_func} characterizes the solution of the Bellman equation using the $q$-function.
\item If the controller action distribution set $\cQ$ is convex, then Theorem \ref{thm:s-rec_conv_constrained_checks} implies the additional DPP $E_\mu[u^*(X_0)] = v(\mu,\PiH,\KSM)$ where the controller is history-dependent and the adversary is Markov. 
\item When both the controller's and the adversary's action sets are convex and at least one is compact, Sion's minimax principle holds for the robust Bellman equation. In this case, the DPP invariably holds. 
\item However, if the controller's action set is non-convex, even with a convex adversary, the DPP generally holds only for the six cases specified in Theorem \ref{thm:s-rec_constrained_checks}. Counterexamples for the other three cases can be found by considering $\cQ = \cQ^{\mrm{D}}$, as detailed in Section \ref{section:counterexample_det_ctrl}.
\end{itemize}

\item Sup-inf interchangeability perspective:
\begin{itemize}
    \item If the robust Bellman equation shares the same solution as the interchanged equation~\eqref{eqn:dr_minmax_bellman_eqn}, then the DPP holds, and the order of the two players can be interchanged in the dynamic decision-making environment characterized by the value~\eqref{eqn:maxmin_opt_val}.
    \item An important special case where this interchange is always valid occurs when the adversary is SA-rectangular and the controller is permitted to use any deterministic policies, as established in Theorem~\ref{thm:sa-rec_dpp}.
\end{itemize}
\end{enumerate}
The two perspectives intersect when Sion's minimax principle is satisfied by the robust Bellman equation. In such cases, both players have convex action sets, and the interchange of sup-inf operations is valid. Consequently, the DPP holds for all nine cases, see Corollary \ref{cor:conv_comp_ctrl_adv}.

\section{Optimality of Markov Time-Homogeneous Policies}\label{section:markov_time_homo_opt}
We consider the general S-rectangular setting where $\cQ\subset\cP(\cA)$ and $\cP_s\subset\set{A\ra\cP(\cS)}$ are arbitrary subsets. Let $u^*$ be the unique solution to \eqref{eqn:dr_bellman_eqn}. 
\cblue{
\begin{definition}[$\eta$-optimal Decision Rule]\label{def:eps_opt_from_bellman}
A decision rule $\Delta = \set{\Delta(\cd|s)\in\cQ:s\in S}$ is $\eta$-optimal for the Bellman equation \eqref{eqn:dr_bellman_eqn} if for all $s\in S$, $u^*(s) \leq \inf_{p_s\in\cP_s}E_{\Delta(\cd|s),p_s}\sqbk{r(s,A_0 ) +  \gamma u^*(X_1)} + \eta. $
\end{definition}}

\begin{theorem}\label{thm:s-rec_markov_opt}
Let $u^*$ be the solution to \eqref{eqn:dr_bellman_eqn} and $\eta \geq 0$. Assume that  
$E_\mu[u^*(X_0)] = v(\mu,\Pi,\mrm{K})$
\cblue{where \( \Pi \) represents \( \PiH \), \( \PiM \), or \( \PiS \), and \( \mrm{K} \) corresponds to \( \KSH \), \( \KSM \), or \( \KSS \). } Then, any $\eta$-optimal decision rule $\Delta$ for \eqref{eqn:dr_bellman_eqn} induces a $\eta/(1-\gamma)$-optimal policy $\pi:=(\Delta,\Delta,\ds )\in\PiS$ for the maxmin control problem $v(\mu,\Pi,\mrm{K})$ in the sense that for all $\mu\in\cP(\cS)$, 
$0\leq  v(\mu,\Pi,\mrm{K}) - \inf_{\kappa\in \mrm{K}}v(\mu,\pi,\kappa) \leq \frac{\eta}{1-\gamma}.$\end{theorem}
The proof of this theorem is provided in the Appendix \ref{A_sec:proof:thm:s-rec_markov_opt}.
\begin{remark}\label{rmk:greedy_policy_from_q}
Here, we intentionally allow $\eta = 0$; i.e., if the assumptions in Theorem \ref{thm:s-rec_markov_opt} hold and the solution of the Bellman equation induces a decision rule $\Delta = \set{\Delta(\cd|s):s\in S}$ that achieve the supremum in \eqref{eqn:dr_bellman_eqn}, then the Markov time-homogeneous policy $\pi := (\Delta,\Delta,\ds)$ is optimal for the maxmin control problem. 

\par For the cases either $\cQ = \cQ^{\mrm{D}}$ or the adversary policy class is SA-rectangular, Corollary \ref{cor:det_ctrl_q_func} and Theorem \ref{thm:sa-rec_dpp} implies that any greedy deterministic decision rule $\pi(s) \in \argmax{a\in A}q^*(s,a)$ is $0$-optimal for the Bellman equation. Therefore, if the DPP holds (a green check in Table \ref{tab:s-rec_det} or \ref{tab:sa-rec}), by Theorem \ref{thm:s-rec_markov_opt}, the corresponding greedy policy $(\pi,\pi,\ds)\in\PiS$ is optimal for the maxmin control problem. 
\end{remark}

The preceding remarks clarify the settings in which Markov time-homogeneous policies are optimal. Furthermore, Theorem \ref{thm:s-rec_markov_opt} implies that for RMDP models satisfying a DPP, policies from the Markov time-homogeneous class can achieve performance arbitrarily close to optimal. As a result, in these settings, a DRRL procedure that approximates the solution to \eqref{eqn:dr_bellman_eqn} is guaranteed to deliver strong performance in solving the max-min control problem.

\section{Absence of the DPP: Counterexamples}\label{section:counterexample}

\text{I}n this section, we provide counterexamples in the form of specific RMDP instances for which the DPP does not hold. \cblue{The guiding insight is that DPP failure emerges precisely when the asymmetric information structures enable the controller to learn salient features of the worst-case adversary by experimenting over time. If the information structure lets the controller treat the interaction as a sequential exploration problem, much like a multi-armed bandit or RL task, then adaptive learning destroys the Markov optimality implied by a DPP. Viewing the controller through this ``learning agent" lens proved crucial in designing the counterexamples that delineate the boundary between DPP validity and breakdown.}

Concretely, this section is organized as follows:
\text{I}n Section \ref{table3_1}, we consider the setting where a history-dependent randomized controller is paired with a Markov time-homogeneous non-convex adversary, as indicated by the cross mark at (1,3)th entry in Table \ref{tab:s-rec_randomized}. \text{I}n Section \ref{table3_2}, we consider an instance with a Markov randomized controller and a Markov time-homogeneous non-convex adversary, represented by the cross mark at (2,3)th entry in Table  \ref{tab:s-rec_randomized}. Finally, in Section \ref{table4:counter}, we establish counterexamples for all the cases with deterministic controllers in Table \ref{tab:s-rec_det} for which a DPP is not proven. 
\subsection[History-dependent Randomized Controller Versus Markov Time-Homogeneous Non-convex Adversary]{History-dependent Randomized Controller Versus\\ Markov Time-Homogeneous Non-convex Adversary}
\label{table3_1}
\begin{figure}[htb]
	\centering
	\subfigure[$p^{(1)}$]{
		\scalebox{0.8}{
			\begin{tikzpicture}
				\node[state]                               (A) {\text{I}};
				\node[state,right=of A,yshift=1cm,xshift=0.5cm]                   (B) {G};
				\node[state,right=of A,yshift=-1cm,xshift=0.5cm]                    (C) {B};

				\path(A)   edge[blue] node[above] {}  (B);
				
				\path (A)   edge[red] node[above] {}   (C);
				\path 	(B.north)   edge[bend right=60]    (A.north);
				\path 	(C.south)   edge[bend left=60]    (A.south);

	\end{tikzpicture}}} \hspace{1in}
	\subfigure[$p^{(2)}$]{
		\scalebox{0.8}{
			\begin{tikzpicture}
				\node[state]                               (A) {\text{I}};
				\node[state,right=of A,yshift=1cm,xshift=0.5cm]                   (B) {G};
				\node[state,right=of A,yshift=-1cm,xshift=0.5cm]                    (C) {B};

				\path(A)   edge[red] node[above] {}  (B);
				
				\path (A)   edge[blue] node[above] {}   (C);
				\path 	(B.north)   edge[bend right=60]    (A.north);
				\path 	(C.south)   edge[bend left=60]    (A.south);

	\end{tikzpicture}}} 
	\caption{The adversarial actions in the adversary's action distribution set, where the red
		line and the blue line represent actions $a_1$ and $a_2$, respectively.}
  \label{table3_13_transition}
\end{figure}
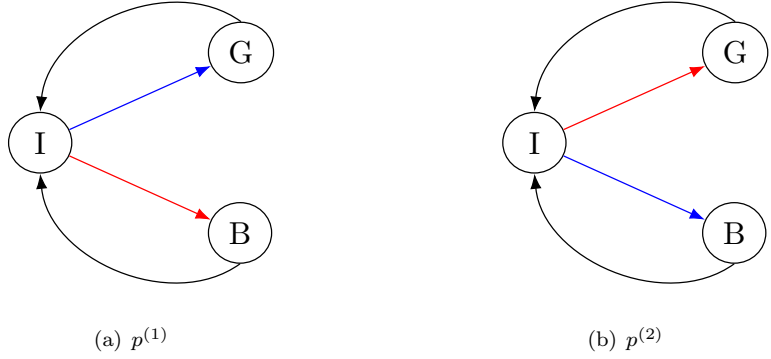

We display the RMDP instance used as the counterexample in Figure \ref{table3_13_transition}. This is used to verify the cross mark at (1,3)th entry in Table \ref{tab:s-rec_randomized}. 
\par There are three states: \text{I} (initial state), G (good state), and B (bad state). The MDP starts from the initial state \text{I}, where there are two actions: ${A} = \{a_1,a_2\}$. There is only one action in the states G and B. 

\cblue{
The adversary's action distribution set $\cP_\text{I}:=\set{p_\text{I}^{(1)},p_\text{I}^{(2)}}$ in the $\text{I}$ state has two distributions characterized by
$
	p^{(1)}_{\text{I},a_{1}}(B) =1, p^{(1)}_{\text{I},a_{2}}(\text{G})=1, \text{ and }
	p^{(2)}_{\text{I},a_{1}}(\text{G}) =1, p^{(2)}_{\text{I},a_{2}}(\text{B})=1, 
$
Here, \( p^{(1)}_{\text{I},a_{1}}(\text{B}) = 1 \) means that in state \(\text{I}\), if the controller selects action \( a_1 \) and the adversary chooses action \( p_\text{I}^{(1)} \), the MDP transitions to state B with probability 1. The other terms—\( p^{(1)}_{\text{I},a_{2}}(\text{G}) \), \( p^{(2)}_{\text{I},a_{1}}(\text{G}) \), and \( p^{(2)}_{\text{I},a_{2}}(\text{B}) \)—as well as other similar terms that follow, have an analogous interpretation.

Furthermore, in states G and B,  the controller's action set contains only one action, and the adversary's uncertainty set includes only a single distribution, which always transitions back to the same state I. That is, %
$
	p_\text{G}(\text{I}) = p_\text{B}(\text{I}) = 1.
$
We assume that reward function $r$ only depends on the states so that
$
	r(\text{I},a_1)=r(\text{I},a_2)=r(\text{I})=0, r(\text{G})=1,r(\text{B})=-1.
$
}

We consider a history-dependent policy $\pi$;\ at time 0, we randomly pick
one action at state \text{I}. \text{I}f we observe state G, we will choose the same action for the
following time steps. If we observe state B, we will choose the alternative action for the following time steps. \cblue{We do not aim to prove that such a history-dependent policy $\pi$ is optimal; rather, we will only show that this policy achieves a higher value than what is implied by the DPP, which means the DPP is not satisfied.}

\par \cblue{ For any Markov time-homogeneous adversary's policy $\kappa$,  we observe that at even time steps, the controller will always return to state I and receive a reward of 0. At odd time steps, the controller will receive a reward of either 1 or -1 with equal probability at step 1, and for steps 
$\geq 3$, the controller will always receive a reward of 1.
Therefore, we have
 $v(\delta_\text{I},\pi,\kappa ) =\sum_{i=3,i \text{ odd}}^\infty \gamma^i=\frac{\gamma ^{3}}{ 		1-\gamma ^{2}}>0,$ }
 where $\delta_\text{I}$ is the point mass measure at \text{I}. 
On the other hand, the robust Bellman equation corresponding to this instance is 
\begin{equation}
	u(\text{I}) = \sup_{\phi\in \mathcal{P}(\mathcal{A})}\inf_{p_\text{I}\in \mathcal{P}_\text{I}
  } \left( 0+\gamma E_{\phi,p_\text{I}} [u(X_1)] \right), \quad
	u(\text{G}) =1+\gamma u(\text{I}), \quad
	u(\text{B}) =-1+\gamma u(\text{I}).
 \label{table31_dpp_1}
\end{equation}%
Solving the equations, we have that the solution $u^*$ to (\ref{table31_dpp_1}) is 
\[
u^*(\text{I})=0, u^*(\text{G})=1, u^*(\text{B})=-1.
\]
Notice that $u^*(\text{I}) < v(\delta_\text{I},\pi,\kappa )$. 
Therefore, the DPP is not satisfied. 

In this case, due to the Markov time-homogeneity of the adversary, the task of the controller becomes a bandit learning problem \citep{lattimore2020bandit}. Therefore, it is possible for the controller to implement a history-dependent policy that optimally ``learns'' the environment.   
\subsection[Markov Randomized  Controller Versus
Markov Time-Homogeneous Non-convex Adversary]{Markov Randomized  Controller Versus\\
Markov Time-Homogeneous Non-convex Adversary}
\label{table3_2}
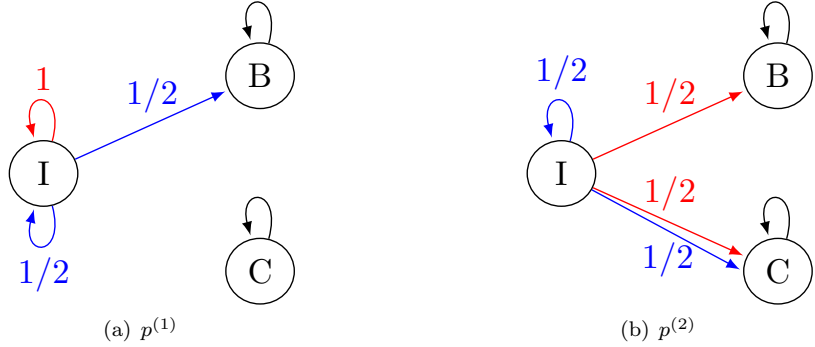
\begin{figure}[htb]
	\centering
	\subfigure[$p^{(1)}$]{
		\scalebox{0.8}{
			\begin{tikzpicture}
				\node[state]                               (A) {\text{I}};
				\node[state,right=of A,yshift=1cm,xshift=0.5cm]                   (B) {B};
				\node[state,right=of A,yshift=-1cm,xshift=0.5cm]                    (C) {C};

				\draw[
				>=latex,
				auto=right,                          loop above/.style={out=75,in=105,loop},
				every loop,
				]
				
				(A)   edge[blue]             node[above] {$1/2$}  (B)
				
				(A)   edge[loop above,red] node[above] {$1$}   (A)
				(A)   edge[loop below,blue] node {$1/2$}   (A)
				(B)   edge[loop above] node {}  (B)
				(C)   edge[loop above] node {}   (C);
	\end{tikzpicture}}} \hspace{1in}
	\subfigure[$p^{(2)}$]{
		\scalebox{0.8}{
			\begin{tikzpicture}
				\node[state]                               (A) {\text{I}};
				\node[state,right=of A,yshift=1cm,xshift=0.5cm]                   (B) {B};
				\node[state,right=of A,yshift=-1cm,xshift=0.5cm]                    (C) {C};

				\draw[
				>=latex,
				auto=right,                      	loop above/.style={out=75,in=105,loop},
				every loop,
				]
				
				(A)   edge[red]             node[above] {$1/2$}  (B)
				(A)   edge[red]             node[above] {$1/2$}   (C)
				(A)   edge[blue]             node[below] {$1/2$}   (C.west)
				(A)   edge[loop above,blue] node[above] {$1/2$}   (A)
				(B)   edge[loop above] node {}  (B)
				(C)   edge[loop above] node {}   (C);
	\end{tikzpicture}}}
	\caption{The adversarial actions in the action distribution set, where the red
		line and the blue line represent actions $a_1$ and $a_2$, respectively.}
  \label{table3_23_transition}
\end{figure}
In Figure \ref{table3_23_transition}, we illustrate the transition structure of the counterexample we constructed for verifying the cross at (2,3)the entry of Table \ref{tab:s-rec_randomized}.
In this example, we consider a Markov chain with three states \text{I},B,C. The chain starts at state \text{I}, where there are two actions: ${A}=\{a_{1},a_{2}\}.$ \text{I}n states B,C, there is only one
action.

The adversary's action distribution set $\cP_\text{I}:=\set{p_\text{I}^{(1)},p_\text{I}^{(2)}}$ at state \text{I} contains two probability distributions:%
\begin{eqnarray*}
	p^{(1)}_{\text{I},a_1} =(1,0,0),\quad  p^{(1)}_{\text{I},a_2}=\left( \frac{1}{2%
	},\frac{1}{2},0\right); \text{ and}\quad
	p^{(2)}_{\text{I},a_1} =\left( 0,\frac{1}{2},\frac{1}{2}\right) ,\quad
	p^{(2)}_{\text{I},a_2}=\left( \frac{1}{2},0,\frac{1}{2}\right),
\end{eqnarray*}%
where the transition probabilities to states \text{I},B,C given the current state and action are represented as row vectors. Further, we notice that once the chain arrives state B or state C, it will stay in that state forever, i.e., 
$
p_\text{B}(\text{B}) = p_\text{C}(\text{C})=1.
$ 

\par We again assume that rewards only depend on the states: $r(\text{I})=r(\text{C})=0, r(\text{B})=3/5 $
and the discount factor $\gamma =0.8.$ \text{I}n Lemma \ref{lma:counter-markov-stat}, we show that there exists a Markov non-time-homogeneous policy, the controller
can have a higher value than the solution of the robust Bellman equation, hence the absence of the DPP. The proof of Lemma \ref{lma:counter-markov-stat} is shown in Appendix \ref{sec:app:proofs:counterexample}.
\begin{lemma} \label{lma:counter-markov-stat}
    By solving the robust Bellman equation, we have $u^{\ast }(\mrm{C})=0,u^{\ast }(\mrm{B})=3,u^{\ast }(\mrm{I})=1$. 
    For a Markov but not time-homogeneous policy $\pi = (\pi_0,\pi_1,\ds)$ defined by 
$
	\pi _{0}(\mrm{I})=a_{1}\text{ and }\pi _{1}(\mrm{I})=a_{2}\text{ and }\pi
	_{i}(\mrm{I})=(3/4,1/4)\text{ for }i\geq 2,
$
we have 
$
	v(\delta_\mrm{I},\pi,p_\mrm{I}^{(1)}) =1.44>1, \text{and } v(\delta_\mrm{I},\pi,p_\mrm{I}^{(2)})=1.2>1.
$
\end{lemma}

\subsection{Deterministic Controllers}\label{section:counterexample_det_ctrl}
\label{table4:counter}
\text{I}n this subsection, we focus on deterministic controllers. Specifically, we will use the RMDP instance in Figure \ref{fig:table4:counter} to construct counterexamples to address the various cases with a red cross in Table \ref{tab:s-rec_det}. These cases include a history-dependent deterministic controller against a Markov convex adversary, a history-dependent deterministic controller paired with Markov time-homogeneous convex adversary, and a Markov deterministic controller with Markov time-homogeneous convex adversary. 
\begin{figure}[ht]
	\centering
	\subfigure[$p^{(1)}$]{
		\scalebox{0.7}{
			\begin{tikzpicture}
				\node[state]                               (A0) {$\text{I}_0$};
				\node[state,right=of A0,yshift=1cm,xshift=0.5cm]                               (A1) {$\text{I}_1$};
				\node[state,right=of A0,yshift=-1cm,xshift=0.5cm]                               (A2) {$\text{I}_2$};
				\node[state,right=of A2,yshift=1cm,xshift=0.5cm]         (A) {\text{I}};
				\node[state,right=of A,yshift=1cm,xshift=0.5cm]                   (B) {G};
				\node[state,right=of A,yshift=-1cm,xshift=0.5cm]                    (C) {B};

				\path(A0)   edge node[above] {{$1/2$}}   (A1);
				\path(A0)   edge node[below] {{$1/2$}}   (A2);
				\path(A1)   edge node[above] {}   (A);
				\path(A2)   edge node[above] {}   (A);
				\path(A)   edge[blue] node[above] {}  (B);
				
				\path (A)   edge[red] node[above] {}   (C);
				\path 	(B.north)   edge[bend right=60]    (A.north);
				\path 	(C.south)   edge[bend left=60]    (A.south);

	\end{tikzpicture}}}  
 \qquad
	\subfigure[$p^{(2)}$]{
		\scalebox{0.7}{
			\begin{tikzpicture}
				\node[state]                               (A0) {$\text{I}_0$};
				\node[state,right=of A0,yshift=1cm,xshift=0.5cm]                               (A1) {$\text{I}_1$};
				\node[state,right=of A0,yshift=-1cm,xshift=0.5cm]                               (A2) {$\text{I}_2$};
				\node[state,right=of A2,yshift=1cm,xshift=0.5cm]         (A) {\text{I}};
				\node[state,right=of A,yshift=1cm,xshift=0.5cm]                   (B) {G};
				\node[state,right=of A,yshift=-1cm,xshift=0.5cm]                    (C) {B};

				\path(A0)   edge node[above] {{$1/2$}}   (A1);
				\path(A0)   edge node[below] {{$1/2$}}   (A2);
				\path(A1)   edge node[above] {}   (A);
				\path(A2)   edge node[above] {}   (A);

				\path(A)   edge[red] node[above] {}  (B);
				
				\path (A)   edge[blue] node[above] {}   (C);
				\path 	(B.north)   edge[bend right=60]    (A.north);
				\path 	(C.south)   edge[bend left=60]    (A.south);
	\end{tikzpicture}}} 
	\caption{The extreme point of adversarial actions in the action distribution set, where the red
		line and the blue line represent actions $a_1$ and $a_2$, respectively.}
  \label{fig:table4:counter}
\end{figure}
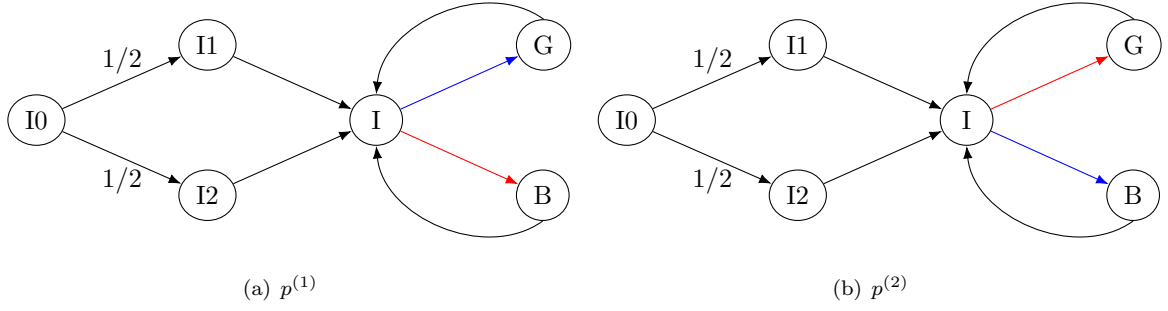
We consider an RMDP environment with six states $\text{I}_0,\text{I}_1,\text{I}_2,\text{I}$,B,G. The initial state is always $\text{I}_0$. At state \text{I}, there
are two actions: ${A}=\{a_{1},a_{2}\}.$ For all other states, there is only one action. As shown in the Figure \ref{fig:table4:counter}, at state $\text{I}_0$, the transition probabilities are 
$
p_{\text{I}_0}(\text{I}_1) = p_{\text{I}_0}(\text{I}_2)=\frac{1}{2}.
$
As before, we assume that rewards only depend on the states: $r(\text{I}_0)=0,r(\text{I}_1)=0,r(\text{I}_2)=0,r(\text{I})=0,r(\text{G})=1,r(\text{B})=-1.$
\par The adversary has action distribution set $\cP_\text{I}$ at state \text{I}. If we assume a convex adversary, we will use $\cP_\text{I}:=\set{\alpha p_\text{I}^{(1)}+(1-\alpha)p_\text{I}^{(2)}:\alpha\in [0,1]}$, else if the adversary is assumed to be non-convex, then we define $\cP_\text{I}:=\set{ p_\text{I}^{(1)},p_\text{I}^{(2)}}$, where the two distributions $ p_\text{I}^{(1)},p_\text{I}^{(2)}$ are characterized by
$$
	p^{(1)}_{\text{I},a_{1}}(\text{B}) =1, p^{(1)}_{\text{I},a_{2}}(\text{G})=1, \text{ and }
	p^{(2)}_{\text{I},a_{1}}(\text{G}) =1, p^{(2)}_{\text{I},a_{2}}(\text{B})=1. 
$$
We will only verify the case when the adversary is convex; i.e. $\cP_\text{I}:=\set{\alpha p_\text{I}^{(1)}+(1-\alpha)p_\text{I}^{(2)}:\alpha\in [0,1]}$. The absence of the DPP is still valid when we assume a non-convex adversary with $\cP_\text{I}:=\set{ p_\text{I}^{(1)},p_\text{I}^{(2)}}$.

Lemmas \ref{lemma:counter:det:bellman} and \ref{lemma:counter:det:Markov} demonstrate that a deterministic Markov controller can attain a value greater than that indicated by the robust Bellman equations. The proofs of these lemmas are provided in Appendix \ref{sec:app:proofs:counterexample}.
\begin{lemma}
    \label{lemma:counter:det:bellman}
    By solving the robust Bellman equation, we have $u^*(\mathrm{I}_0)=-\frac{\gamma^3 }{1-\gamma ^{2}}$.
\end{lemma}

 We first consider the case of a Markov controller with a Markov time-homogeneous convex adversary. 
 
 \begin{lemma} 
 \label{lemma:counter:det:Markov}
\text{I}f controller uses the following Markov policy $\pi$ with
$
	\pi_2 = a_{1}, \pi_4=a_{2},\pi_6=a_{1},\pi_8=a_{2},\ldots, 
$
we have 
$v(\delta_{\mrm{I}_0},\pi ,p_{\mrm{I}}^{\alpha }) =(1-2\alpha )\frac{\gamma ^{3}}{1+\gamma ^{2}}%
>u^{\ast }(\mrm{I}_0).$
\end{lemma}

By Lemma \ref{lemma:top->bot_right->left}, there must be a history-dependent deterministic controller that achieves no worse performance than $\pi$. So, the above constructions and inequalities also imply that a history-dependent deterministic controller with a Markov time-homogeneous convex adversary can achieve a higher value. Therefore, the DPP doesn't hold for the two cases discussed above. 

\par Finally, we analyze the case with a history-dependent deterministic controller and a Markov convex adversary.
We consider the history-dependent policy such that $\pi_{2i}=a_1,i=1,2,\ldots,$ if the controller sees state $\text{I}_1$ and $\pi_{2i}=a_2,i=1,2,\ldots,$ if the controller sees state $\text{I}_2$. Then, for any Markov adversary $\kappa$, we always have 
$
v(\delta_{\text{I}_0},\pi ,\kappa) =0> u^*(\text{I}_0). 
$ 
Again, this suggests the absence of a DPP in this setting as well. 
\cblue{\section{Asymptotically Optimal History-Dependent Control}\label{section:asymp_opt_hd_policy}

As demonstrated in Section \ref{table3_1}, the DPP does not generally hold when a history-dependent controller interacts with a time-homogeneous adversary selecting actions from a non-convex S-rectangular set. This breakdown implies that Markov time-homogeneous policies are strictly suboptimal and that an effective policy in this setting must intelligently adapt to the historical information it generates. As previously discussed, this scenario is of practical significance, as real-world decision-making often allows the flexibility to employ non-stationary strategies in a potentially adversarial environment that remains completely static. Given these considerations, a natural question arises: Can we design good history-dependent policies that are both theoretically sound and implementable in real-world applications?

\par In this section, we provide an affirmative answer by proposing a reinforcement learning (RL)-based history-dependent policy that can be effectively implemented in various RMDP settings. We demonstrate that this policy is asymptotically optimal as the effective horizon \( 1/(1-\gamma) \) approaches infinity, as formally defined in Definition \ref{def:asymp_opt_HD_policy}. Crucially, our proposed policy is built on the intuition developed in Section \ref{section:counterexample}: to achieve asymptotic optimality, the controller must exploit the adversary’s time-homogeneous nature while effectively learning the fixed transition structure imposed by the adversary. Moreover, this strategic adaptation must ensure sufficient exploration and the near-optimal utilization of historical data. Together, these elements enable the policy to achieve near-optimal performance over long horizons.

Concretely, we analyze a scenario in which a time-homogeneous adversary interacts with a history-dependent controller. For clarity in this presentation, we assume \( \cQ = \cP(\cA) \); however, we will later discuss in Remark \ref{rmk:det_controller_asymp_opt} how the main result can be extended to non-convex \( \cQ \). We aim to demonstrate that, given an adversarial policy class \( \KSS \) with non-convex S-rectangular ambiguity sets \( \{ \cP_s : s \in S \} \), we can construct an asymptotically optimal history-dependent policy \( \pi \in \PiH \) in the following sense.
\begin{definition}\label{def:asymp_opt_HD_policy}
We say that a history-dependent policy $\pi\in\PiH$ (which may depend on $\gamma$) is asymptotically optimal for the RMDP if
\begin{equation}\label{eqn:def_asymp_opt_policy}0\leq (1-\gamma)[v_{\gamma}(\mu,\PiH,\KSS) -v_{\gamma}(\mu,\pi,\KSS)]\ra 0\end{equation}
as $\gamma\ra 1$. Note that here we emphasize the $\gamma$ dependence of the value function by adding a subscript.

\end{definition}
\begin{remark}
    If we further assume that the controller can achieve positive reward at any state; i.e. $r_\vee(s):= \max_{b\in A}r(s,b)  > 0$ for all $s\in S$, then $v_\gamma(s,\PiH,\KSS)\geq \frac{1}{1-\gamma}r_\vee(s)$. Therefore, \eqref{eqn:def_asymp_opt_policy} implies that
    $$\begin{aligned}1\geq \frac{v_{\gamma}(\mu,\pi,\KSS)}{v_{\gamma}(\mu,\PiH,\KSS)} &= 1-\frac{v_{\gamma}(\mu,\PiH,\KSS) - v_{\gamma}(\mu,\pi,\KSS)}{v_{\gamma}(\mu,\PiH,\KSS)}\\
    &\geq 1- \frac{(1-\gamma)[v_{\gamma}(\mu,\PiH,\KSS) - v_{\gamma}(\mu,\pi,\KSS)]}{E_\mu[ r_\vee(X_0)]} \ra 1
    \end{aligned}$$
    i.e. the approximation ratio goes to $1$ as $\gamma\ra1$. 
\end{remark}

Building on the intuition that the controller can potentially exploit the time-homogeneous adversary’s inability to adapt and learn the full dynamics of the adversary, we propose the following Policy \ref{policy:ETE} as a candidate for asymptotic optimality in the sense of Definition \ref{def:asymp_opt_HD_policy}.

{\renewcommand{\algorithmcfname}{Policy}
\begin{algorithm}[ht]
    
    \KwSty{Input:} Discount factor $\gamma\in(0,1)$ and exploration period $n\geq 1$. \\
    Let $\bar \pi_0\ds \bar \pi_n$ and $\hat p$ be the output of the exploration Subroutine \ref{subroutine:exploration}. \\
    Apply Subroutine \ref{subroutine:exploitation} with input $(n,\gamma,\hat p)$. \\
    \KwSty{Output:} History dependent policy $\pi = (\bar \pi_0,\bar \pi_1,\ds,\bar \pi_n, \hat \pi,\hat \pi,\ds  )$, where $\pi_0,\bar \pi_1,\ds,\bar \pi_n$ are produced by the \textbf{Exploration} subroutine and $\hat \pi$ is produced by the \textbf{Exploitation} subroutine.  \\
    \caption{Explore-then-Exploit Policy}\label{policy:ETE}
\end{algorithm}}

\par Policy \ref{policy:ETE} has two phases. During the first phase, the policy use the decisions generated by the Exploration Subroutine \ref{subroutine:exploration} that explores the entire state and action spaces and gathers a data set. Using these data, the policy constructs a deterministic decision rule as described in the Exploitation Subroutine \ref{subroutine:exploitation} and uses it thereafter. We show that under the assumption that the all $p\in\cP^S$ is communicating \citep{Puterman1994MDP, Zurek_2024_plugin_SC_MDP}, then this algorithm is asymptotically optimal in the above sense.

{\renewcommand{\algorithmcfname}{Subroutine}
\begin{algorithm}[ht]
\caption{Exploration}
    \label{subroutine:exploration}
    \KwSty{Input:} Exploration period $n\geq 1$. \\
    Choose the effective sample size \begin{equation}\label{eqn:choice_of_m}
        m = \floor{\frac{ n }{ 8 |Z| D^3|A|^D}}
    \end{equation} and define random variables $\tau_{1,1} = \inf_{t\geq  0 }\set{s_t = z^1(0) }$ and for $k = 1,2,\ds,m$, \begin{equation}\tau_{j,k+1}:= \inf_{t\geq  \tau_{j,k}+1 }\set{s_t = z^j(0) }, \quad \tau_{j+1,1}:= \inf_{t\geq  \tau_{j,m}+1 }\set{s_t = z^{j+1}(0) }.\label{eqn:def_renewal_times}\end{equation}\\
    This defines a sequence of $\cH_t$ stopping times of consecutively hitting each of $s\in S$ for $|A|m$ times. 
    \par Since $\omega\ra \tau_{j,k}(\omega)$ is $\cH_t$ measureable, we can define $B_{j,k,t}(h_t):= \dsi\set{\tau_{j,k}(\omega) \leq t,\omega = (h_t,\ds)}$ and $ 
    K_{j,t}(h_t):=\sum_{k=1}^{m}B_{j,k,t}(h_t)$, and set $$\begin{aligned}
    J_t(h_t)&:= 1+ \sum_{j=1}^{|Z|}\dsi\set{K_j(h_t) = |A|m}\\
    B_t(h_t)&:= B_{J_t,K_{J_t,t}+1}(h_t)\dsi\set{J_t(h_t)\leq |Z|}.
    \end{aligned}$$  
    Naturally, these definitions extend to random variables with $B_{j,k,t}(\omega) = B_{j,k,t}(H_t)$. \\
    \KwSty{Output 1:} Construct the sequence of decision rules for the exploration phase by defining $$\bar \pi_t(h_t)(a) = B_t(h_t)\dsi\set{z^{J}(1) = a } +  (1-B_t(h_t))\frac{1}{|A|}$$ for all $h_t\in\bd{H}_t$ and $a\in A$. Return $\bar\pi_0,\ds ,\bar\pi_n$. \\

    \KwSty{Output 2: } Recall the state projection mapping $X_t(\omega) = s_t$ where $\omega = (s_0,a_0,\ds,s_t,\ds)$. We let  $X_{z^j}^{(k)} = X_{\tau_{j,k}+1}$, and define the following mixture kernel
    $\hat p = \set{\hat p_{z}\in\cP(\cS):z\in Z}$, seen as a random element on $(\Omega,\cH_n)$,
    where $$\hat  p_{z}(s' ) = \dsi\set{J_n = |Z|+1 }\frac{1}{m}\sum_{k=1}^m \dsi\set{ X_{z^j}^{(k)} = s'}+ \dsi\set{J_n\leq |Z|,s' = z(0)}.$$
    Return $\hat p$. 
\end{algorithm}}

To introduce the Exploration and Exploitation subroutines, we define some new notations. For notation simplicity, we index every $(s,a)\in S\times A$ by denoting $S\times A:= |Z| = \set{z^{1},\ds ,z^{|Z|}}$. Also, for $z = (s,a)$, we write $z(0) = s$ and $z(1) = a$. Let $\cH_t = \sigma(H_t)$ be the $\sigma$-field generated by $H_t = (X_0,A_0,\ds A_{t-1},X_t)$. 

{\renewcommand{\algorithmcfname}{Subroutine}
\begin{algorithm}[ht]
\caption{Exploitation---Empirical Value Iteration (EVI)}
    \label{subroutine:exploitation}
    \KwSty{Input:} Exploration period $n\geq1$, discount factor $\gamma\in(0,1)$, and empirical kernel $\hat p$ seen as a measurable function on $(\Omega,\cH_n)$.\\
    Conditioned on $\cH_n$, solve for the unique fixed point of the Bellman equation
    $$\hat v(s) = \max_{a\in A} r(s,a) +\gamma \sum_{s'\in S}\hat p_{s,a}(s') \hat v(s'), \quad \forall s\in S.$$
    Note that $\hat v\in m\cH_n$. \\
    \KwSty{Output:} For all $t\geq n+1$, $h_t\in \bd H_t$, and $a\in A$, set 
    $$\hat \pi_{t}(a|h_t) = \frac{\dsi\set{a\in \arg\max_{b\in A} \hat  r(s,b) +\gamma \sum_{s'\in S}\hat p_{s,b}(s') \hat v(s')}}{\abs{\arg\max_{b\in A} \hat  r(s,b) +\gamma \sum_{s'\in S}\hat p_{s,b}(s') \hat v(s') } }.$$
    Return $\hat\pi. $
\end{algorithm}}

\par We provide an intuitive description of Policy \ref{policy:ETE}. The sequence of decisions $\bar\pi_0,\ds ,\bar\pi_n,\ds $ generated by the Exploration Subroutine \ref{subroutine:exploration} achieves the following mechanism for any given integer $m \ge1 $. 
\begin{enumerate}
    \item Starting from an initial state, use a uniformly random decision until hitting $s^1$. Set $j=1$, $k = 0$, $l=1$. 
    \item Use action $a^l$ to transition to a new state and set $k\la k+1$. If $k=m$, and $l\leq |A|-1$, we move on to the next action by assigning $l\la l+1$. Else, if $k=m$, and $l= |A|$, moving onto the next state by assigning $j\la j+1$ and starting from action $l=1$. 
    \item Use the uniformly random decision until hitting state $s^j$. In particular, if $j > |S|$, the uniformly random decision is used indefinitely. Upon hitting $s^j$, go back to step 2. 
\end{enumerate}
\par Under the assumption that every \( p \in \cP^S \) is communicating and $n$ is sufficiently large, this subroutine can, with high probability, successfully collect \( m \) i.i.d. transition samples starting from each \( z \in Z \), sampled from the adversarial transition kernel fixed over time and history. Upon the successful collection of the transition samples, we form an empirical transition kernel $\hat p$. 
\par The exploitation phase, as outlined in Subroutine \ref{subroutine:exploitation}, simply applies the optimal policy obtained from the empirical transition kernel $\hat p$ by solving the corresponding non-robust empirical Bellman equation. We note that the policy in Subroutine \ref{subroutine:exploitation} randomizes the action in cases of tied values. Alternatively, one could opt for deterministically selecting one action when ties occur.

\par In summary, under Policy \ref{policy:ETE}, the controlled Markov chain will, with high probability, successfully collect an empirical kernel consisting of $m$ i.i.d. transition samples for each state-action pair. Subsequently, an empirical optimal policy is computed using the empirical value iteration subroutine, which, with high probability, yields a policy that is close to optimal. Concretely, Policy \ref{policy:ETE} satisfies the following Theorem: 

\begin{theorem}[Asymptotic Optimality of Policy \ref{policy:ETE}]
\label{thm:asymp_opt_policy}
Assume that for every $p\in\cP^\mrm{S}$, $p$ is communicating with diameter at most $D$ (c.f. \citet{Zurek_2024_plugin_SC_MDP}). Then, Policy \ref{policy:ETE} with $n = \frac{1}{\sqrt{1-\gamma}}$ is asymptotically optimal in the sense of \eqref{eqn:def_asymp_opt_policy}. 
\end{theorem}
\begin{remark}\label{rmk:det_controller_asymp_opt}
One can extend this result to establish the asymptotic optimality of policies analogous to Policy \ref{policy:ETE} in the deterministic controller setting \( \cQ = \cQ^{\mrm{D}} \). However, weak communication alone may not suffice, as uniform randomization cannot be leveraged for exploration. Instead, we can assume the existence of a Markov time-homogeneous decision rule \( \bar\pi \) such that, for all \( p \in \cP \), the transition kernel \( (s, s') \mapsto p_{s, \bar\pi(s)}(s') \) induces an irreducible Markov chain with a bounded mixing time \( t_{\text{mix}} \). 

Under this condition, substituting \( \bar\pi \) for the uniform randomization policy \( 1/|A| \) in the Exploration Subroutine \ref{subroutine:exploration}, along with an appropriately adjusted choice of \( m \), ensures that a version of Theorem \ref{thm:asymp_opt_policy} holds in the setting where \( \cQ = \cQ^{\mrm{D}} \).
\end{remark}

\par To conclude this section, we note that readers familiar with the sample complexity literature in RL might, based on the intuition established earlier, consider model-free algorithms such as asynchronous Q-learning \citep{watkins1992q} as potential candidates for asymptotically optimal algorithms. However, due to its non-minimax optimal sample complexity dependence on $1/(1-\gamma)$ \citep{Li2024QL_minmax}, Q-learning does not yield an asymptotically optimal policy in this setting. Instead, achieving asymptotic optimality requires more refined model-free algorithms that are optimal under appropriate mixing conditions, such as the variance-reduced Q-learning method proposed in \cite{wang2023DMDP_mixing}.

}
\section{Connections with other Formulations} 
In this section, we elucidate the links between the formulations of RMDP and the multistage stochastic programs \citep[Chapter 3]{shapiro2021lectures}  within the realm of optimization literature. Additionally, we explore their connections with stochastic games, as introduced by \citet{shapley1953stochastic} in the economics literature.

\subsection{Multistage Stochastic Programs}\label{section:multistage_sp}
Early works on RMDP, such as \citet{gonzalez2002minimax} and \citet{nilim2005robust}, interpreted the max-min problem as dynamic interactions between a controller and an SA-rectangular adversary.
Indeed, the presence of a Dynamic Programming Principle (DPP) in these scenarios implies that the interactions under optimal robust control are equivalent to another dynamic model wherein the controller and adversary sequentially make decisions.
This interpretation results in what is now considered as a robustified variant of the multistage stochastic program (MSP) \citep{huang2017study}. In contrast, \citet{nilim2005robust}'s version of \cblue{sequential min-max-...-min-max decisions (see Equation (5) therein) is selected before the state and action sequences are realized.} As a widely used paradigm for modeling dynamic decision-making, MSP finds numerous applications in addressing real-world control problems. There exists a substantial body of literature on MSP, and multiple recent works dedicating effort to consider a robust variant of it; see, for example, \citet[Chapter 7]{shapiro2021lectures} and \citet{huang2017study, shapiro2021distributionally, pichler2021mathematical}. 

\subsection{Stochastic Games}
\label{section:stochastic_game}
Stochastic games (SGs a.k.a. Markov games) model interactions among players where the environment evolves according to their actions \citep{solan2015stochastic}. At each stage, the game is at a state, and each player selects an action, potentially history-dependent and randomized, from their decision sets. The realized actions and the current state determine both the players’ payoffs for that stage and the probability distribution transitioning to the next state. 

Our formulation of an RMDP with a controller and an S-rectangular adversary can be viewed as a two-player zero-sum game, closely linked to the framework of SGs. This connection has also been observed in \citet{gonzalez2002minimax,nilim2005robust,le_tallec2007robustMDP,shapiro2021distributionally,li_shapiro2023rectangularity}.
Research on SGs was pioneered by \citet{shapley1953stochastic}, who proved the existence of a stationary equilibrium in a zero-sum game. Following the groundbreaking results of  \citet{shapley1953stochastic}, a substantial body of literature has emerged, with a primary focus on identifying scenarios in which stationary Markov equilibria exist \citep{fink1964equilibrium,takahashi1964equilibrium,parthasarathy1989existence,nowak2003new,simon2003games,levy2013discounted,maskin2001markov}. We also note that SGs become a standard framework of multi-agent reinforcement learning \citep{littman1994markov}. For a review of the history of SGs, we refer readers to \citet{solan2015stochastic}. 

\cblue{ A key distinction between SGs and RMDPs is that solutions to the RMDP Bellman equations \eqref{eqn:maxmin_opt_val} may not guarantee the existence of a stationary Markov equilibrium when the sup–inf cannot be interchanged. In such cases, the controller can attain a higher value by deviating from the strategy $\pi^*$ that achieves the maximum in \eqref{eqn:dr_bellman_eqn} once the adversary’s strategy is fixed as the worst case $\set{p_s^*:s\in S}$ corresponding to $\pi^*$.}

Another distinction between RMDPs and lies in the symmetry of the players. In SGs, both players can adopt history-dependent policies, making them symmetric. In contrast, RMDPs often feature asymmetry: the controller may use history-dependent strategies, while the adversary is typically modeled as Markov time-homogeneous, as also noted in \citet{grand2023beyond}.

RMDPs with various controller–adversary attributes are also connected to SGs with asymmetric information \citep{nayyar2013common,nayyar2017information}, which build on partially observable MDPs. These works address the asymmetry by transforming the game into one with symmetric information.

\subsection{Distributionally Robust MDPs}
Distributionally Robust MDPs (DRMDPs) \citep{xu2010distributionally,clement2021drmdp} provide an alternative framework for robust dynamic decision-making. The adversary in a DRMDP can be interpreted from a Bayesian-robust perspective, where it selects a prior distribution over transition probabilities, and the system dynamics are induced by samples drawn from this prior. While DRMDP and RMDP models differ at the level of sample-path dynamics, they are equivalent from the perspective of policy evaluation and optimization. In particular, all of the results in this paper remain valid if we consider the \textit{expected} version of the adversary. We defer a more detailed discussion to Appendix \ref{A_sec:equiv_DRMDP}.
\section{Discussion of Various Attributes}\label{section:model_selection_guide}
In this section, we provide guidance on selecting the most suitable formulation from the 36 cases presented in Tables \ref{tab:sa-rec} to \ref{tab:s-rec_det} for the DRRL problems. Our criteria for selection are based on four main factors: rectangularity, information structure, convexity of adversaries, and the policy classes of controllers, whether deterministic or randomized. As we shall elaborate in the proceeding discussions, a variety of attributes related to the adversary in the underlying RMDP not only lead to a formulation that facilitates robust decision-making in a dynamic environment but also function as modeling tools to mitigate various forms of model misspecification. 

\par \textbf{Rectangularity of the adversary:}
\begin{itemize}
    \item \textbf{SA-rectangular adversary:} 
    SA-rectangular adversary models arise naturally in scenarios where an RL agent is expected to encounter distinct environments or state transitions following different actions. For example, consider an autonomous driving agent deciding whether to turn left, right, or go straight at an intersection. Each of these actions leads to entirely different environments and state observations. In such cases, it is reasonable to assume that the stochasticity associated with state transitions following one action is independent of that associated with other actions. This motivates the use of SA-rectangularity as a natural structural assumption on the adversary in the underlying DRRL formulation. Notably, \citet{ding2023seeing} adopted a DRRL model with a specially structured SA-rectangular ambiguity set and demonstrated that their algorithm outperformed all baselines in autonomous driving tasks.
    
\item \textbf{S-rectangular adversary:} Compared to SA-rectangular adversaries, S-rectangular adversaries exhibit a more constrained influence by limiting the set of transition probabilities through enforced correlations across different actions. As illustrated in the inventory control Example~\ref{example:inventory_S_rec}, it is often more natural to restrict the adversary from using the realized action to determine its shift. In effect, S-rectangularity enables the modeler to reduce the adversary's power by preventing it from reacting to the controller's chosen action in real time. As a result of this refinement, the resulting robust policies tend to be less conservative. In addition, adding reasonable constraints could facilitate the calibration of the uncertainty set; see, for example, \cite{lotidis23Mg_constraints}.
\end{itemize}
\par It is important to note that SA-rectangular adversaries are more powerful than S-rectangular counterparts, which may result in a more conservative policy. Nevertheless, for SA-rectangular adversaries, the existence of a DPP is guaranteed (Table \ref{tab:sa-rec}), and a $q$-function \citep{watkins1992q} with its Bellman equation is definable (equation \eqref{eqn:SA_q_func}). These characteristics facilitate the design of model-based \citep{zhou21,Panaganti2021,yang2021,ShiChi2022,xu2023improved,shi2023curious_price,blanchet2023double_pessimism_drrl} and model-free \citep{liu22DRQ,Wang2023MLMCDRQL,wang2023VRDRQL,yang2023avoiding} DRRL algorithms, facilitating efficient learning.
\par \textbf{Information structure of the adversary:}
\begin{itemize}
\item \textbf{Markov time-homogeneous adversary: } Markov time-homogeneous formulation is widely analyzed in the (R)MDP literature. An adversary with this attribute is compelled to predefine \textit{one} transition kernel and adhere to its transition probability across all time. Hence, such an adversary is sometimes referred to as \textit{static} in the literature. Although this formulation is preferred and commonly used by practitioners and DRRL theorists due to its structural simplicity and ease of interpretation \citep{wiesemann2013robust,zhou21,grand2023beyond}, the corresponding RMDP problem could pose significant computational challenges due to the absence of a DPP and, consequently, the sub-optimality of time-homogeneous policies. Specifically, DPPs do not exist in general if the ambiguity set is not convex (Table \ref{tab:s-rec_randomized}) or if the controller is deterministic (Table \ref{tab:s-rec_det}). Therefore, caution is advised for researchers and practitioners when attempting to derive optimal robust control under this formulation. \cblue{In scenarios where a DPP does not exist, one may develop new exact or approximate methods on a case-by-case basis; alternatively, our asymptotically optimal RL-based policy design (see Section~\ref{section:asymp_opt_hd_policy}) can be consulted for practical guidance in these settings.}

\par It's worth noting that previous DRRL papers commonly adopt the modeling perspective of a time-homogeneous adversary \citep{zhou21,Panaganti2021,liu22DRQ,ShiChi2022,xu2023improved,shi2023curious_price}. However, these papers all refer to the optimality of deterministic Markov time-homogeneous policies and Bellman equations established in \cite{iyengar2005robust}, where the principal modeling assumption involves a Markov adversary and a history-dependent controller (as summarized in Table \ref{tab:sa-rec}). In this paper, we establish the existence or non-existence of the DPP and Bellman equation for both the value and the Q-function in the case of a Markov time-homogeneous adversary, providing the theoretical foundation for these works.

\item \textbf{Markov adversary:} A Markov adversary is the natural modeling choice when the environment is believed to exhibit memoryless transitions under the prescribed state and action spaces. On the other hand, this attribute permits non-stationary transition probabilities. Therefore, a Markov adversary could be an appropriate modeling assumption when the probability structure in the underlying Markov dynamics can vary across time. For instance, in the aforementioned inventory Example \ref{example:inventory_S_rec}, the demand distribution could change from day to day, resulting in possible time-nonhomogeneity. From a computational perspective, DPPs also always exist with randomized controllers, see the second columns of Tables \ref{tab:sa-rec} and \ref{tab:s-rec_randomized}.

\item \textbf{History-dependent adversary:} History-dependence allows the adversary to select different distributions based on the evolving history of the system. Even when the modeler has strong reason to believe the environment follows Markov dynamics, partial observability or incomplete modeling of the system state often leads to an optimal controller that is inherently history-dependent, thereby inducing non-Markovian transition dynamics. In this context, formulating the DRRL problem with a history-dependent adversary not only strengthens the robustness of the optimal controller against non-Markov shifts in the environment, but also offers a practical safeguard against model misspecification when the true Markov state is only partially observed or incorrectly specified.

\par Moreover, history-dependent adversaries are relevant in RL settings where the true environment is not an MDP, but a DRRL model is adopted for approximation due to computational or data-efficiency considerations. This situation frequently arises when an RL agent interacts strategically with its environment. For instance, consider an RL agent learning bidding strategies in online advertising auctions \citep{cai2017real}. In recent years, the industry has shifted from second-price to first-price auction mechanisms\footnote{\url{https://blog.google/products/adsense/our-move-to-a-first-price-auction/}}. In second-price auctions, truthful bidding is a weakly dominant strategy, but this is no longer true under first-price rules. Suppose an RL agent is trained in a second-price environment but is deployed in a first-price exchange. The environment, consisting of strategic competitors, may respond differently to the agent’s varied actions and may incorporate the agent’s bidding history into their strategies. This behavior renders the effective environment non-Markovian.

\par In such cases, formulating a fully non-Markovian control problem that explicitly models the strategic learning dynamics of competitors may be prohibitively complex to specify, learn, or solve. In contrast, a DRRL formulation that admits DPP can serve as a computationally tractable and principled approximation.

\par Furthermore, from a computational standpoint, a history-dependent adversary model is especially appealing, as the first columns of Tables~\ref{tab:sa-rec}--\ref{tab:s-rec_det} confirm the guaranteed existence of a DPP.
\end{itemize}

By symmetry, the controller can also adopt these attributes regarding information availability. The rationale for selecting such information structures for the controller mirrors that of the adversary. Consequently, we do not provide a separate discussion on this point.

\textbf{Convexity of the adversarial ambiguity sets: }
\begin{itemize}
    \item \textbf{Convex adversary:} RMDP and DRRL models with convex adversaries are prevalent in literature, with instances spanning various ambiguity sets such as optimal-transport-based \citep{taşkesen2023DRLQG, Yang2021Wasserstein_dr_control}, convex $f$-divergence-based \citep{yang2021,shi2023curious_price}, TV-distance-based \citep{shi2023curious_price}, as well as the special ambiguity set employed in \citet{xu2010distributionally,wiesemann2013robust,ding2023seeing}. Our findings, as presented in Tables \ref{tab:sa-rec} and \ref{tab:s-rec_conv}, further underscore the computational benefits of convex adversaries. Nevertheless, a word of caution is necessary when navigating scenarios involving a deterministic controller paired with an S-rectangular convex adversary, as the DPP may not be applicable in this case.
    \item \textbf{Non-convex adversary:} A non-convex adversary may arise as a modeling choice when the underlying system dynamics are believed to be deterministic \citep{post2013det_MDP}. Additionally, the use of a non-convex adversary can result from adopting unions of convex ambiguity sets—often motivated by data-driven calibration procedures that aim to refine the size and geometry of the uncertainty sets, as proposed by \citet{hong2020learningbased}. From a computational perspective, the presence of non-convex adversaries not only increases the complexity of solving the robust variants of the Bellman equations (both the value function equation \eqref{eqn:dr_bellman_eqn} and the $q$-function equation \eqref{eqn:dr_bellman_eqn_q}), but our findings further indicate that the DPP may fail to hold altogether. This breakdown occurs under both deterministic and randomized controller policies, as summarized in Tables~\ref{tab:s-rec_randomized} and~\ref{tab:s-rec_det}.
\end{itemize}

\textbf{Possibility of randomization (convexity) of the controller: }
\begin{itemize}
    \item \textbf{The randomized controller policy class: } 
    While it is well known that deterministic policies can achieve optimality in the non-robust MDP setting (see, e.g. \citet{Puterman1994MDP}), this is generally not the case in the presence of an S-rectangular adversary, a phenomenon first observed and analyzed by \citet{wiesemann2013robust}. As such, in the context of robust control or DRRL, allowing for randomized controller policies becomes essential for attaining optimal performance. Moreover, from the standpoint of DPP validity, randomized controllers offer an additional advantage: when a history-dependent controller is paired with a Markov adversary, the use of randomized policies can ensure the existence of a DPP (c.f. Tables~\ref{tab:s-rec_randomized} and~\ref{tab:s-rec_det}).
    
    \item \textbf{The deterministic controller policy class:} While optimality may not always be attainable with deterministic policies, implementing randomized policies can be challenging in many real-world scenarios. For example, consider the context of medical treatment planning, where a patient’s current health condition guides the choice among a range of dynamic treatment regimes \citep{chakraborty2014dynamic}. In such settings, a randomized policy would require assigning probability distributions over treatment options for a given health state—a practice that is often neither legally permissible nor ethically acceptable.

    \par When focusing on deterministic policies, it is important to recognize that the absence of a DPP implies that Markov time-homogeneous policies are not guaranteed to be optimal in many cases, even when the adversary’s decision set is convex. This limitation is highlighted in Table~\ref{tab:s-rec_det}. To achieve higher rewards in such settings, the controller may need to consider history-dependent policies.
\end{itemize}

It is worth noting that while our earlier discussion focuses on randomized and deterministic controller policies, the paper extends this distinction to the broader framework of convex versus non-convex controller action sets. Specifically, the randomized controller policy class corresponds to a convex set, while the deterministic policy class is an example of a non-convex set. The DPP results presented in Tables~\ref{tab:s-rec_randomized} and~\ref{tab:s-rec_det} naturally extend to general convex and non-convex controller policy classes, respectively. From both a modeling and learning perspective, the existence of a DPP for controller policy sets with general geometries opens promising avenues for extending DRRL algorithms. Such extensions could accommodate scenarios in which the controller is restricted to operate within a structured subset of randomized decisions.

\section{Conclusion and Future Works}
In this paper, we clarify the definitions of robust Markov decision processes (RMDP) and identify the scenarios in which the dynamic programming principle (DPP) is either held with full generality or violated.  Following this trajectory, several intriguing future directions emerge:
\begin{itemize}
\item In cases where a DPP in the form of the robust Bellman equation does not hold with full generality, are there alternative optimality equations that can characterize the optimal robust control? If such equations exist, can they be effectively employed to facilitate DRRL?

\item While our current findings center on SA- and S-rectangularity, can analogous results be developed for general-rectangular adversaries?
\end{itemize}
We believe that partial answers to either research question will yield valuable theoretical insights and ultimately facilitate the development of more effective DRRL algorithms by improving learning efficiency and enriching the expressiveness of the DRRL model. 

\subsection*{Acknowledgments}
The material in this paper is partly supported by the Air Force Office of Scientific Research under award number FA9550-20-1-0397. Support from NSF 2229012, 2312204, 2312205, 2403007, 2419564, ONR 13983111, 13983263, and 2025 New York University Center for Global Economy and Business grant is also gratefully acknowledged.
\bibliographystyle{apalike}
\bibliography{bibs/rl,bibs/mybib,bibs/DR_MDP,bibs/game,bibs/drrl,bibs/dro,bibs/asymptotic_optimal_policy,bibs/risk-sensitive}

\begin{thebibliography}{}

\bibitem[Afsar et~al., 2022]{afsar2022reinforcement}
Afsar, M.~M., Crump, T., and Far, B. (2022).
\newblock Reinforcement learning based recommender systems: A survey.
\newblock {\em ACM Computing Surveys}, 55(7):1--38.

\bibitem[B{\"a}uerle and Rieder, 2014]{bauerle2014more}
B{\"a}uerle, N. and Rieder, U. (2014).
\newblock More risk-sensitive {Markov} decision processes.
\newblock {\em Mathematics of Operations Research}, 39(1):105--120.

\bibitem[Bellman, 1954]{bellman1954dpp}
Bellman, R. (1954).
\newblock {The theory of dynamic programming}.
\newblock {\em Bulletin of the American Mathematical Society}, 60(6):503 -- 515.

\bibitem[Bertsimas and Perakis, 2006]{bertsimas2006dynamic}
Bertsimas, D. and Perakis, G. (2006).
\newblock Dynamic pricing: A learning approach.
\newblock {\em Mathematical and computational models for congestion charging}, pages 45--79.

\bibitem[Blanchet et~al., 2024]{blanchet2023double_pessimism_drrl}
Blanchet, J., Lu, M., Zhang, T., and Zhong, H. (2024).
\newblock Double pessimism is provably efficient for distributionally robust offline reinforcement learning: Generic algorithm and robust partial coverage.
\newblock {\em Advances in Neural Information Processing Systems}, 36.

\bibitem[Cai et~al., 2017]{cai2017real}
Cai, H., Ren, K., Zhang, W., Malialis, K., Wang, J., Yu, Y., and Guo, D. (2017).
\newblock Real-time bidding by reinforcement learning in display advertising.
\newblock In {\em Proceedings of the tenth ACM international conference on web search and data mining}, pages 661--670.

\bibitem[Chakraborty and Murphy, 2014]{chakraborty2014dynamic}
Chakraborty, B. and Murphy, S.~A. (2014).
\newblock Dynamic treatment regimes.
\newblock {\em Annual review of statistics and its application}, 1:447--464.

\bibitem[Chow et~al., 2018]{chow2018risk}
Chow, Y., Ghavamzadeh, M., Janson, L., and Pavone, M. (2018).
\newblock Risk-constrained reinforcement learning with percentile risk criteria.
\newblock {\em Journal of Machine Learning Research}, 18(167):1--51.

\bibitem[Chow et~al., 2015]{chow2015risk}
Chow, Y., Tamar, A., Mannor, S., and Pavone, M. (2015).
\newblock Risk-sensitive and robust decision-making: a cvar optimization approach.
\newblock {\em Advances in neural information processing systems}, 28.

\bibitem[Clement and Kroer, 2021]{clement2021drmdp}
Clement, J.~G. and Kroer, C. (2021).
\newblock First-order methods for wasserstein distributionally robust {MDP}.
\newblock In {\em Proceedings of the 38th International Conference on Machine Learning}, pages 2010--2019.

\bibitem[Delage and Mannor, 2010]{delage2010percentile}
Delage, E. and Mannor, S. (2010).
\newblock Percentile optimization for {Markov} decision processes with parameter uncertainty.
\newblock {\em Operations research}, 58(1):203--213.

\bibitem[Di~Masi and Stettner, 1999]{di1999risk}
Di~Masi, G.~B. and Stettner, L. (1999).
\newblock Risk-sensitive control of discrete-time {Markov} processes with infinite horizon.
\newblock {\em SIAM Journal on Control and Optimization}, 38(1):61--78.

\bibitem[Di~Masi and Stettner, 2007]{di2007infinite}
Di~Masi, G.~B. and Stettner, {\L}. (2007).
\newblock Infinite horizon risk sensitive control of discrete time {Markov} processes under minorization property.
\newblock {\em SIAM Journal on Control and Optimization}, 46(1):231--252.

\bibitem[Ding et~al., 2023]{ding2023seeing}
Ding, W., Shi, L., Chi, Y., and Zhao, D. (2023).
\newblock Seeing is not believing: Robust reinforcement learning against spurious correlation.
\newblock {\em arXiv preprint arXiv:2307.07907}.

\bibitem[Fink, 1964]{fink1964equilibrium}
Fink, A.~M. (1964).
\newblock Equilibrium in a stochastic $ n $-person game.
\newblock {\em Journal of Science of the Hiroshima University, Series A-I (Mathematics)}, 28(1):89--93.

\bibitem[Fleming and Hern{\'a}ndez-Hern{\'a}ndez, 1997]{fleming1997risk}
Fleming, W.~H. and Hern{\'a}ndez-Hern{\'a}ndez, D. (1997).
\newblock Risk-sensitive control of finite state machines on an infinite horizon i.
\newblock {\em SIAM Journal on Control and Optimization}, 35(5):1790--1810.

\bibitem[Gong and Simchi-Levi, 2023]{gong2023bandits}
Gong, X.-Y. and Simchi-Levi, D. (2023).
\newblock Bandits atop reinforcement learning: Tackling online inventory models with cyclic demands.
\newblock {\em Management Science}.

\bibitem[Gonz{\'a}lez-Trejo et~al., 2002]{gonzalez2002minimax}
Gonz{\'a}lez-Trejo, J., Hern{\'a}ndez-Lerma, O., and Hoyos-Reyes, L.~F. (2002).
\newblock Minimax control of discrete-time stochastic systems.
\newblock {\em SIAM Journal on Control and Optimization}, 41(5):1626--1659.

\bibitem[Goyal and Grand-Cl\'{e}ment, 2023]{Goyal2023Beyond_Rectangularity}
Goyal, V. and Grand-Cl\'{e}ment, J. (2023).
\newblock Robust {Markov} decision processes: Beyond rectangularity.
\newblock {\em Mathematics of Operations Research}, 48(1):203--226.

\bibitem[Grand-Clement et~al., 2023]{grand2023beyond}
Grand-Clement, J., Petrik, M., and Vieille, N. (2023).
\newblock Beyond discounted returns: Robust {Markov} decision processes with average and blackwell optimality.
\newblock {\em arXiv preprint arXiv:2312.03618}.

\bibitem[Grislain et~al., 2019]{grislain2019recurrent}
Grislain, N., Perrin, N., and Thabault, A. (2019).
\newblock Recurrent neural networks for stochastic control in real-time bidding.
\newblock In {\em Proceedings of the 25th ACM SIGKDD International Conference on Knowledge Discovery \& Data Mining}, pages 2801--2809.

\bibitem[Haskell and Jain, 2013]{haskell2013stochastic}
Haskell, W.~B. and Jain, R. (2013).
\newblock Stochastic dominance-constrained {Markov} decision processes.
\newblock {\em SIAM Journal on Control and Optimization}, 51(1):273--303.

\bibitem[Haskell and Jain, 2015]{haskell2015convex}
Haskell, W.~B. and Jain, R. (2015).
\newblock A convex analytic approach to risk-aware {Markov} decision processes.
\newblock {\em SIAM Journal on Control and Optimization}, 53(3):1569--1598.

\bibitem[Hern{\'a}ndez-Hern{\'a}ndez and Marcus, 1996]{hernandez1996risk}
Hern{\'a}ndez-Hern{\'a}ndez, D. and Marcus, S.~I. (1996).
\newblock Risk sensitive control of {Markov} processes in countable state space.
\newblock {\em Systems \& control letters}, 29(3):147--155.

\bibitem[Hong et~al., 2020]{hong2020learningbased}
Hong, L.~J., Huang, Z., and Lam, H. (2020).
\newblock Learning-based robust optimization: Procedures and statistical guarantees.

\bibitem[Hu et~al., 2021]{hu2021prediction}
Hu, Y., Chan, C.~W., and Dong, J. (2021).
\newblock Prediction-driven surge planning with application in the emergency department.
\newblock {\em Management Science}.

\bibitem[Huang et~al., 2017]{huang2017study}
Huang, J., Zhou, K., and Guan, Y. (2017).
\newblock A study of distributionally robust multistage stochastic optimization.
\newblock {\em arXiv preprint arXiv:1708.07930}.

\bibitem[Iyengar, 2005]{iyengar2005robust}
Iyengar, G.~N. (2005).
\newblock Robust dynamic programming.
\newblock {\em Mathematics of Operations Research}, 30(2):257--280.

\bibitem[Janson, 2018]{Janson_2018_geometric_tail_bd}
Janson, S. (2018).
\newblock Tail bounds for sums of geometric and exponential variables.
\newblock {\em Statistics \& Probability Letters}, 135:1--6.

\bibitem[Kiran et~al., 2021]{kiran2021deep}
Kiran, B.~R., Sobh, I., Talpaert, V., Mannion, P., Al~Sallab, A.~A., Yogamani, S., and P{\'e}rez, P. (2021).
\newblock Deep reinforcement learning for autonomous driving: A survey.
\newblock {\em IEEE Transactions on Intelligent Transportation Systems}, 23(6):4909--4926.

\bibitem[Kormushev et~al., 2013]{kormushev2013reinforcement}
Kormushev, P., Calinon, S., and Caldwell, D.~G. (2013).
\newblock Reinforcement learning in robotics: Applications and real-world challenges.
\newblock {\em Robotics}, 2(3):122--148.

\bibitem[Landgraf et~al., 2021]{Landgraf2021automation}
Landgraf, C., Meese, B., Pabst, M., Martius, G., and Huber, M.~F. (2021).
\newblock A reinforcement learning approach to view planning for automated inspection tasks.
\newblock {\em Sensors}, 21(6).

\bibitem[Lattimore and Szepesv{\'a}ri, 2020]{lattimore2020bandit}
Lattimore, T. and Szepesv{\'a}ri, C. (2020).
\newblock {\em Bandit algorithms}.
\newblock Cambridge University Press.

\bibitem[Le~Tallec, 2007]{le_tallec2007robustMDP}
Le~Tallec, Y. (2007).
\newblock {\em Robust, risk-sensitive, and data-driven control of {Markov} decision processes}.
\newblock PhD thesis, Massachusetts Institute of Technology.

\bibitem[Levy, 2013]{levy2013discounted}
Levy, Y. (2013).
\newblock Discounted stochastic games with no stationary nash equilibrium: two examples.
\newblock {\em Econometrica}, 81(5):1973--2007.

\bibitem[Li et~al., 2024]{Li2024QL_minmax}
Li, G., Cai, C., Chen, Y., Wei, Y., and Chi, Y. (2024).
\newblock Is q-learning minimax optimal? a tight sample complexity analysis.
\newblock {\em Operations Research}, 72(1):222--236.

\bibitem[Li and Shapiro, 2023]{li_shapiro2023rectangularity}
Li, Y. and Shapiro, A. (2023).
\newblock Rectangularity and duality of distributionally robust {Markov} decision processes.

\bibitem[Littman, 1994]{littman1994markov}
Littman, M.~L. (1994).
\newblock {Markov} games as a framework for multi-agent reinforcement learning.
\newblock In {\em Machine Learning Proceedings 1994}, pages 157--163. Elsevier.

\bibitem[Liu et~al., 2022]{liu22DRQ}
Liu, Z., Bai, Q., Blanchet, J., Dong, P., Xu, W., Zhou, Z., and Zhou, Z. (2022).
\newblock Distributionally robust {Q}-learning.
\newblock In Chaudhuri, K., Jegelka, S., Song, L., Szepesvari, C., Niu, G., and Sabato, S., editors, {\em Proceedings of the 39th International Conference on Machine Learning}, volume 162 of {\em Proceedings of Machine Learning Research}, pages 13623--13643. PMLR.

\bibitem[Lotidis et~al., 2023]{lotidis23Mg_constraints}
Lotidis, K., Bambos, N., Blanchet, J., and Li, J. (2023).
\newblock Wasserstein distributionally robust linear-quadratic estimation under martingale constraints.
\newblock In Ruiz, F., Dy, J., and van~de Meent, J.-W., editors, {\em Proceedings of The 26th International Conference on Artificial Intelligence and Statistics}, volume 206 of {\em Proceedings of Machine Learning Research}, pages 8629--8644. PMLR.

\bibitem[Mannor et~al., 2016]{mannor2016robust}
Mannor, S., Mebel, O., and Xu, H. (2016).
\newblock Robust {MDP}s with k-rectangular uncertainty.
\newblock {\em Mathematics of Operations Research}, 41(4):1484--1509.

\bibitem[Mao et~al., 2020]{mao2020model}
Mao, W., Zhang, K., Zhu, R., Simchi-Levi, D., and Ba{\c{s}}ar, T. (2020).
\newblock Model-free non-stationary rl: Near-optimal regret and applications in multi-agent rl and inventory control.
\newblock {\em arXiv preprint arXiv:2010.03161}.

\bibitem[Maskin and Tirole, 2001]{maskin2001markov}
Maskin, E. and Tirole, J. (2001).
\newblock {Markov} perfect equilibrium: I. observable actions.
\newblock {\em Journal of Economic Theory}, 100(2):191--219.

\bibitem[Mnih et~al., 2015]{mnih2015human}
Mnih, V., Kavukcuoglu, K., Silver, D., Rusu, A.~A., Veness, J., Bellemare, M.~G., Graves, A., Riedmiller, M., Fidjeland, A.~K., Ostrovski, G., et~al. (2015).
\newblock Human-level control through deep reinforcement learning.
\newblock {\em nature}, 518(7540):529--533.

\bibitem[Nayyar and Gupta, 2017]{nayyar2017information}
Nayyar, A. and Gupta, A. (2017).
\newblock Information structures and values in zero-sum stochastic games.
\newblock In {\em 2017 American Control Conference (ACC)}, pages 3658--3663. IEEE.

\bibitem[Nayyar et~al., 2013]{nayyar2013common}
Nayyar, A., Gupta, A., Langbort, C., and Ba{\c{s}}ar, T. (2013).
\newblock Common information based {Markov} perfect equilibria for stochastic games with asymmetric information: Finite games.
\newblock {\em IEEE Transactions on Automatic Control}, 59(3):555--570.

\bibitem[Nilim and El~Ghaoui, 2005]{nilim2005robust}
Nilim, A. and El~Ghaoui, L. (2005).
\newblock Robust control of {Markov} decision processes with uncertain transition matrices.
\newblock {\em Operations Research}, 53(5):780--798.

\bibitem[Nowak, 2003]{nowak2003new}
Nowak, A.~S. (2003).
\newblock On a new class of nonzero-sum discounted stochastic games having stationary nash equilibrium points.
\newblock {\em International Journal of Game Theory}, 32(1):121.

\bibitem[Osogami, 2012]{osogami2012robustness}
Osogami, T. (2012).
\newblock Robustness and risk-sensitivity in {Markov} decision processes.
\newblock {\em Advances in neural information processing systems}, 25.

\bibitem[Pan et~al., 2017]{pan2017virtual}
Pan, X., You, Y., Wang, Z., and Lu, C. (2017).
\newblock Virtual to real reinforcement learning for autonomous driving.
\newblock {\em arXiv preprint arXiv:1704.03952}.

\bibitem[Panaganti and Kalathil, 2021]{Panaganti2021}
Panaganti, K. and Kalathil, D. (2021).
\newblock Sample complexity of robust reinforcement learning with a generative model.

\bibitem[Parthasarathy and Sinha, 1989]{parthasarathy1989existence}
Parthasarathy, T. and Sinha, S. (1989).
\newblock Existence of stationary equilibrium strategies in non-zero sum discounted stochastic games with uncountable state space and state-independent transitions.
\newblock {\em International Journal of Game Theory}, 18:189--194.

\bibitem[Pichler and Shapiro, 2021]{pichler2021mathematical}
Pichler, A. and Shapiro, A. (2021).
\newblock Mathematical foundations of distributionally robust multistage optimization.
\newblock {\em SIAM Journal on Optimization}, 31(4):3044--3067.

\bibitem[Post and Ye, 2013]{post2013det_MDP}
Post, I. and Ye, Y. (2013).
\newblock The simplex method is strongly polynomial for deterministic {Markov} decision processes.

\bibitem[Puterman, 1994]{Puterman1994MDP}
Puterman, M.~L. (1994).
\newblock {\em {Markov} Decision Processes: Discrete Stochastic Dynamic Programming}.
\newblock John Wiley \& Sons, Inc., USA, 1st edition.

\bibitem[Ruszczy{\'n}ski, 2010]{ruszczynski2010risk}
Ruszczy{\'n}ski, A. (2010).
\newblock Risk-averse dynamic programming for {Markov} decision processes.
\newblock {\em Mathematical programming}, 125:235--261.

\bibitem[Saghafian, 2023]{saghafian2023ambiguous}
Saghafian, S. (2023).
\newblock Ambiguous dynamic treatment regimes: A reinforcement learning approach.
\newblock {\em Management Science}.

\bibitem[Shalev-Shwartz et~al., 2016]{shalev2016safe}
Shalev-Shwartz, S., Shammah, S., and Shashua, A. (2016).
\newblock Safe, multi-agent, reinforcement learning for autonomous driving.
\newblock {\em arXiv preprint arXiv:1610.03295}.

\bibitem[Shapiro, 2021]{shapiro2021distributionally}
Shapiro, A. (2021).
\newblock Distributionally robust optimal control and {MDP} modeling.
\newblock {\em Operations Research Letters}, 49(5):809--814.

\bibitem[Shapiro, 2022]{shapiro2022distributionally}
Shapiro, A. (2022).
\newblock Distributionally robust modeling of optimal control.
\newblock {\em Operations Research Letters}, 50(5):561--567.

\bibitem[Shapiro et~al., 2021]{shapiro2021lectures}
Shapiro, A., Dentcheva, D., and Ruszczynski, A. (2021).
\newblock {\em Lectures on Stochastic Programming: Modeling and Theory}.
\newblock SIAM.

\bibitem[Shapiro and Pichler, 2016]{shapiro2016time}
Shapiro, A. and Pichler, A. (2016).
\newblock Time and dynamic consistency of risk averse stochastic programs.
\newblock {\em Optimization Online}.

\bibitem[Shapiro and Pichler, 2023]{shapiro2023conditional}
Shapiro, A. and Pichler, A. (2023).
\newblock Conditional distributionally robust functionals.
\newblock {\em Operations Research}.

\bibitem[Shapley, 1953]{shapley1953stochastic}
Shapley, L.~S. (1953).
\newblock Stochastic games.
\newblock {\em Proceedings of the National Academy of Sciences}, 39(10):1095--1100.

\bibitem[Shi and Chi, 2022]{ShiChi2022}
Shi, L. and Chi, Y. (2022).
\newblock Distributionally robust model-based offline reinforcement learning with near-optimal sample complexity.

\bibitem[Shi et~al., 2023]{shi2023curious_price}
Shi, L., Li, G., Wei, Y., Chen, Y., Geist, M., and Chi, Y. (2023).
\newblock The curious price of distributional robustness in reinforcement learning with a generative model.

\bibitem[Silver et~al., 2018]{silver2018general}
Silver, D., Hubert, T., Schrittwieser, J., Antonoglou, I., Lai, M., Guez, A., Lanctot, M., Sifre, L., Kumaran, D., Graepel, T., et~al. (2018).
\newblock A general reinforcement learning algorithm that masters chess, shogi, and go through self-play.
\newblock {\em Science}, 362(6419):1140--1144.

\bibitem[Silver et~al., 2017]{silver2017mastering}
Silver, D., Schrittwieser, J., Simonyan, K., Antonoglou, I., Huang, A., Guez, A., Hubert, T., Baker, L., Lai, M., Bolton, A., et~al. (2017).
\newblock Mastering the game of go without human knowledge.
\newblock {\em nature}, 550(7676):354--359.

\bibitem[Simon, 2003]{simon2003games}
Simon, R.~S. (2003).
\newblock Games of incomplete information, ergodic theory, and the measurability of equilibria.
\newblock {\em Israel Journal of Mathematics}, 138(1):73--92.

\bibitem[Sion, 1958]{Sion1958minmax}
Sion, M. (1958).
\newblock {On general minimax theorems.}
\newblock {\em Pacific Journal of Mathematics}, 8(1):171 -- 176.

\bibitem[Solan and Vieille, 2015]{solan2015stochastic}
Solan, E. and Vieille, N. (2015).
\newblock Stochastic games.
\newblock {\em Proceedings of the National Academy of Sciences}, 112(45):13743--13746.

\bibitem[Takahashi, 1964]{takahashi1964equilibrium}
Takahashi, M. (1964).
\newblock Equilibrium points of stochastic non-cooperative $ n $-person games.
\newblock {\em Journal of Science of the Hiroshima University, Series A-I (Mathematics)}, 28(1):95--99.

\bibitem[Tamar et~al., 2016]{tamar2016sequential}
Tamar, A., Chow, Y., Ghavamzadeh, M., and Mannor, S. (2016).
\newblock Sequential decision making with coherent risk.
\newblock {\em IEEE transactions on automatic control}, 62(7):3323--3338.

\bibitem[Taşkesen et~al., 2023]{taşkesen2023DRLQG}
Taşkesen, B., Iancu, D.~A., Çağıl Koçyiğit, and Kuhn, D. (2023).
\newblock Distributionally robust linear quadratic control.

\bibitem[Wang et~al., 2023a]{wang2023DMDP_mixing}
Wang, S., Blanchet, J., and Glynn, P. (2023a).
\newblock Optimal sample complexity of reinforcement learning for mixing discounted {Markov} decision processes.

\bibitem[Wang et~al., 2023b]{Wang2023MLMCDRQL}
Wang, S., Si, N., Blanchet, J., and Zhou, Z. (2023b).
\newblock A finite sample complexity bound for distributionally robust {Q}-learning.

\bibitem[Wang et~al., 2023c]{wang2023VRDRQL}
Wang, S., Si, N., Blanchet, J., and Zhou, Z. (2023c).
\newblock Sample complexity of variance-reduced distributionally robust {Q}-learning.

\bibitem[Watkins and Dayan, 1992]{watkins1992q}
Watkins, C.~J. and Dayan, P. (1992).
\newblock {Q}-learning.
\newblock {\em Machine learning}, 8:279--292.

\bibitem[Wiesemann et~al., 2013]{wiesemann2013robust}
Wiesemann, W., Kuhn, D., and Rustem, B. (2013).
\newblock Robust {Markov} decision processes.
\newblock {\em Mathematics of Operations Research}, 38(1):153--183.

\bibitem[Xu and Mannor, 2010]{xu2010distributionally}
Xu, H. and Mannor, S. (2010).
\newblock Distributionally robust markov decision processes.
\newblock In {\em Advances in Neural Information Processing Systems}, volume~23, pages 2505--2513.

\bibitem[Xu et~al., 2023]{xu2023improved}
Xu, Z., Panaganti, K., and Kalathil, D. (2023).
\newblock Improved sample complexity bounds for distributionally robust reinforcement learning.

\bibitem[Yang, 2021]{Yang2021Wasserstein_dr_control}
Yang, I. (2021).
\newblock Wasserstein distributionally robust stochastic control: A data-driven approach.
\newblock {\em IEEE Transactions on Automatic Control}, 66(8):3863--3870.

\bibitem[Yang et~al., 2023]{yang2023avoiding}
Yang, W., Wang, H., Kozuno, T., Jordan, S.~M., and Zhang, Z. (2023).
\newblock Avoiding model estimation in robust {Markov} decision processes with a generative model.

\bibitem[Yang et~al., 2021]{yang2021}
Yang, W., Zhang, L., and Zhang, Z. (2021).
\newblock Towards theoretical understandings of robust {Markov} decision processes: Sample complexity and asymptotics.

\bibitem[Zhalechian et~al., 2023]{zhalechian2023data}
Zhalechian, M., Keyvanshokooh, E., Shi, C., and Van~Oyen, M.~P. (2023).
\newblock Data-driven hospital admission control: A learning approach.
\newblock {\em Operations Research}.

\bibitem[Zhao et~al., 2018]{zhao2018deep}
Zhao, J., Qiu, G., Guan, Z., Zhao, W., and He, X. (2018).
\newblock Deep reinforcement learning for sponsored search real-time bidding.
\newblock In {\em Proceedings of the 24th ACM SIGKDD international conference on knowledge discovery \& data mining}, pages 1021--1030.

\bibitem[Zhou et~al., 2021]{zhou21}
Zhou, Z., Zhou, Z., Bai, Q., Qiu, L., Blanchet, J., and Glynn, P. (2021).
\newblock Finite-sample regret bound for distributionally robust offline tabular reinforcement learning.
\newblock In Banerjee, A. and Fukumizu, K., editors, {\em Proceedings of The 24th International Conference on Artificial Intelligence and Statistics}, volume 130 of {\em Proceedings of Machine Learning Research}, pages 3331--3339. PMLR.

\bibitem[Zurek and Chen, 2024]{Zurek_2024_plugin_SC_MDP}
Zurek, M. and Chen, Y. (2024).
\newblock The plug-in approach for average-reward and discounted {MDPs}: Optimal sample complexity analysis.

\end{thebibliography}

\newpage
\appendix
\appendixpage
\section{Proofs of Auxiliary Results}

\subsection{Proof of Lemma \ref{lemma:top->bot_right->left}}\label{A_sec:proof:lemma:top->bot_right->left}
\begin{proof}{Proof}
The first inequality follows simply from $\PiH\supset\PiM\supset\PiS$. For the second inequality,
recall the inclusion relation \eqref{eqn:adv_policy_sets_order_hist}. Thus, for any $\pi\in \Pi = \PiS,\PiM,\PiH$
\[
\inf_{\kappa\in\KSH} v(\mu,\pi,\kappa)\leq \inf_{\kappa\in\KSM} v(\mu,\pi,\kappa) \leq \inf_{\kappa\in\KSS} v(\mu,\pi,\kappa). 
\]
This implies the second inequality. 
\end{proof}

\subsection{Proof of Proposition \ref{prop:bellman_exst_unq_sol}}\label{A_sec:proof:prop:bellman_exst_unq_sol}
\begin{proof}{Proof}We consider the robust Bellman operator $\cB:\R^S\ra \R^S$ defined as
\[
\cB(u)(s) := \sup_{\phi\in\cQ}\inf_{p_s\in\cP_s} E_{\phi,p_s}[r(s,A_0) + \gamma u(X_1)],\quad \forall s\in S. 
\]
Then, $\cB$ is a $\gamma$-contraction on $(\R^S,\norminf{\cd})$. \text{I}ndeed for $u,v\in \R^S$, 
\begin{align*}
  \norminf{\cB(u)(s) - \cB(v)(s)} &\stackrel{(i)}{\leq}  \max_{s\in S}\sup_{\phi\in\cQ}\abs{\crbk{\inf_{p_s\in\cP_s}E_{\phi,p_s}[r(s,A_0) + \gamma u(X_1)]} - \crbk{\inf_{p_s\in\cP_s} E_{\phi,p_s}[r(s,A_0) + \gamma v(X_1)]}   }\\
  &=  \gamma \max_{s\in S}\sup_{\phi\in\cQ}\abs{\inf_{p_s\in\cP_s}  E_{\phi,p_s}[u(X_1)] + \sup_{p_s\in\cP_s}  -E_{\phi,p_s}[v(X_1)]} \\
  &\stackrel{(ii)}{\leq}\gamma \max_{s\in S}\sup_{\phi\in\cQ}\max \set{\abs{\sup_{p_s\in\cP_s}  E_{\phi,p_s}[u(X_1)-v(X_1)]} , \abs{\inf_{p_s\in\cP_s}  E_{\phi,p_s}[u(X_1)-v(X_1)]}}\\
  &\leq\gamma\norminf{u-v}
\end{align*}
where $(i)$ follows from the inequality $|\sup_x f(x) - \sup_x g(x)|\leq \sup_x | f(x) -  g(x)|$ and $(ii)$ is due to $\inf_x [f+g] \leq \inf_xf + \sup_xg\leq \sup_x[f + g]$. Therefore, by the Banach fixed point theorem, there exists a unique $u^*\in \R^S$ s.t. $\cB(u^*) = u^*$. Moreover, 
\begin{align*}
\norminf{u^*} &= \norminf{\cB(u^*)}\leq \norminf{r} + \gamma \norminf{u^*}. 
\end{align*}
Rearrange and use the upper bound $\norminf{r} = 1$ proves that $\norminf{u^*}\leq 1/(1-\gamma)$. 
\end{proof}

\subsection{Time-Truncation Technique}
To prove the remaining results in the paper, a technique that we will be using in the subsequent proof is to truncate the infinite horizon reward to a finite time $T$ that is large enough. Specifically, consider for any $\pi\in\PiH$ and $\kappa\in \KSH$, 
\begin{equation}\label{eqn:truncate_v_err_bd}
  \begin{aligned}
v(\mu,\pi,\kappa)&= E_\mu^{\pi,\kappa}\sqbk{\sum_{k=0}^{T-1}\gamma^k r(X_k,A_k) + \gamma^T u^*(X_{T}) }+E_\mu^{\pi,\kappa}\sqbk{\sum_{k=T}^\infty \gamma^k r(X_k,A_k) -  \gamma^T u^*(X_{T}) }\\
&= E_\mu^{\pi,\kappa}\sqbk{\sum_{k=0}^{T-1}\gamma^k r(X_k,A_k) + \gamma^T u^*(X_{T}) }\\
&\quad +\frac{\gamma^T}{1-\gamma}E_\mu^{\pi,\kappa}\sqbk{(1-\gamma)\sum_{k=T}^\infty \gamma^k r(X_k,A_k) -  (1-\gamma)\gamma^T u^*(X_{T}) }\\
&\leq E_\mu^{\pi,\kappa}\sqbk{\sum_{k=0}^{T-1}\gamma^{k-T} r(X_k,A_k) + \gamma^{k-T} u^*(X_{T}) }+\frac{\epsilon}{2}
\end{aligned}  
\end{equation}
if $\gamma^T\leq \epsilon(1-\gamma)/4$. Define 
\begin{equation}\label{eqn:def_v_mu^pi}
v_\mu^\pi(T,\kappa) := E_\mu^{\pi,\kappa}\sqbk{\sum_{k=0}^{T-1}\gamma^k r(X_k,A_k) + \gamma^T u^*(X_{T})},
\end{equation}
\cblue{where we will use the convention that $\sum_{k=0}^{-1}r_k = 0$ if $T = 0$. \text{I}n light of \eqref{eqn:truncate_v_err_bd}, we will choose $T$ sufficiently large so that $\gamma^T\leq \epsilon(1-\gamma)/4$. Then, for all $\pi\in\PiH$ and $\kappa\in \KSH$, $|v_\mu^\pi(T,\kappa) - v(\mu,\pi,\kappa)|\leq \epsilon/2$. }

\section{Proofs  for Dynamic Programming Theorems in Section \ref{section:DPP_constrained}}\label{A_sec:proofs_for_DPP_constrained}

In this section, we outline some key properties of the maxmin values for S-rectangular robust MDPs. We present a proof of Theorem \ref{thm:s-rec_constrained_checks} and  \ref{thm:s-rec_constrained_max_min_eq_min_max} given these results.

\subsection{Proof of Theorem  \ref{thm:s-rec_constrained_checks}}\label{A_sec:proof:thm:s-rec_constrained_checks}
We break down the proof of Theorem \ref{thm:s-rec_constrained_checks} using the following two key propositions. 
\begin{proposition}\label{prop:s-rec_constrained_diag}
Let $u^*$ be the solution to \eqref{eqn:dr_bellman_eqn}. 
Then, $E_\mu[u^*(X_0)]\geq v(\mu,\PiH,\KSH)$, $E_\mu[u^*(X_0)]\geq v(\mu,\PiM,\KSM)$ and $E_\mu[u^*(X_0)]\geq v(\mu,\PiS,\KSS)$ for all $\mu\in\cP(\cS)$. 
\end{proposition}
\cblue{Proposition \ref{prop:s-rec_constrained_diag} shows that, without any restrictions on the sets $\cQ$ and $\set{\cP_s:s\in S}$, for any given initial distribution $\mu$, the optimal control values for controllers and adversaries with \textit{symmetric} information structures are dominated by the $\mu$-expectation of the solution $u^*$ of \eqref{eqn:dr_bellman_eqn}. On the other hand, by Lemma \ref{lemma:top->bot_right->left}, it is not hard to see that all the optimal control values in Theorem \ref{thm:s-rec_constrained_checks} are lower bounded by $v(\mu,\PiS,\KSH)$. 

\par In the next proposition, we show that this smallest value $v(\mu,\PiS,\KSH)$ always equals $E_\mu [u^*(X_0)]$. }
\begin{proposition}\label{prop:s-rec_constrained_(3,1)}
Let $u^*$ be the solution to \eqref{eqn:dr_bellman_eqn}. 
Then, $ E_\mu[u^*(X_0)] = v(\mu,\PiS,\KSH)$ for all $\mu\in\cP(\cS)$. 
\end{proposition}

\cblue{With these results, we use Lemma \ref{lemma:top->bot_right->left} to upper and lower bound all the relevant optimal control values by $E_\mu [u^*(X_0)]$ and hence establish Theorem \ref{thm:s-rec_constrained_checks}. } 

\begin{proof}{Proof of Theorem \ref{thm:s-rec_constrained_checks}}
First, by Lemma \ref{lemma:top->bot_right->left}, $v(\mu,\PiS,\KSH)$ is the smallest amongst all the 9 values in Lemma \ref{lemma:top->bot_right->left}. This and Propositions \ref{prop:markov_adv_strong_enough} and \ref{prop:s-rec_constrained_(3,1)} imply that
\[
E_\mu[u^*(X_0)]=v(\mu,\PiS,\KSH)\leq v(\mu,\PiM,\KSH)\leq v(\mu,\PiH,\KSH)\leq E_\mu[u^*(X_0)].
\]
Therefore, all the above inequality signs achieve equalities.
In addition, by Lemma \ref{lemma:top->bot_right->left} and Proposition \ref{prop:s-rec_constrained_diag} and \ref{prop:s-rec_constrained_(3,1)}, 
\begin{equation}\label{eqn:to_use_87}
E_\mu[u^*(X_0)] \geq v(\mu,\PiS,\KSS)\geq v(\mu,\PiS,\KSM)\geq v(\mu,\PiS,\KSH) = E_\mu[u^*(X_0)].
\end{equation}
So, all the above inequality signs must achieve equalities as well. Finally, combining \eqref{eqn:to_use_87} and Proposition \ref{prop:s-rec_constrained_diag},
\[
E_\mu[u^*(X_0)] = v(\mu,\PiS,\KSM)\leq v(\mu,\PiM,\KSM)\leq E_\mu[u^*(X_0)]. 
\]
Again, all the above inequality signs must achieve equalities. This completes the proof. \end{proof}

\subsubsection{Proof of Proposition \ref{prop:s-rec_constrained_diag}}\label{A_sec:proof:prop:s-rec_constrained_diag}
\begin{proof}{Proof}
It is equivalent to show that for any $\pi\in \Pi_\mrm{U}$
\[
\inf_{\kappa\in \mrm{K}^{\mrm{S}}_\mrm{U}}v(\mu,\pi,\kappa)\leq E_\mu[u^*(X_0)];
\]
for $\mrm{U} = \mrm{H}, \mrm{M},\mrm{S}$, respectively. 
\par Fix any $\epsilon > 0$. Given a policy $\pi = (\pi_0,\pi_1,\ds)\in\PiH$ (or $\PiM$, $\PiS$), we prove this by constructing an adversarial policy $\kappa = (\kappa_0,\kappa_1,\ds)\in \KSH$ (or $\KSM$, $\KSS$ resp.), so that 
    $v(\mu,\pi,\kappa)\leq E_\mu[u^*(X_0)] +\epsilon.$
\par For all $k\geq 0$ and $h_k = (g_{k-1},s)\in\bd{H}_k$, as $u^*$ solves \eqref{eqn:dr_bellman_eqn},  
\[
\inf_{p_{s}\in\cP_{s}} E_{\pi_k(\cd|h_k),p_s}[r(s,A_0) + \gamma u^*(X_1)]\leq u^*(s). 
\]
Hence, for any $\delta > 0$, there always exists $\kappa_{k}(\cd |h_{k},\cd) \in \cP_{s} $ so that
\begin{equation}\label{eqn:hd_adv_from_bellman_eqn}
\sum_{a\in A}\sum_{s'\in S}\pi_k(a|h_k)\kappa_k(s'|h_k,a)[r(s,a) + \gamma u^*(s')]\leq u^*(s) + \delta
\end{equation}
for all $k\geq 0$ and $h_k = (g_{k-1},s)\in\bd{H}_k$. Moreover, by the same argument, if $\pi_{k}(\cd |h_{k}) = \pi_{k}(\cd |s)$ is Markov (or $\pi_{k}(\cd|h_{k}) = \pi(\cd|s)$ is Markov time-homogeneous), there exists $\kappa_{k}(\cd|h_{k},\cd) = \kappa_{k}(\cd|s,\cd)$ (or $\kappa_{k}(\cd|h_{k},\cd) = \kappa(\cd|s,\cd)$ Markov time-homogeneous) so that \eqref{eqn:hd_adv_from_bellman_eqn} holds for all $k\geq 0$ and $h_k = (g_{k-1},s)\in\bd{H}_k$. We construct an adversarial policy $\kappa := (\kappa_0,\kappa_1,\ds)\in \KSH$ (or $\KSM$, $\KSS$ resp.) with $\set{\kappa_k:k\geq 0}$ specified above. 
\par With this construction, we consider
\begin{equation}\label{eqn:E_r_gamma_u_bound}
\begin{aligned}
 E_\mu^{\pi,\kappa}\sqbk{ r(X_{k-1},A_{k-1}) + \gamma u^*(X_{k})} 
 &=E_\mu^{\pi,\kappa}E_\mu^{\pi,\kappa}\sqbkcond{ r(X_{k-1},A_{k-1}) + \gamma u^*(X_{k})}{H_{k-1}} \\
 &\stackrel{(i)}{=} E_\mu^{\pi,\kappa }\sum_{a\in A}\sum_{s'\in S}\pi_k(a|H_{k-1})\kappa_k(s'|H_{k-1},a)[r(X_{k-1},a) + \gamma u^*(s')]\\
 &\stackrel{(ii)}{\leq}  E_\mu^{\pi,\kappa}u^*(X_{k-1}) +\delta
\end{aligned}
\end{equation}
where $(i)$ follows from that given $H_{k-1},A_{k-1} = a$, $X_k\sim \kappa_{k}(\cd|H_{k-1},a)$; $(ii)$ follows from  \eqref{eqn:hd_adv_from_bellman_eqn}.
\par We use this inequality to bound 
$v_\mu^\pi(T,\kappa)$ defined in \eqref{eqn:def_v_mu^pi}. We choose $T$ large so that for all $\pi,\kappa$, $|v_\mu^\pi(T,\kappa) - v(\mu,\pi,\kappa)|\leq \epsilon/2$. Then, 
\begin{align*}
v_\mu^\pi(T,\kappa) &= E_\mu^{\pi,\kappa}\sqbk{\sum_{k=0}^{T-1}\gamma^k r(X_k,A_k)} +\gamma^T E_\mu^{\pi,\kappa}\sqbk{ u^*(X_{T})}\\
&\stackrel{(i)}{\leq} E_\mu^{\pi,\kappa}\sqbk{\sum_{k=0}^{T-1}\gamma^k r(X_k,A_k)} +\gamma^{T-1} \sqbk{E_\mu^{\pi,\kappa} u^*(X_{T-1}) - E_\mu^{\pi,\kappa}r(X_{T-1},A_{T-1}) +\delta}\\
&= E_\mu^{\pi,\kappa}\sqbk{\sum_{k=0}^{T-2}\gamma^k r(X_k,A_k) +\gamma^{T-1}  u^*(X_{T-1})  } +\gamma^{T-1}\delta\\
&= v_\mu^\pi(T-1,\kappa) +  \gamma^{T-1}\delta\\
&\leq \ds \\
&\leq v_\mu^\pi(0,\kappa)  +\sum_{k = 0}^{T-1}\gamma ^ k \delta \\
& \leq E_\mu[u^*(X_0)] + \frac{\delta}{1-\gamma}
\end{align*}
where $(i)$ uses \eqref{eqn:E_r_gamma_u_bound}. Since $\delta > 0$ is arbitrary, we let $\delta = \epsilon(1-\gamma)/2$ and conclude that 
\begin{align*}
 v(\mu,\pi,\kappa)\leq  v^\pi_\mu(T,\kappa') + \frac{\epsilon}{2}\leq E_\mu[u^*(X_0)] + \epsilon.  
\end{align*}
Recall that $\epsilon >0$ and $\pi\in\PiH$ (or $\PiM$, $\PiS$) are arbitrary, with corresponding $\kappa\in\KSH$ (or $\KSM$, $\KSS$ resp.) This completes the proof of the proposition. 
\end{proof}

\subsubsection{Proof of Proposition \ref{prop:s-rec_constrained_(3,1)}}\label{A_sec:proof:s-rec_(3,1)_lb}
\begin{proof}{Proof}
\par Since $u^*$ is the solution to \eqref{eqn:dr_bellman_eqn}, for any fixed $\delta > 0$, there exists Markov time-homogeneous decision rule $\Delta:S\ra \cQ$ s.t. for all $s\in S$, 
\begin{equation}\label{eqn:bellman_sup_pi_delta}
u^*(s)\leq  \inf_{p_s\in\cP_s}E_{\Delta(\cd|s),p_s}[r(s,A_0) + \gamma u^*(X_1)] + \delta.    
\end{equation}  
Let $\pi = (\Delta,\Delta,\ds)\in\PiS$. We consider for any $\kappa\in \KSH$, 

\begin{equation}\label{eqn:gamma_E_u_bd}
\begin{aligned}
\gamma E_\mu^{\pi,\kappa}\sqbk{ u^*(X_{k})} &= \gamma E_\mu^{\pi,\kappa}E_\mu^{\pi,\kappa}\sqbkcond{ u^*(X_{k})}{G_{k-1}} \\
&\stackrel{(i)}{=}\gamma E_\mu^{\pi,\kappa} \sum_{s'\in S} \kappa_{k-1}(s'|G_{k-1}) u^*(s')\\
&= \gamma E_\mu^{\pi,\kappa} E_\mu^{\pi,\kappa} \sqbkcond{  \sum_{s'\in S} \kappa_{k-1}(s'|G_{k-1}) u^*(s')}{H_{k-1}}\\
&= \gamma E_\mu^{\pi,\kappa} \sum_{a\in A} \sum_{s'\in S} \Delta(a|X_{k-1}) \kappa_{k-1}(s'|H_{k-1},a)u^*(s')
\end{aligned}
\end{equation}
where $(i)$ is because under $E_\mu^{\pi,\kappa}[\cd |G_{k-1}]$, $X_k\sim \kappa_{k-1}(\cd|G_{k-1})$ and the last equality follows from that under $E_\mu^{\pi,\kappa}[\cd |H_{k-1}]$, $A_{k-1}\sim \Delta(\cd |X_{k-1})$. 
\par Now, we define the random adversarial decision $K_{k-1}(\cd|X_{k-1},a)= \kappa_{k-1}(\cd|H_{k-1},a)$.  By \eqref{eqn:gamma_E_u_bd}
\begin{equation}\label{eqn:Er_gamma_u_bd_2}
\begin{aligned}
 E_\mu^{\pi,\kappa}\sqbk{ r(X_{k-1},A_{k-1}) + \gamma u^*(X_{k})}
 &= E_\mu^{\pi,\kappa}E_{\mu}^{\pi,\kappa}\sqbk{ r(X_{k-1},A_{k-1}) + \gamma u^*(X_{k})|H_{k-1}}\\
 &= E_\mu^{\pi,\kappa} \sum_{a\in A} \sum_{s'\in S} \Delta(a|X_{k-1}) \kappa_{k-1}(s'|H_{k-1},a)[r(X_{k-1},a)+\gamma u^*(s')]\\
 &= E_\mu^{\pi,\kappa} \sum_{a\in A} \sum_{s'\in S} \Delta(a|X_{k-1})K_{k-1}(s'|X_{k-1},a) [r(X_{k-1},a)+\gamma u^*(s')]\\
 &= E_\mu^{\pi,\kappa} \crbk{E_{\Delta(\cd|s), K_{{k-1}(\cd|s,\cd)}} [r(s,A_0)+\gamma u^*(X_1)]}_{s = X_{k-1}}.
 \end{aligned}
\end{equation}
where the last equality is a just a change of notation. Note that as $\kappa\in \KSH$, we have $K_{k-1}(\cd|X_{k-1},\cd) \in \cP_{X_{k-1}}$ w.p.1. Therefore, 
 \begin{equation}\label{eqn:to_use_178}\begin{aligned}
 E_\mu^{\pi,\kappa}\sqbk{ r(X_{k-1},A_{k-1}) + \gamma u^*(X_{k})} 
&\geq  E_\mu^{\pi,\kappa} \crbk{\inf_{p_{s}\in \cP_{s}}E_{\Delta(\cd|s), p_s} [r(s,A_0)+\gamma u^*(X_1)]}_{s = X_{k-1}}\\
&\stackrel{(i)}{\geq} E_\mu^{\pi,\kappa} u^*(X_{k-1})  -\delta
\end{aligned}
 \end{equation}
where $(i)$ follows from \eqref{eqn:bellman_sup_pi_delta}.  
\par Fix any $\epsilon> 0$, let $\pi$ be defined as above. 
Recall the definition of $v_\mu^\pi(T,\kappa)$ in \eqref{eqn:def_v_mu^pi}. We still choose $T$ large so that for all $\kappa$, $|v_\mu^\pi(T,\kappa) - v(\mu,\pi,\kappa)|\leq \epsilon/2$. Then, for any $\kappa\in \KSH$, 
\begin{align*}
v_\mu^\pi(T,\kappa) &= E_\mu^{\pi,\kappa}\sqbk{\sum_{k=0}^{T-1}\gamma^k r(X_k,A_k)} +\gamma^T E_\mu^{\pi,\kappa}\sqbk{ u^*(X_{T})}\\
&\stackrel{(i)}{\geq}  E_\mu^{\pi,\kappa}\sqbk{\sum_{k=0}^{T-1}\gamma^k r(X_k,A_k)} +\gamma^{T-1} \sqbk{E_\mu^{\pi,\kappa} u^*(X_{T-1}) - E_\mu^{\pi,\kappa}r(X_{T-1},A_{T-1}) -\delta}\\
&= E_\mu^{\pi,\kappa}\sqbk{\sum_{k=0}^{T-2}\gamma^k r(X_k,A_k) +\gamma^{T-1}  u^*(X_{T-1})  } - \gamma^{T-1}\delta\\
&= v_\mu^\pi(T-1,\kappa) - \gamma^{T-1}\delta\\
&\stackrel{(ii)}{\geq} \ds \\
&\stackrel{(ii)}{\geq} v_\mu^\pi(0,\kappa) -\sum_{k = 0}^{T-1}\gamma ^ k \delta \\
& \geq E_\mu[u^*(X_0)] - \frac{\delta}{1-\gamma},
\end{align*}
where $(i)$ applies inequality \eqref{eqn:to_use_178} and $(ii)$ iteratively repeats the previous steps. 
\par By definition of $\inf$ and that $\pi\in\PiS$, there exists $\kappa'\in \KSH$ s.t.
\begin{equation}\label{eqn:to_use_198}
v(\mu,\pi,\kappa') - \frac{\epsilon}{4}\leq\inf_{\kappa\in \KSH}v(\mu,\pi,\kappa)\leq v(\mu,\PiS,\KSH).
\end{equation}
Then, 
\begin{equation}\label{eqn:PiS_KSH_err_bound}
\begin{aligned}
0&\stackrel{(i)}{\leq} E_\mu[u^*(X_0)] -v(\mu,\PiS,\KSH)\\
&\stackrel{(ii)}{\leq} E_\mu[u^*(X_0)] - v(\mu,\pi,\kappa') + \frac{\epsilon}{4}\\
&\stackrel{(iii)}{\leq} E_\mu[u^*(X_0)] - v^\pi_\mu(T,\kappa') + \frac{3\epsilon}{4}\\
&\leq \frac{\delta}{1-\gamma} +  \frac{3\epsilon}{4}.
\end{aligned}
\end{equation}
where $(i)$ follows from Proposition \ref{prop:s-rec_constrained_diag}, $(ii)$ uses inequality \eqref{eqn:to_use_198}, and $(iii)$ is due to the choice of $T$. Since $\delta$ can be chosen as $(1-\gamma)\epsilon/4$, this implies $ E_\mu[u^*(X_0)] = v(\mu,\PiS,\KSH)$ as $\epsilon > 0$ is arbitrary. 

\end{proof}

\subsection{Proof of Corollary \ref{cor:det_ctrl_q_func}}\label{A_sec:proof:cor:det_ctrl_q_func}

\begin{proof}{Proof}
The uniqueness follows from the same contraction argument as in the proof of Proposition \ref{prop:bellman_exst_unq_sol}. Moreover,
\begin{align*}
u^*(s) =\sup_{\phi\in\cQ^\mrm{D}} \inf_{p_s\in\cP_s}  E_{\phi,p_s}\sqbk{r(s,A_0 ) + \gamma u^*(X_1)}= \max_{a\in A}q^*(s,a).
\end{align*}
Plug the above equality into \eqref{eqn:dr_bellman_eqn_q_original}, one sees that $q^*$ is indeed the solution, hence equivalently solves \eqref{eqn:dr_bellman_eqn_q}. 
Lastly, Theorem \ref{thm:s-rec_constrained_checks} still holds in the context of $\cQ =\cQ^{\mrm{D}}$. This implies the statement of the corollary. 
\end{proof}

\subsection{Proof of Theorem \ref{thm:s-rec_constrained_max_min_eq_min_max}}\label{A_sec:proof:thm:s-rec_constrained_max_min_eq_min_max(1)}
\par We recall that in the setting of Theorem \ref{thm:s-rec_constrained_max_min_eq_min_max}, $\cQ\subset\cP(\cA)$ can be an arbitrary subset. We highlight a consequence of the extra assumption \eqref{eqn:dr_minmax_bellman_eqn} on $u^*$, namely the largest maxmin control value $v(\mu,\PiH,\KSS)$ is upper bounded by $ E_\mu[u^*(X_0)]$. Recall that by Proposition \ref{prop:s-rec_constrained_(3,1)}, we always have that the smallest maxmin control value $v(\mu,\PiS,\KSH) = E_\mu[u^*(X_0)]$. Together, these facts would imply the first item in Theorem \ref{thm:s-rec_constrained_max_min_eq_min_max}. 
\begin{proposition} \label{prop:minmax_maxmin_upper_bound}
Assume that the solution $u^*$ to \eqref{eqn:dr_bellman_eqn} further satisfies \eqref{eqn:dr_minmax_bellman_eqn}, then
$v(\mu,\PiH,\KSS)\leq E_\mu[u^*(X_0)]$.
\end{proposition}

\par The proof to this proposition is deferred to Appendix \ref{A_sec:proof:minmax_maxmin_upper_bound_maxmin_minmax_delta_kappa}. Assuming the validity of this proposition, we proof of item 1 in Theorem \ref{thm:s-rec_constrained_max_min_eq_min_max}. 
\begin{proof}{Proof of Item 1 in Theorem \ref{thm:s-rec_constrained_max_min_eq_min_max}}
Note that, by Lemma \ref{lemma:top->bot_right->left}, 
\[
v(\mu,\PiS,\KSH)\leq v(\mu,\Pi,\mrm{K})\leq v(\mu,\PiH,\KSS)
\]
for all $\mu\in\cP(\cS)$ and any mixture of $\mrm{K} = \KSH, \KSM,\KSS$ and $\Pi = \PiH, \PiM,\PiS$. This, Theorem \ref{thm:s-rec_constrained_checks}, and Proposition \ref{prop:minmax_maxmin_upper_bound} implies that
\[
E_\mu[u^*(X_0)]= v(\mu,\PiS,\KSH)\leq v(\mu,\Pi,\mrm{K})\leq v(\mu,\PiH,\KSS)\leq v(\mu,\PiH,\KSS).
\]
So, the inequalities are indeed equalities. This implies item 1 of Theorem \ref{thm:s-rec_constrained_max_min_eq_min_max}. 
\end{proof}

Next we prove item 2 of Theorem \ref{thm:s-rec_constrained_max_min_eq_min_max}. We will make use of the following proposition: 
\begin{proposition}\label{prop:maxmin_minmax_delta_kappa}
Assume the assumptions of Proposition \ref{prop:minmax_maxmin_upper_bound}. Then, for any $\epsilon > 0$, there exists $\kappa\in\KSS$ such that
\[
\sup_{\pi\in\PiH}v(\mu,\pi,\kappa)\leq E_\mu[u^*(X_0)] + \epsilon.
\]
\end{proposition}
The proof of Proposition \ref{prop:maxmin_minmax_delta_kappa} is deferred to the appendix Section \ref{A_sec:proof:minmax_maxmin_upper_bound_maxmin_minmax_delta_kappa} as well. Equipped with Proposition \ref{prop:maxmin_minmax_delta_kappa}, we can establish item 2 of Theorem \ref{thm:s-rec_constrained_max_min_eq_min_max}. 

\begin{proof}{Proof of Item 2 in Theorem \ref{thm:s-rec_constrained_max_min_eq_min_max}}
Note that by assumption, weak duality of the sup-inf interchange and set inclusions, it is clear that
\begin{equation}\label{eqn:u=supinf<=infsup}
E_\mu[u^*(X_0)]=\sup_{\pi\in\Pi}\inf_{\kappa\in\mrm{K}}v(\mu,\pi,\kappa)\leq\inf_{\kappa\in\mrm{K}}\sup_{\pi\in\Pi}v(\mu,\pi,\kappa)\leq \inf_{\kappa\in\KSS}\sup_{\pi\in\PiH}v(\mu,\pi,\kappa)
\end{equation}
for any mixture of $\mrm{K} = \KSH, \KSM,\KSS$ and $\Pi = \PiH, \PiM,\PiS$. On the other hand, the statement of Proposition \ref{prop:maxmin_minmax_delta_kappa} implies that   
\[
\inf_{\kappa\in\KSS}\sup_{\pi\in\PiH}v(\mu,\pi,\kappa)\leq E_\mu[u^*(X_0)].
\]
Therefore, we conclude that the inequalities in \eqref{eqn:u=supinf<=infsup} must be equalities. This proves item 2 of Theorem \ref{thm:s-rec_constrained_max_min_eq_min_max}. 
\end{proof}

\subsubsection{Proof of Propositions \ref{prop:minmax_maxmin_upper_bound} and \ref{prop:maxmin_minmax_delta_kappa} }\label{A_sec:proof:minmax_maxmin_upper_bound_maxmin_minmax_delta_kappa}
\begin{proof}{Proof of Proposition \ref{prop:minmax_maxmin_upper_bound}}
\par To prove Proposition \ref{prop:minmax_maxmin_upper_bound}, we use the min-max equation \eqref{eqn:dr_minmax_bellman_eqn} to construct a time-homogeneous adversary $\kappa = \set{\psi,\psi,\ds}$ that upper bounds the control value. 
\par Since $u^*$ satisfies \eqref{eqn:dr_minmax_bellman_eqn}, for any $\delta > 0$, there exists a Markov time-homogeneous adversarial decision rule $\psi = \set{\psi(\cd|s,\cd)\in\cP_s:s\in S}$ s.t. 
\[
u^*(s) \geq \sup_{\phi\in \cQ} E_{\phi,\psi(\cd|s,\cd)}[r(s,A_0) + \gamma u^*(X_1)] - \delta. 
\]
Let $\kappa = (\psi,\psi,\ds)\in\KSS$ and $\pi = (\pi_0,\pi_1,\pi_2\ds)\in\PiH$ be any controller policy.  Consider
\begin{align*}
 E_\mu^{\pi,\kappa}\sqbk{ r(X_{k-1},A_{k-1}) + \gamma u^*(X_{k})} 
 &=E_\mu^{\pi,\kappa}E_\mu^{\pi,\kappa}\sqbkcond{ r(X_{k-1},A_{k-1}) + \gamma u^*(X_{k})}{H_{k-1}} \\
 &\stackrel{(i)}{=} E_\mu^{\pi,\kappa} \sum_{a\in A}\sum_{s'\in S}\pi_{k-1}(a|H_{k-1})\psi(s'|X_{k-1},a)[r(X_{k-1},a) + \gamma u^*(s')]\\
 &\leq E_\mu^{\pi,\kappa}\sup_{\phi\in\cQ} \sum_{a\in A}\sum_{s'\in S}\phi(a)\psi(s'|X_{k-1},a)[r(X_{k-1},a) + \gamma u^*(s')]\\
 &=E_\mu^{\pi,\kappa}\crbk{ \sup_{\phi\in \cQ} E_{\phi,\psi(\cd|s,\cd)}[r(s,A_0) + \gamma u^*(X_1)] }_{s = X_{k-1}} \\
 &\stackrel{(ii)}{\leq}  E_\mu^{\pi,\kappa}u^*(X_{k-1}) +\delta
\end{align*}
where $(i)$ follows from the construction that under $E_\mu^{\pi,\kappa}[\cd |H_{k-1}]$, $A_{k-1},X_k\sim E_{\pi_k(\cd |X_{H_{k-1}}),\psi(\cd |X_{k-1},\cd)}$, and $(ii)$ follows from the definition of $\psi$. 
\par Fix any $\epsilon> 0$. Choose $T$ large s.t. $|v_\mu^\pi(T,\kappa) - v(\mu,\pi,\kappa)|\leq \epsilon/2$. Then, 
\begin{align*}
v_\mu^\pi(T,\kappa) &= E_\mu^{\pi,\kappa}\sqbk{\sum_{k=0}^{T-1}\gamma^k r(X_k,A_k)} +\gamma^T E_\mu^{\pi,\kappa}\sqbk{ u^*(X_{T})}\\
&\leq E_\mu^{\pi,\kappa}\sqbk{\sum_{k=0}^{T-1}\gamma^k r(X_k,A_k)} +\gamma^{T-1} \sqbk{E_\mu^{\pi,\kappa} u^*(X_{T-1}) - E_\mu^{\pi,\kappa}r(X_{T-1},A_{T-1}) +\delta}\\
&= E_\mu^{\pi,\kappa}\sqbk{\sum_{k=0}^{T-2}\gamma^k r(X_k,A_k) +\gamma^{T-1}  u^*(X_{T-1})  } +\gamma^{T-1}\delta\\
&= v_\mu^\pi(T-1,\kappa) +  \gamma^{T-1}\delta\\
&\leq \ds \\
&\leq v_\mu^\pi(0,\kappa)  +\sum_{k = 0}^{T-1}\gamma ^ k \delta \\
& \leq E_\mu[u^*(X_0)] + \frac{\delta}{1-\gamma}.
\end{align*}
Since $\delta > 0$ is arbitrary, we let $\delta = \epsilon(1-\gamma)/2$ and conclude that 
\begin{align*}
\inf_{\kappa\in\KSS}v(\mu,\pi,\kappa)\leq v(\mu,\pi,\kappa)\leq  v^\pi_\mu(T,\kappa') + \frac{\epsilon}{2}\leq E_\mu[u^*(X_0)] + \epsilon.  
\end{align*}
Since $\pi\in\PiH$ and $\epsilon >0$ are arbitrary, this implies the proposition.
\end{proof}
Next, we present the proof of Proposition \ref{prop:maxmin_minmax_delta_kappa} using the results in the previous proof. 
\begin{proof}{Proof of Proposition \ref{prop:maxmin_minmax_delta_kappa}}
For any $\delta > 0$, let $\kappa = (\psi,\psi,\ds)\in\KSS$ be constructed as above. Then the same proof as above implies that for any $\pi\in \PiH$, 
\[
v_\mu^\pi(T,\kappa)\leq E_\mu[u^*(X_0)] + \frac{\delta}{1-\gamma}.
\]
Therefore, by the choice of $T$ and the definition of sup, there exists $\pi'\in\PiH$ s.t.
\begin{align*}
\sup_{\pi\in\PiH}v(\mu,\pi,\kappa)&\leq v(\mu,\pi',\kappa) + \frac{\epsilon}{4}\\
&\leq v_\mu^{\pi'}(T,\kappa) + \frac{3\epsilon}{4}   \\
&\leq E_\mu[u^*(X_0)] + \frac{\delta}{1-\gamma} + \frac{3\epsilon}{4}. 
\end{align*}
Since $\delta > 0$ is arbitrary, choose $\delta = (1-\gamma)\epsilon/4$ will complete the proof. 
\end{proof}

\subsection{Proof of Theorem \ref{thm:sa-rec_dpp}}\label{A_sec:proof:thm:sa-rec_dpp}
\begin{proof}{Proof}
Notice that since $\cP_s = \bigtimes_{a\in A}\cP_{s,a}$, the Bellman equation for the $q$-function \eqref{eqn:dr_bellman_eqn_q} can be written as 
\begin{equation}\label{eqn:sa_rec_dr_bellman_eqn_q}
    q(s,a) = r(s,a ) + \gamma \inf_{p\in\cP_{s,a}} p\sqbk{\sup_{b\in A}q(\cd,b)}.
\end{equation}
\par Since $u^*$ solves \eqref{eqn:dr_bellman_eqn} and $\cQ\subset \cP(\cA)$, we naturally have that
\begin{align*}
u^*(s)&\leq \sup_{\phi\in\cP(\cA)} \inf_{p_s\in\cP_s}\sum_{a\in A}\phi(a)r(s,a) + \gamma \phi(a) \sum_{s'\in S}p_{s,a}(s')u^*(s')\\
&\stackrel{(i)}{ = }\sup_{\phi\in\cP(\cA)} \inf_{p_{s,a_1}\in\cP_{s,a_1}}\ds \inf_{p_{s,a_{|A|}}\in\cP_{s,a_{|A|}}}\sum_{a\in A}\phi(a)r(s,a) + \gamma \phi(a) \sum_{s'\in S}p_{s,a}(s')u^*(s')\\
&= \sup_{\phi\in\cP(\cA)} \sum_{a\in A}\phi(a)r(s,a) + \gamma \phi(a) \inf_{p_{s,a}\in\cP_{s,a}}\sum_{s'\in S}p_{s,a}(s')u^*(s')\\
&\stackrel{(ii)}{=} \max_{a\in A} r(s,a) + \gamma\inf_{\psi\in\cP_{s,a}}E_{\psi}u^*(X_1)\\
&=\max_{a\in A}q^*(s,a)
\end{align*}
where $(i)$ follows from SA-rectangularity and the enumeration $A = \set{a_1,a_2,\ds,a_{|A|}}$, and $(ii)$ follows from a standard result in linear programming that optima are achieved at extreme points $\cQ^{\mrm{D}}$ of $\cP(\cA)$. 
\par On the other hand, since $\cQ^{\mrm{D}}\subset \cQ$, 
\begin{align*}
u^*(s)&\geq\max_{a\in A} \inf_{p_s\in\cP_s}r(s,a) + \gamma E_{\delta_{\set{a}},p_s}[u^*(X_1)]\\
&\stackrel{(i)}{ = }\max_{a\in A} r(s,a) + \gamma \inf_{\psi\in\cP_{s,a}} E_\psi[u^*(X_1)]\\
&=\max_{a\in A}q^*(s,a)
\end{align*}
where $(i)$ again follows from SA-rectangularity. Therefore, $u^*(s) = \max_{a\in A}q^*(s,a)$ and hence $q^*$ solves \eqref{eqn:sa_rec_dr_bellman_eqn_q}. 
\par Finally, we show that $u^*$ solves \eqref{eqn:dr_minmax_bellman_eqn}. By the definition of $q^*$, for any $\delta > 0$, there exists $\psi_{s,a}\in\cP_{s,a}$ s.t. 
\[
r(s,a) + \gamma E_{\psi_{s,a}}[u^*(X_1)]-\delta\leq q^*(s,a)
\]
for all $(s,a)\in S\times A$. Then, we consider
\begin{align*}
u^*(s) &\leq \inf_{p_s\in\cP_s}\sup_{\phi\in\cQ} E_{\phi,p_s}[r(s,A_0)  +\gamma u^*(X_1)]\\
&\leq  \sup_{\phi\in\cP(\cA)} E_\phi\sqbk{r(s,A)  + \gamma \sum_{s'\in S}\psi_{s,A_0} u^*(s')}\\ 
&\stackrel{(i)}{=} \max_{a\in A} r(s,a)  +\gamma E_{\psi_{s,a}}[u^*(X_1)]\\
&\leq \max_{a\in A}q^*(s,a) + \delta\\
&= u^*(s) + \delta
\end{align*}
where $(i)$ again follows from the LP fact that the optima are achieved at $\cQ^{\mrm{D}}$. Since $\delta$ is arbitrary, all inequalities above must be equalities. \text{I}n particular, $u^*$ satisfies \eqref{eqn:dr_minmax_bellman_eqn}. 
\end{proof}

\subsection{Proof of Theorem \ref{thm:s-rec_conv_constrained_checks}}\label{A_sec:proof:thm:s-rec_conv_constrained_checks}

\begin{proof}{Proof}
We highlight the effect of the convexity of $\cQ$ by introducing the following proposition. 
\begin{proposition}\label{prop:markov_adv_strong_enough}
Let the set of action distributions $\cQ$ be convex, and $u^*$ be the solution to \eqref{eqn:dr_bellman_eqn}. Then, 
\[
v(\mu,\PiH,\KSM)\leq E_\mu[u^*(X_0)]
\]
for all $\mu\in\cP(\cS)$. 
\end{proposition}
The proof of this key proposition is deferred to Appendix \ref{A_sec:proof:markov_adv_strong_enough}.

\par Given Theorem \ref{thm:s-rec_constrained_checks}, it remains to show that $v(\mu,\PiH,\KSM)= E_\mu[u^*(X_0)]$. By Theorem \ref{thm:s-rec_constrained_checks}, Lemma \ref{lemma:top->bot_right->left}, and Proposition \ref{prop:markov_adv_strong_enough}, we have
\[
E_\mu[u^*(X_0)] = v(\mu,\PiS,\KSM)\leq v(\mu,\PiH,\KSM)\leq E_\mu[u^*(X_0)].
\]
Therefore, the inequalities must be equalities, completing the proof. 
\end{proof}

\subsubsection{Proof of Proposition \ref{prop:markov_adv_strong_enough}}\label{A_sec:proof:markov_adv_strong_enough}
\begin{proof}{Proof}
\par Fix an arbitrary $\pi = (\pi_1,\pi_2,\ds)\in\PiH$. Recall that by definition of $\PiH$, for fixed $h_k\in\bd{H}_k$, $\pi_k(\cd|h_k)\in\cQ$, where $\cQ$ is assumed to be convex. \cblue{Similar to the spirit of previous proofs, for any $\epsilon > 0$, we construct a Markovian adversarial policy (in this case, $\kappa^{(T-1)}$) so that $v(\mu,\pi,\kappa^{(T-1)})\leq E_\mu[u^*(X_0)]+\epsilon$. However, the construction in this case is more intricate: a sufficiently powerful Markovian adversary capable of achieving this inequality must strategically compensate for its lack of access to historical information relative to the controller. To do so, the adversarial policy \( \kappa^{(T-1)} \) we construct projects the controller’s history-dependent policy onto a Markovian information structure and treats this Markovian projection as an effective reference policy.  
\par To realize this construction, two main challenges arise. First, the Markovian projection of the controller’s decision rule at time \( k \) depends on the adversary’s decision rules from time 1 to \( k-1 \). We address this by constructing the adversarial decision rule iteratively over time. Second, to ensure that the value function is bounded by the solution to the Bellman equation, we need that the projected controller policy have action distributions contained within \( \cQ \). This is ensured by the additional assumption that \( \cQ \) is convex.} 

\par As in \eqref{eqn:truncate_v_err_bd}, we first fix $T\in \N$ large enough so that for all $\pi\in\PiH$ and $\kappa\in \KSH$, $|v_\mu^\pi(T,\kappa) - v(\mu,\pi,\kappa)|\leq \epsilon/2$. We then construct the sequence of Markovian adversarial policies $\set{\kappa^{(-1)},\kappa^{(0)},\ds,\kappa^{(T-1)}}\subset\KSM$ as follow. Fix any time-homogeneous adversary $\kappa'$ and let $\kappa^{(-1)} = \kappa' = (\kappa',\kappa',\ds)$. Then recursively construct
\[
\kappa^{(k+1)} = (\kappa_0^{(k)},\kappa_1^{(k)}, \ds ,\kappa_k^{(k)}, \psi_{k+1}, \kappa',\kappa',\ds)
\]
for some Markovian adversarial action $\psi_{k+1} (\cd|g_{k+1}) = \psi_{k+1}(\cd|s_{k+1},a_{k+1})$ that we will specify later in \eqref{eqn:bellman_property_of_psi}. This leads to
\[
\kappa^{(k)} = (\psi_0, \psi_1,\psi _2,\ds,\psi_k,\kappa',\kappa'\ds )
\]
for all $ k = 0,\ds,T-1$. 
\par Now we specify the recursive procedure to construct $\set{\kappa^{(k)}}$. Note that $\kappa^{(-1)}$ is fully specified. Fix any $0\leq k\leq T-1$, suppose that we have specified $\kappa^{(k-1)}$ (or equivalently, $\psi_0,\ds,\psi_{k-1}$). Let us define the projected Markovian controller's decision rule $\set{\Delta_k(\cd |s)\in\cP(\cA):s\in S}$ by
\begin{equation}\label{eqn:effective_decision_rule}
\Delta_{k}(a|s) := E_\mu^{\pi,\kappa^{(k-1)}}\sqbkcond{\pi_{k}(a|H_{k})}{X_{k} = s},\quad \forall a\in A.
\end{equation}
Since $\pi_{k}(\cd|h_{k})\in\cQ$ where $\cQ$ is convex,
\begin{align*}
    \Delta_{k}(a|s) &= \sum_{h_{k}\in\bd{H}_{k}} E_\mu^{\pi,\kappa^{(k-1)}}\sqbkcond{\pi_{k}(a|H_{k})\dsi\set{H_{k} = h_{k}} }{X_{k} = s}\\
    &= \sum_{g_{k-1}\in\bd{G}_{k-1}} \pi_{k}(a|g_{k-1},s)E_\mu^{\pi,\kappa^{(k-1)}}\sqbkcond{\dsi\set{G_{k-1} = g_{k-1}} }{X_{k} = s}.
\end{align*}
for all $a\in A$. So, $\Delta_{k}(\cd|s)$ is a convex combination of $\set{\pi_{k}(\cd|g_{k-1},s)\in \cQ:g_{k-1}\in\bd{G}_{k-1}}$. Therefore, $\Delta_{k}(\cd|s)\in\cQ$ as $\cQ$ is assumed to be convex. 
\par For time $k$, the adversary chooses its decision by treating the controller's decision as $\Delta_k$ and consulting the Bellman equation. Specifically, fix any $\delta > 0$. Since for all $s\in S$
\begin{equation}\label{eqn:markov_strong_avged_controller_d}
\begin{aligned}
u^*(s) &= \sup_{\phi\in\cQ}\inf_{p_s\in \cP_s} E_{\phi,p_s}[r(s,A_0 ) + \gamma u^*(X_1)]\\
&\geq \inf_{p_s\in \cP_s} E_{\Delta_{k}(\cd|s),p_s}[r(s,A_0 ) + \gamma u^*(X_1)]. 
\end{aligned}
\end{equation}
Hence there exists $\set{\psi_{s}\in\cP_s:s\in S}$ s.t.
\begin{align*}
u^*(s) +\delta \geq  E_{\Delta_{k}(\cd|s),\psi_s}[r(s,A_0 ) + \gamma u^*(X_1)]. 
\end{align*}
Then, for each $s$ and $\psi_s$ as above, we define $\psi_{k}(s'|s,a) = \psi_{s,a}(s')$. Then
\begin{equation}\label{eqn:bellman_property_of_psi}
\sum_{a\in A}\sum_{s'\in S}\Delta_{k}(a|s)\psi_{k}(s'|s,a)[r(s,a) + \gamma u^*(s')]\leq u^*(s) + \delta. 
\end{equation}
Doing this for $k = 0,\ds,T-1$ will complete the construction. 
\par Now, we claim and defer the proof of the inequality
\begin{equation}\label{eqn:u_k_val_bound}
v_\mu^\pi(k+1,\kappa^{(k)})\leq v_\mu^\pi(k,\kappa^{(k-1)})+  \gamma^{k}\delta
\end{equation}
for all $0\leq k\leq T $. With \eqref{eqn:u_k_val_bound}, we see that
\begin{align*}
v_\mu^\pi(T,\kappa^{(T-1)})&\leq v_\mu^\pi(T-1,\kappa^{(T-2)})+\gamma^{T-1}\delta\\
&\leq 
v_\mu^\pi(0,\kappa^{(-1)})+\sum_{k=0}^{T-1}\gamma^k \delta \\
&\leq E_\mu [u^*(X_0)] + \frac{\delta}{1-\gamma}. 
\end{align*}
We choose $\delta = (1-\gamma)\epsilon/2$. Since $\kappa^{(T)}\in \KSM$, 
\begin{align*}
\inf_{\kappa\in\KSM}v(\mu,\pi,\kappa)
&\leq v(\mu,\pi,\kappa^{(T)})\\
&\leq v_\mu^\pi(T,\kappa^{(T)})+\frac{\epsilon}{2}\\
&\leq E_\mu[u^*(X_0)]+\epsilon.    
\end{align*}
As $\pi$ and $\epsilon > 0$ are arbitrary, this implies the statement of Proposition \ref{prop:markov_adv_strong_enough}. 
\par Finally, it remains to prove the claim \eqref{eqn:u_k_val_bound}. We begin with analyzing for each $0\leq k\leq T-1$ the following quantity
\begin{equation}\label{eqn:bellman_term}
\begin{aligned}
&\quad E_\mu^{\pi,\kappa^{(k)}}\sqbk{r(X_{k},A_{k}) +  \gamma u^*(X_{k+1})} \\
&= E_\mu^{\pi,\kappa^{(k)}}\sqbk{ \sum_{a\in A} \sum_{s'\in S}\pi_{k}(a |H_{k})\psi_{k}(s'|X_{k},a)[r(X_{k},a) + \gamma u^*(s')]}
\end{aligned}
\end{equation}
Notice that by conditioning on $X_k$, we have
\begin{align*}
&\quad E_\mu^{\pi,\kappa^{(k)}}\sqbkcond{\sum_{a\in A} \sum_{s'\in S}\pi_{k}(a |H_{k})\psi_{k}(s'|X_{k},a) u^*(s')}{X_{k} = s} \\
&=\sum_{a\in A} \sum_{s'\in S}\psi_{k}(s'|s,a)u^*(s') E_\mu^{\pi,\kappa^{(k)}}\sqbkcond{\pi_{k}(a|H_{k})}{X_{k} = s}\\
&\stackrel{(i)}{=} \sum_{a\in A}\sum_{s'\in S}\psi_{k}(s'|s,a)u^*(s') E_\mu^{\pi,\kappa^{(k-1)}}\sqbkcond{\pi_{k}(a|H_{k})}{X_{k} = s}\\
&= \sum_{a\in A}\sum_{s'\in S}\Delta_{k}(a|s)\psi_{k}(s'|s,a)u^*(s') 
\end{align*}
where $(i)$ follows from the observation that the law of $H_{k-1}$ is not dependent on the adversarial decision rules after time $k-1$, and the last line recalls the definition of the projected decision rule \eqref{eqn:effective_decision_rule}.
Apply this to \eqref{eqn:bellman_term}, we have that
\begin{equation}\label{eqn:cond_on_X_k_bellman_term}
\begin{aligned}
&\quad E_\mu^{\pi,\kappa^{(k)}}\sqbk{r(X_{k},A_{k}) +  \gamma u^*(X_{k+1})|X_{k}} \\
&= \sum_{a\in A}\sum_{s'\in S}\Delta_{k}(a|X_{k})\psi_{k}(s'|X_{k},a)[r(X_{k},a)+\gamma u^*(s') ] \\
&\stackrel{(i)}{\leq} u^*(X_{k}) + \delta.
\end{aligned}
\end{equation}
where $(i)$ follows from \eqref{eqn:bellman_property_of_psi}. 
With this result, we can bound
\begin{align*}
    v_\mu^\pi(k+1,\kappa^{(k)})
    &= E_\mu^{\pi,\kappa^{(k)}}\sqbk{\sum_{j=0}^{k-1}\gamma^j r(X_j,A_j) }+ \gamma^{k} E_\mu^{\pi,\kappa^{(k)}} E_\mu^{\pi,\kappa^{(k)}}\sqbkcond{r(X_{k},A_{k}) +  \gamma u^*(X_{k+1})}{X_{k}}\\
    &\leq E_\mu^{\pi,\kappa^{(k)}}\sqbk{\sum_{j=0}^{k-1}\gamma^j r(X_j,A_j) } + \gamma^{k} E_\mu^{\pi,\kappa^{(k)}} u^*(X_{k}) + \gamma^{k}\delta\\
    &\stackrel{(i)}{=}E_\mu^{\pi,\kappa^{(k-1)}}\sqbk{\sum_{j=0}^{k-1}\gamma^j r(X_j,A_j)  +  \gamma^{k}u^*(X_{k})} +  \gamma^{k}\delta\\
    &= v_\mu^\pi(k,\kappa^{(k-1)})+  \gamma^{k}\delta
\end{align*}
as claimed, where $(i)$ follows from the quantity inside the expectation is a bounded function of $H_{k}$ and doesn't depend on $X_{k+1}$. 
\end{proof}

\section{Proof of Theorem \ref{thm:s-rec_markov_opt}}\label{A_sec:proof:thm:s-rec_markov_opt}
\begin{proof}{Proof}
\par By $\eta$-optimality of $\Delta$,  for all $s\in S$
\begin{align*}
u^*(s) 
&\leq \inf_{p_s\in\cP_s}E_{\Delta(\cd|s),p_s}[r(s,A_0 ) +  \gamma u^*(X_1)] + \eta
\end{align*}
Recall the proof of Proposition \ref{prop:s-rec_constrained_(3,1)}, where \eqref{eqn:bellman_sup_pi_delta} is satisfied with $\delta$ replaced by $\eta$. Therefore, for any $\epsilon > 0$, adapting the same argument, one can conclude that as in \ref{eqn:PiS_KSH_err_bound}
\begin{align*}
0 
&= E_\mu[u^*(X_0)] - v(\mu,\Pi,\mrm{K})\\
&\leq v(\mu,\Pi,\mrm{K})- \inf_{\kappa\in \mrm{K}}v(\mu,\pi,\kappa)\\
&\leq E_\mu[u^*(X_0)]- \inf_{\kappa\in \KSH}v(\mu,\pi,\kappa)\\
&\leq \frac{\eta}{1-\gamma} + \frac{3\epsilon}{4}.
\end{align*}
Since $\epsilon > 0$ is arbitrary, this implies that 
\[
0\leq v(\mu,\Pi,\mrm{K}) -  \inf_{\kappa\in \mrm{K}}v(\mu,\pi,\kappa)\leq \frac{\eta}{1-\gamma}
\]
as claimed. 
\end{proof}

\section{Proofs of Lemmas in Section \ref{section:counterexample} }
\label{sec:app:proofs:counterexample}
\begin{proof}{Proof of Lemma \ref{lma:counter-markov-stat}}
We start by solving the robust Bellman equation. It is easy to see that $u^{\ast }(\text{C})=0,u^{\ast }(\text{B})=3$. We claim that $u^{\ast }(\text{\text{I}})=1$. By Proposition \ref{prop:bellman_exst_unq_sol}, the solution $u^*$ is unique, so if $u^{\ast }(\text{C})=0,u^{\ast }(\text{B})=3,u^{\ast }(\text{\text{I}})=1$ verifies \eqref{eqn:dr_bellman_eqn}, it must be the solution. We proceed to check 
\begin{eqnarray*}
	u^{\ast }(\text{\text{I}}) &=&\gamma \sup_{\phi\in \mathcal{P}(\mathcal{A})}\inf_{p_\text{I}\in \cP_\text{I}}\phi(a_1)p_{\text{I},a_{1}}\cd u^{\ast }+\phi(a_2)p_{\text{I},a_{2}}\cd u^{\ast
	} \\
	&=&\gamma \sup_{\phi(a_1)\in \lbrack 0,1]}\inf \left\{ \phi(a_1)+\frac{1+3}{2}(1-\phi(a_1)), \frac{0+3}{2}%
	\phi(a_1) +\frac{0+1}{2}(1-\phi(a_1))\right\} \\
	&=&0.8\sup_{\phi(a_1)\in \lbrack 0,1]}\inf \left\{ 2-\phi(a_1),\frac{1}{2}+\phi(a_1)\right\} \\
	&=&1,
\end{eqnarray*}%
where $\cd$ denotes the vector dot product and the optimal controller's action distribution is $\phi(a_{1})=3/4.$

We then consider a Markov but not time-homogeneous policy $\pi = (\pi_0,\pi_1,\ds)$ where  
\begin{equation*}
	\pi _{0}(\text{I})=a_{1}\text{ and }\pi _{1}(\text{I})=a_{2}\text{ and }\pi
	_{i}(\text{I})=(3/4,1/4)\text{ for }i\geq 2.
\end{equation*}%
At time 0, if the adversary chooses $p_\text{\text{I}}^{(2)},$ we have 
\begin{equation*}
	v(\delta_\text{I},\pi,p_\text{\text{I}}^{(2)})=0.8(u^{\ast }(\text{B})+u^{\ast }(\text{C}))/2=1.2>1.
\end{equation*}%
On the other hand, if the adversary chooses $p_\text{I}^{(1)},$  we have 
\begin{eqnarray*}
	v_1(\text{I},\pi,p_\text{\text{I}}^{(1)}) &=&\frac{\gamma }{2}(u^{\ast }(\delta_\text{I})+u^{\ast }(\text{B}))=1.6,\text{ and} \\
	v(\delta_\text{I},\pi,p_\text{\text{I}}^{(1)}) &=&\gamma v_1(\text{A},\pi,p_\text{\text{I}}^{(1)})=1.44>1,
\end{eqnarray*}%
where $v_1(\text{I},\pi,p_\text{\text{I}}^{(1)})$ and $v_1(A,\pi,p_\text{\text{I}}^{(1)})$ stands for the value of states I and A at time 1 under the Markov policy $\pi$ and the adversary with the choice $p_\text{\text{I}}^{(1)}$.
\end{proof}

\begin{proof}{Proof of Lemma \ref{lemma:counter:det:bellman}}
We  compute the solution $u^*(\cdot)$ of the robust Bellman equation:
\begin{eqnarray*}
	u(\text{I}) &=&0+ \sup_{a\in \left\{ a_{1},a_{2}\right\} }\inf_{p\in\cP_\text{\text{I}} }\gamma E_{p_{\text{I},a}} [u(X_1)], \\
 u(\text{I}_0) &=& \frac{\gamma}{2} u(\text{I}_1) +  \frac{\gamma}{2} u(\text{I}_2), u(\text{I}_1) = \gamma u(\text{I}) ,u(\text{I}_2) = \gamma u(\text{I}), \\
	u(\text{G}) &=&1+\gamma u(\text{I}), u(\text{B}) =-1+\gamma u(\text{I}).
\end{eqnarray*}%
The solution $u^*$ is
\begin{align*}
	&u^*(\text{I}) =-\frac{\gamma }{1-\gamma ^{2}},\quad u^*(\text{I}_1)=u^*(\text{I}_2)=-\frac{\gamma^2 }{1-\gamma ^{2}},\\
 &u^*(\text{I}_0)=-\frac{\gamma^3 }{1-\gamma ^{2}},\quad u^*(\text{G}) =1-\frac{\gamma ^{2}}{1-\gamma ^{2}}, \quad u^*(\text{B}) =-\frac{1}{1-\gamma ^{2}}.
\end{align*}
\end{proof}

\begin{proof}{Proof of Lemma \ref{lemma:counter:det:Markov}}
    We use $p_\text{\text{I}}^\alpha$ to denote the adversary's action $\alpha p_\text{\text{I}}^{(1)}+(1-\alpha)p_\text{\text{I}}^{(2)}$. Then, we have 
\begin{eqnarray*}
v_{2}(\text{I},\pi ,p_{\text{I}}^{\alpha }) &=&0+\gamma \alpha v_{3}(\text{B},\pi ,p_{\text{I}}^{\alpha
})+\gamma (1-\alpha )v_{3}(\text{G},\pi ,p_{\text{I}}^{\alpha }), \\
v_{3}(\text{B},\pi ,p_{\text{I}}^{\alpha }) &=&-1+\gamma v_{4}(\text{I},\pi ,p_{\text{I}}^{\alpha
}),\\
v_{3}(\text{G},\pi ,p_{\text{I}}^{\alpha })&=& 1+\gamma v_{4}(\text{I},\pi ,p_{\text{I}}^{\alpha }), \\
v_{4}(\text{I},\pi ,p_{\text{I}}^{\alpha }) &=&0+\gamma \alpha v_{5}(\text{G},\pi ,p_{\text{I}}^{\alpha
})+\gamma (1-\alpha )v_{5}(\text{B},\pi ,p_{\text{I}}^{\alpha }), \\
v_{5}(\text{B},\pi ,p_{\text{I}}^{\alpha }) &=&-1+\gamma v_{6}(\text{I},\pi ,p_{\text{I}}^{\alpha
}),\\
v_{5}(\text{G},\pi ,p_{\text{I}}^{\alpha }) &=&1+\gamma v_{6}(\text{I},\pi ,p_{\text{I}}^{\alpha }), \\
v_{6}(\text{I},\pi ,p_{\text{I}}^{\alpha }) &=&v_{2}(\text{I},\pi ,p_{\text{I}}^{\alpha }),
\end{eqnarray*}%
where the notations $v_\cdot(\cdot,\pi,p_{\text{I}}^{\alpha })$ are defined similar to those in the proof of Lemma \ref{lma:counter-markov-stat}.

Solving the recursion, one obtains the solution and verifies the value $v(\delta_{\text{I}_0},\pi ,p_{\text{I}}^{\alpha })>u^{\ast }(\text{I}_0)$ uniformly in $\alpha$ as follows:
\begin{eqnarray*}
v(\delta_{\text{I}_0},\pi ,p_{\text{I}}^{\alpha }) &=&(1-2\alpha )\frac{\gamma ^{3}}{1+\gamma ^{2}}%
\geq -\frac{\gamma ^{3}}{1+\gamma ^{2}}>u^{\ast }(\text{I}_0), \\
v_{2}(\text{I},\pi ,p_{\text{I}}^{\alpha }) &=&(1-2\alpha )\frac{\gamma }{1+\gamma ^{2}}%
\geq -\frac{\gamma }{1+\gamma ^{2}}>u^{\ast }(\text{I}), \\
v_{3}(\text{B},\pi ,p_{\text{I}}^{\alpha }) &=&-1+\left( 2\alpha -1\right) \frac{\gamma
^{2}}{1+\gamma ^{2}},\\
v_{3}(\text{G},\pi ,p_{\text{I}}^{\alpha }) &=& 1+\left( 2\alpha
-1\right) \frac{\gamma ^{2}}{1+\gamma ^{2}} \\
v_{4}(\text{I},\pi ,p_{\text{I}}^{\alpha }) &=&\left( 2\alpha -1\right) \frac{\gamma }{%
1+\gamma ^{2}}.
\end{eqnarray*}
\end{proof}

\cblue{
\section{Proof of Theorem \ref{thm:asymp_opt_policy}}

To prove Theorem \ref{thm:asymp_opt_policy}, we first establish some preliminary results. The proof of Proposition \ref{prop:concentration_of_renewal_times} is deferred to Section \ref{sec:proof:prop:concentration_of_renewal_times}. Proposition \ref{prop:EVI_SC_diam} is obtained by directly adapting Theorem 9 and Lemma 4 from \citet{Zurek_2024_plugin_SC_MDP}.

\begin{proposition}
\label{prop:concentration_of_renewal_times}
Under the assumption of Theorem \ref{thm:asymp_opt_policy}, let the probability measure $P^{\pi,p}_\mu$ be generated by any time-homogeneous adversary $ p\in\cP^\mrm S$, Policy \ref{policy:ETE}, and initial distribution $\mu\in\cP(\cS)$. \text{I}f $n\geq (4D|Z| + 2\log (1/\delta))D(D+1)|A|^D $, then $J_n$ defined in Subroutine \ref{subroutine:exploration} satisfies that $J_n = |Z|+1$ w.p. at least $1-\delta$ under $P^{\pi,p}_\mu$. 

\par  \text{I}n another word, with probability at least $1-\delta$, the exploration process has collected $m$ i.i.d. samples from $p_{s,a}(\cd)$ for each $(s,a)\in S\times A$, where $m$ is defined in \eqref{eqn:choice_of_m}. 
\end{proposition}

\begin{proposition}[\citet{Zurek_2024_plugin_SC_MDP}]\label{prop:EVI_SC_diam}
Let $\hat p$ be an empirical kernel formed by $m$ i.i.d. samples from $p_{s,a}(\cd)$ for each $(s,a)\in S\times A$. Assume that $p$ is communicating with diameter at most $D$. Then, with probability at least $1-\delta$ under the probability measure generating the samples forming $\hat p$, 
\begin{equation}\label{eqn:prop:EVI_SC_diam}v_\gamma(\mu,\PiH,p) - v_\gamma(\mu,\hat \pi,p) \leq \frac{c}{1-\gamma}\sqrt{
\frac{D}{m}\log\crbk{ \frac{|Z|m}{(1-\gamma)\delta}}^3}
\end{equation}
where $\hat \pi$ is computed as in Subroutine \ref{subroutine:exploitation} and $c$ is a numerical constant. 
\end{proposition}
\begin{remark}
Proposition \ref{prop:EVI_SC_diam} follows from a simple calculation that combines Theorem 9 and Lemma 4 in \citet{Zurek_2024_plugin_SC_MDP}. 
\end{remark}
With these preliminary results, we present a proof of Theorem \ref{thm:asymp_opt_policy}. 

\begin{proof}{Proof of Theorem \ref{thm:asymp_opt_policy}}
    We decompose and bound the loss using $\pi$ as follows. 
\begin{equation}\label{eqn:policy_error_decomp}\begin{aligned}&v_{\gamma}(\mu,\PiH,\KSS) -v_{\gamma}(\mu,\pi,\KSS)\\
&\leq \inf_{p\in\cP}v_{\gamma}(\mu,\PiH,p) -\inf_{p\in\cP} v_{\gamma}(\mu,\pi,p)\\
&\stackrel{(i)}{\leq} \sup_{p\in\cP^S}\crbk{v_{\gamma}(\mu,\PiH,p) -v_{\gamma}(\mu,\pi,p)} \\
&\stackrel{(ii)}{=}  \sup_{p\in\cP^S}\crbk{v_{\gamma}(\mu,\pi_\gamma^*(p),p) -v_{\gamma}(\mu,\pi,p)} \\
&=\sup_{p\in\cP^S}\crbk{E_\mu^{\pi_\gamma^*(p),p}\sqbk{\sum_{k=0}^\infty \gamma^kr(X_k,A_k)}-E_\mu^{\pi,p}\sqbk{\sum_{k=0}^\infty \gamma^kr(X_k,A_k)}}\\
&\leq \frac{1-\gamma^n}{1-\gamma} +  \sup_{p\in\cP^S}\crbk{E_\mu^{\pi_\gamma^*(p),p}\sqbk{\sum_{k=n+1}^\infty \gamma^kr(X_k,A_k)}-E_\mu^{\pi,p}\sqbk{\sum_{k=n+1}^\infty \gamma^kr(X_k,A_k)}} \end{aligned}\end{equation}
where $\pi^*_\gamma(p)$ is an optimal Markov time-homogeneous policy under $p,\gamma$, $(ii)$ follows from Markov optimality of non-robust MDPs, and the last inequality uses $0\leq r\leq 1$. Here, $(i)$ uses the following fact that $$\inf_x f(x) - \inf_x g(x) = -\sup_x-f(x) + \sup_x-g(x) \leq \sup_x (f(x)-g(x)).$$

\par First, we bound $1-\gamma^n$. Note that
$$\gamma^n = \exp(n\log(1- 1-\gamma) ) \geq \exp\crbk{n\frac{\gamma-1} {\gamma}} = \exp\crbk{-\frac{\sqrt{1-\gamma}} {\gamma}}\geq 1 - \frac{\sqrt{1-\gamma}} {\gamma}.$$
Here we used the fact that $\log(1+x)\geq x/(1+x)$ for $x> -1$ as well as $e^{x}\geq x+1$ for all $x$. Therefore, 
\begin{equation}\label{eqn:exploration_gamma_bound}\frac{1-\gamma^n}{1-\gamma} \leq \frac{\sqrt{1-\gamma}}{\gamma(1-\gamma)}\leq \frac{2}{\sqrt{1-\gamma}}\end{equation} for $\gamma \geq 1/2$. 

\par Next we analyze the second term in the bound in \ref{eqn:policy_error_decomp}. By the Markov property, we have that 
\begin{align*} E_\mu^{\pi_\gamma^*(p),p}\sqbk{\sum_{k=n+1}^\infty \gamma^kr(X_k,A_k)} &= E_\mu^{\pi_\gamma^*(p),p} E_\mu^{\pi_\gamma^*(p),p}\sqbkcond{\sum_{k=n+1}^\infty \gamma^kr(X_k,A_k)}{\cH_{n+1}}\\
&=\gamma^{n+1} E_\mu^{\pi_\gamma^*(p),p}E_{X_{n+1}}^{\pi_\gamma^*(p),p}\sqbk{\sum_{k=0}^\infty \gamma^kr(X_k,A_k)}\\
&=\gamma^{n+1} E_\mu^{\pi_\gamma^*(p),p}v(X_{n+1},\pi^*_\gamma(p),p)
\end{align*}
On the other hand, 
\begin{align*} E_\mu^{\pi,p}\sqbk{\sum_{k=n+1}^\infty \gamma^kr(X_k,A_k)} &= E_\mu^{\pi,p} E_\mu^{\pi,p}\sqbkcond{\sum_{k=n+1}^\infty \gamma^kr(X_k,A_k)}{H_{n+1}}\\
&= \gamma^{n+1}E_\mu^{\pi,p}  v_\gamma(X_{n+1},\hat\pi,p) \\
&= \gamma^{n+1}E_\mu^{\pi,p} \dsi\set{J_n = |Z|+1}v_\gamma(X_{n+1},\hat\pi,p)   +  \dsi\set{J_n < |Z|+1}v_\gamma(X_{n+1},\eta,p).
\end{align*}
where we use the Markov property of $\set{(X_k,A_k):k\geq 0}$ and the definition of $\hat\pi$.

Therefore, going back to \eqref{eqn:policy_error_decomp}, since $0\leq v_\gamma (s,d,p) \leq \frac{1}{1-\gamma}$ uniformly in $s,d,p$, we can bound
\begin{equation}\label{eqn:bounding_exploitation_err}\begin{aligned}&\abs{E_\mu^{\pi_\gamma^*(p),p}\sqbk{\sum_{k=n+1}^\infty \gamma^kr(X_k,A_k)}  -  E_\mu^{\pi,p}\sqbk{\sum_{k=n+1}^\infty \gamma^kr(X_k,A_k)} }\\
&= \gamma^{n+1}\abs{E_\mu^{\pi_\gamma^*(p),p}[v(X_{n+1},\pi^*_\gamma(p),p)] - E_\mu^{\pi,p}[ (\dsi\set{J_n = |Z|+1}+\dsi\set{J_n <|Z|+1} )v_\gamma(X_{n+1},\hat \pi,p)   ]}\\
&\leq \abs{E_\mu^{\pi_\gamma^*(p),p}[v(X_{n+1},\PiS,p)] - E_\mu^{\pi,p}  \dsi\set{J_n = |Z|+1} v_\gamma(X_{n+1},\hat \pi,p)   }+ \frac{1}{1-\gamma}P_\mu^{\pi,p} (J_n < |Z|+1)\\
&\leq \abs{E_\mu^{\pi_\gamma^*(p),p}[v(X_{n+1},\PiS,p)] - E_\mu^{\pi,p}  \dsi\set{J_n = |Z|+1} v_\gamma(X_{n+1},\hat \pi,p)   }+ \frac{1}{\sqrt{1-\gamma}}\end{aligned}\end{equation}
by invoking Proposition \ref{prop:concentration_of_renewal_times} with $\delta = \sqrt{1-\gamma}$, as long as
$$n = \frac{1}{\sqrt{1-\gamma}}\geq \crbk{4D|Z| + \log\crbk{\frac{1}{1-\gamma}} }D(D+1)|A|^D. $$ Clearly, this will be the case for all sufficiently large $\gamma$ as $\gamma\ua 1$. 

\par To continue from \eqref{eqn:bounding_exploitation_err}, we proceed to upper and lower bound $$E_\mu^{\pi_\gamma^*(p),p}[v(X_{n+1},\PiS,p)] - E_\mu^{\pi,p}  \dsi\set{J_n = |Z|+1} v_\gamma(X_{n+1},\hat \pi,p). $$ 
Before we do this, we simply the notation by defining $m$ as in \eqref{eqn:choice_of_m} and
$$\epsilon_\delta:= \frac{c}{1-\gamma}\sqrt{
\frac{D}{m}\log\crbk{ \frac{|Z|m}{(1-\gamma)\delta}}^3}$$
which appears in \eqref{eqn:prop:EVI_SC_diam}. We also denote the event in \eqref{eqn:prop:EVI_SC_diam} by $\Omega_\delta$; i.e. 
$$\Omega_\delta = \set{\omega = (h_n,a_n,s_{n+1},\ds )\in\Omega:v_\gamma(\mu,\PiH,p) - v_\gamma(\mu,\hat \pi (h_n),p) \leq\epsilon_\delta}. $$
\par Again, with $\delta = \sqrt{1-\gamma}$ and sufficiently large $\gamma$, we have that
\begin{align*} 
& (1-2\delta)\inf_{s\in S}v_\gamma(s,\PiS,p) - \epsilon_{\delta}\\
&\leq E_\mu^{\pi,p}\dsi\set{J_n = |Z|+1,\omega\in\Omega_\delta} [v_\gamma(X_{n+1},\PiS,p)-\epsilon_{\delta}]\\
&\leq E_\mu^{\pi,p}\dsi\set{J_n = |Z|+1,\omega\in\Omega_\delta} v_\gamma(X_{n+1},\hat \pi,p)\\
&\leq  E_\mu^{\pi,p}  \dsi\set{J_n = |Z|+1} v_\gamma(X_{n+1},\hat \pi,p)  \\
&\leq  E_\mu^{\pi,p}\dsi\set{J_n = |Z|+1,\omega\in\Omega_\delta} v_\gamma(X_{n+1},\hat \pi,p) + \frac{1}{1-\gamma} P_\mu^{\pi,p}(\set{J_n = |Z|+1}\cap \Omega_\delta^c)\\
&\leq E_\mu^{\pi,p}\dsi\set{J_n = |Z|+1,\omega\in\Omega_\delta} [v_\gamma(X_{n+1},\PiS,p)+ \epsilon_{\delta}] + \frac{1}{\sqrt{1-\gamma}}  \\
&\leq \sup_{s\in S}v_\gamma(s,\PiS,p)+ \epsilon_{\delta} + \frac{1}{\sqrt{1-\gamma}}\end{align*}
Also, 
$$\inf_{s\in S}v(s,\PiS,p) \leq v(X_{n+1},\PiS,p)\leq \sup_{s\in S}v(s,\PiS,p).$$
Therefore, it is not hard to see that
\begin{align*}&\abs{E_\mu^{\pi_\gamma^*(p),p}[v(X_{n+1},\PiS,p)] - E_\mu^{\pi,p}  \dsi\set{J_n = |Z|+1} v_\gamma(X_{n+1},\hat \pi,p)   }\\
&\leq \sup_{s\in S}v(s,\PiS,p) - (1-2\delta)\inf_{s\in S}v_\gamma(s,\PiS,p)+\epsilon_{\delta}+\frac{1}{\sqrt{1-\gamma}} \\
&\leq \spnorm{v_\gamma(\PiS,p)} + \epsilon_{\delta} +\frac{3}{\sqrt{1-\gamma}}\end{align*}
where $\spnorm{f} = \sup_s f(s) - \inf_sf(s)$. 
\par We plug this back into \eqref{eqn:bounding_exploitation_err} and then into \eqref{eqn:policy_error_decomp} to get that
$$\begin{aligned}v_{\gamma}(\mu,\PiH,\KSS) -v_{\gamma}(\mu,\pi,\KSS)&\leq \sup_{p\in\cP^S}\spnorm{v_\gamma(\PiS,p)} + \epsilon_{\delta} +\frac{3}{\sqrt{1-\gamma}} + \frac{1-\gamma^n}{1-\gamma}\\
&\leq D +  \frac{c}{1-\gamma}\sqrt{
\frac{D}{m}\log\crbk{ \frac{|Z| m}{(1-\gamma)^{3/2}}}^3
}+ \frac{5}{\sqrt{1-\gamma}} \end{aligned}$$
for sufficiently large $\gamma\ua 1$. Note that $m$ as defined in \eqref{eqn:choice_of_m} is $O(n) = O(\frac{1}{\sqrt{1-\gamma}})$. Therefore, $$0\leq (1-\gamma)v_{\gamma}(\mu,\PiH,\KSS) -v_{\gamma}(\mu,\pi,\KSS)\leq O\crbk{(1-\gamma)^{1/4}\log\crbk{\frac{1}{1-\gamma}}^3}\ra 0.$$ This completes the proof. We remark that the rate can be improved by further balancing the exploration and exploitation errors. 
\end{proof}

\subsection{Proof of Proposition \ref{prop:concentration_of_renewal_times}}\label{sec:proof:prop:concentration_of_renewal_times}
To prove Proposition \ref{prop:concentration_of_renewal_times}, we first establish the following Lemma \ref{lemma:return_prob_bound}, whose proof can be found in Section \ref{sec:proof:lemma:return_prob_bound}. 

\begin{lemma}\label{lemma:return_prob_bound}
    Assume that for every $p\in\cP^\mrm{S}$, $p$ is communicating with diameter at most $D$. Consider the Markov time-homogeneous decision rule $\eta$ given by $\eta(a|s) = 1/|A|$ for all $s\in S, a\in A$. Then for all $p\in\cP^\mrm{S}$ and every $s,s'\in S$, there exists integer $t\leq D$ s.t. 
    $$(p^\eta)^t(s,s')\geq \frac{1}{2D^2|A|^D}$$
    where $p^\eta(s,s') = \sum_{a\in A} p_{s,a}(s')\eta(a|s). $
\end{lemma}

With Lemma \ref{lemma:return_prob_bound}, we can proceed to prove Proposition \ref{prop:concentration_of_renewal_times}. 
\begin{proof}{Proof of Proposition \ref{prop:concentration_of_renewal_times}}
    From Markov renewal theory, the hitting times $\tau_{j,k}$ defined in \ref{eqn:def_renewal_times} are independent. For each $\tau_{j,k}$, $k\leq m$, upon hitting $s = z_j(0)$ for $k-1$ times, the action $a = z_j(1)$ is used. Then, the uniform randomize decision $\eta$ is used until the chain hits $s$ again. By Lemma \ref{lemma:return_prob_bound}, we know that for any time-homogeneous adversary $p$, after $t = t(p, s,s)\leq D$ steps of applying $\eta$, the probability of returning to $s$ is lower bounded by $q:= \frac{1}{2D^2|A|^D}$. Therefore, adding the first transition step using action $a$, we have 
    \begin{equation}P_\mu^{\bar \pi,p}(\tau_{j,k} \leq i ) \geq P( t(p, s,s) G +1 \leq i )\geq P( D G  +1\leq i )\label{eqn:bound_success_exploration_with_geoms}\end{equation}
where $\bar\pi = (\bar\pi_0,\ds ,\bar \pi_n,\ds)$ and $G$ is a geometric random variable under $P$, with success probability $q = \frac{1}{2D^2|A|^D}$. The same bound holds for $\tau_{j+1,1}$. Therefore, by independence, 
    $$P_\mu^{\pi,p}\crbk{J_n = |Z|+1} = P_\mu^{\bar \pi,p}\crbk{\sum_{j=1}^{|Z|}\sum_{k=1}^{m}\tau_{j,k} \leq n }\geq P\crbk{D|Z|m + D\sum_{k=1}^{|Z|m}G_{k} \leq n }$$
where $G_k\sim \geom(q)$ i.i.d.. We note that the first equality make use of the fact that the restricted measures $P_\mu^{\pi,p}|_{\cH_n} = P_ \mu^{\bar \pi,p}|_{\cH_n}$ where $\cH_n = \sigma(H_n)$. 
    
\par  Now we analyze the concentration property of the sum of geometric random variables. Let $X = \sum_{k=1}^{|Z|m}G_k$, then by Theorem 2.1 in \citet{Janson_2018_geometric_tail_bd}, 
    $$P(X> \lambda EX)\leq \exp\crbk{-|Z|m(\lambda - 1-\log\lambda )}.$$
Observe that if we choose $\lambda \geq 6$, then  $\lambda - 1 - \log\lambda\geq \lambda/2$, and hence $P(X> \lambda EX)\leq e^{-\lambda |Z|m/2}$. Moreover, by the choice of $m$ in \eqref{eqn:choice_of_m} and $n$ in the statement of Proposition \ref{prop:concentration_of_renewal_times},
$$n\geq\frac{12D|Z| + 4\log (1/\delta)}{3q} \implies m\geq\frac{\log (1/\delta)}{3|Z|}.$$
Therefore, using $\lambda = 6$,  
$$\begin{aligned}
\delta& \geq e^{-3 |Z|m} \\
&\geq P\crbk{ D|Z|m + DX> 3 D|Z|m  +  \frac{D|Z|m}{q}}\\
&\geq  P\crbk{ D|Z|m + DX>  \frac{4D|Z|m}{q}}\\
&\geq  P\crbk{ D|Z|m + DX>  n}\end{aligned}$$
where the last inequality follows from the choice of $m$ in \eqref{eqn:choice_of_m}. Combining this with \eqref{eqn:bound_success_exploration_with_geoms}, we have that 
$P_\mu^{\pi,p}\crbk{J_n = |Z|+1}\geq 1-\delta$. 
\end{proof}

\subsection{Proof of Lemma \ref{lemma:return_prob_bound}}\label{sec:proof:lemma:return_prob_bound}

\begin{proof}{Proof of Lemma \ref{lemma:return_prob_bound}}
    We first prove that, under the assumption of Lemma \ref{lemma:return_prob_bound}, for each $p\in\cP^\mrm{S}$ and $s,s'\in S$, there exists decision rule $d(\cd|\cd)$ and integer $t \leq D$, which can be dependent on $p,s,s'$,  s.t. $(p^d)^{t}(s,s') \geq \frac{1}{D(D+1)}$. 
    \par To show this, we consider a Markov chain with transition kernel $p^d(s_0,s_1):= \sum_{a\in A}p_{s_0,a}(s_1)d(a|s_0)$. Let $\tau:= \inf_{k\geq1}\set{X_k = s'}$. We have that $$\begin{aligned}D\
    \geq E_s^{d,p}\tau 
    \geq (D+1)P_s^{d,p}(\tau \geq D+1) .\end{aligned}$$
    So, we must have that $P_s^{d,p}(\tau \geq D+1) \leq D/(D+1)$. Hence,   
    $$  \max_{n=0,\ds,D}P_s^{d,p}(\tau= n) \geq \frac{1}{D}\sum_{n=0}^D P_s^{d,p}(\tau= n) = \frac{1}{D}P_s^{d,p}(\tau\leq  D)\geq \frac{1}{D(D+1)}$$
Therefore, letting $t = \argmax{n=0,\ds,D}P_s^{d,p}(\tau= n)$, we conclude that $ (p^d)^t(s,s') = P_s^{d,p}(\tau= t)\geq \frac{1}{D(D+1)}$. 

\par Next, denoting $s_0 = s$ and $s_n = s'$ we observe that
$$ \begin{aligned}(p^d)^t(s,s') &= \sum_{s_1,\ds,s_{t-1}}\prod_{k=0}^{t-1} p^d_{s_k}(s_{k+1})\\
&= \sum_{s_1,\ds,s_{t-1}} \prod_{k=0}^{t-1} \sum_{a\in A}d_{s_k}(a)p_{s_k,a}(s_{k+1})\\
&\leq \sum_{s_1,\ds,s_{t-1}} \prod_{k=0}^{t-1} |A| \sum_{a\in A}\frac{1}{|A|}p_{s_k,a}(s_{k+1})\\
&= |A|^t \sum_{s_1,\ds,s_{t-1}} \prod_{k=0}^{t-1} \sum_{a\in A}p_{s_k}^\eta (s_{k+1}) \\
&= |A|^t(p^\eta)^t(s,s').\end{aligned}$$
Therefore, we conclude that for all $p\in\cP^\mrm S$ and $(s,s')\in S$
$$(p^\eta)^t(s,s')\geq\frac{1}{D(D+1)|A|^D}\geq \frac{1}{2D^2|A|^D}.  $$
where we note that $D\geq 1$. 
\end{proof}

}

\cblue{\section{Equivalence Between RMDPs and DRMDPs}\label{A_sec:equiv_DRMDP}
In this section, we elaborate on the equivalence between RMDP and DRMDP models in the context of policy evaluation and optimization. We consider a DRMDP with a Markov time-homogeneous adversary (one can consider similar formulations with history-dependent or Markov adversary, analogously defined as in our manuscript). In a S-rectangular setting, for every $s\in S$, the adversary chooses a ``prior distribution" $\Psi_s\in  \cP(A\ra\cP(\cS))$ for the transition kernel out of state $s$ from an ambiguity set of priors  $\bd{D}_s\subset \cP(A\ra\cP(\cS))$. Given an adversarial prior $\Psi:=\set{\Psi_s \in \cP(A\ra\cP(\cS)):s\in S}$ for every state $s$, the dynamics is realized with the i.i.d. sampled sequence of transition kernels 
    $$\hat \kappa:=\set{\kappa_i:\hat\kappa_{i}(\cd|s,\cd)\in \cP_s,\forall s\in S,i = 0,1,\ds }$$ with $\hat\kappa_i(\cd|s,\cd)\sim\Psi_s$.  
    Here, the hat notation suggests that $\hat\kappa$ is random. 
    
    We let $\bd{D} = \bigtimes_{s\in S}\bd{D}_s$. The max-min control value is then defined as 
    $$
    \sup_{\pi\in \PiH}\inf_{\Psi\in \bd{D}}E_{\hat \kappa_i(\cd|s,\cd)\sim \Psi_s,\forall i =0,1,\ds, s\in S} E^{\pi,\hat \kappa}_\mu\sum_{k=0}^\infty \gamma^k r(X_k,A_k).
    $$ 

    We now justify the equivalence between DRMDPs and robust MDPs. For fixed policy $\pi\in\PiH$ and adversarial prior $\Psi$, we consider separately for each $t = 0,1,\ds $
    $$\begin{aligned}
    &E^{\pi,\hat\kappa}_\mu r(X_t,A_t)\\
    &= \sum_{s_0,a_0,\ds,s_t,a_t\in (S\times A)^t} \mu(s_0) \pi_0(a_0|s_0) \hat \kappa_0(s_1|s_0,a_0) \pi_1(a_1|h_1) \ds \hat\kappa_{t-1}(s_t|s_{t-1},a_{t-1})\pi_t(a_t|h_t)r(s_t,a_t)
    \end{aligned} $$
    where $h_t = (s_0,a_0,\ds s_t)$, the equality simply follows the definition of the measure $P^{\pi,\hat\kappa}_\mu$. 
    
    Let $\psi(\cd|s,\cd) = E_{\Psi_s} \hat\kappa(\cd|s,\cd) $ be the \textit{expected} (or the \textit{effective}) adversarial decision. By the linearity of expectation, we have
    $$
    \begin{aligned}
    &E_{\hat \kappa_i(\cd|s,\cd)\sim \Psi_s,\forall i =0,1,\ds, s\in S} E^{\pi,\hat \kappa}_\mu  r(X_t,A_t) \\
    &= \sum_{s_0,a_0,\ds,s_t,a_t\in (S\times A)^t} \mu(s_0) \pi_0(a_0|s_0) \cblue{E_{\Psi_{s_0}}[\hat \kappa_0(s_1|s_0,a_0)] }\pi_1(a_1|h_1) \ds \cblue{E_{\Psi_{s_{t-1}}}[\hat\kappa_{t-1}(s_t|s_{t-1},a_{t-1})]}\pi_t(a_t|h_t)r(s_t,a_t) \\
    &= \sum_{s_0,a_0,\ds,s_t,a_t\in (S\times A)^t} \mu(s_0) \pi_0(a_0|s_0) \cblue{\psi(s_1|s_0,a_0) }\pi_1(a_1|h_1) \ds \cblue{\psi(s_t|s_{t-1},a_{t-1})}\pi_t(a_t|h_t)r(s_t,a_t) \\
    &=  E^{\pi,\psi}_\mu  r(X_t,A_t)
    \end{aligned}$$
    where the adversary policy is $(\psi,\psi,\psi,\ds)$.

    Therefore, we have that 
    $$
    \sup_{\pi\in \PiH}\inf_{\Psi\in \bd{D}}E_{\hat \kappa_i(\cd|s,\cd)\sim \Psi_s,\forall i =0,1,\ds, s\in S} E^{\pi,\hat \kappa}_\mu\sum_{k=0}^\infty \gamma^k r(X_k,A_k) = \sup_{\pi\in \PiH}\inf_{\Psi\in  \bd{D}} E^{\pi,\psi}_\mu\sum_{k=0}^\infty \gamma^k r(X_k,A_k). 
    $$
    But since $E^{\pi,\psi}_\mu\sum_{k=0}^\infty \gamma^k r(X_k,A_k)$ only depend on $\Psi$ through the expectation $\psi(\cd|s,\cd) = E_{\Psi_s} \hat\kappa(\cd|s,\cd)$, this value above is equivalent to
    $$\sup_{\pi\in \PiH}\inf_{\psi\in \cD} E^{\pi,\psi}_\mu\sum_{k=0}^\infty \gamma^k r(X_k,A_k),$$
    where $\cD = \bigtimes_{s\in S} \cD_s$, $\cD_s = \set{E_{\Psi_s}\hat\kappa(\cd|s,\cd):\Psi_s\in \bd{D}_s}$. This is the max-min control value of an RMDP with a Markov time-homogeneous adversary, satisfying all the theory in this paper. 

    We also note that the same derivation holds if we replace \( \psi(\cdot \mid s, \cdot) \) with \( \psi_t(\cdot \mid s, \cdot)\) or \(\psi_t(\cdot \mid h_t, \cdot) \) in the context of Markov or history-dependent adversarial priors, with or without taking the supremum over the controller policy. Therefore, from both policy evaluation and learning perspectives, every DRMDP model is equivalent to an RMDP, where the ambiguity set in the RMDP corresponds to the set of expected transition kernels, with the expectation taken over the adversarial priors in the corresponding DRMDP.
    
    On the other hand, there could be a distinction in the existence of DPP if we consider a risk-sensitive version of the DRMDP model. In the risk-sensitive setting, we would be considering a risk measure on the cumulative reward; i.e., $\rho(\sum_k \gamma^kr(X_k,A_k))$. Due to the nonlinearity introduced by this risk measure, the previous arguments no longer work. This could be a theoretically interesting research direction. 
}

\end{document}